\documentclass[12pt]{article}
\usepackage[margin=1.25in]{geometry}
\usepackage[ruled, linesnumbered, algo2e]{algorithm2e}
\usepackage{tikz}
\usepackage{tikz-qtree}
\usepackage[page, header]{appendix}
\usepackage{titletoc}
\usepackage{amsmath, amsfonts, amssymb, amsthm, amsbsy, amscd, bm, bbm,mathrsfs} 
\usepackage{color}
\usepackage[colorlinks,linkcolor=blue,citecolor=blue]{hyperref}
\usepackage{algpseudocode,algorithm}
\usepackage{thmtools}
\usepackage{thm-restate}
\usepackage{verbatim}

\newtheorem{assumption}{Assumption}

\newcommand{\Var}{\operatorname{\textrm{Var}}}

\renewcommand\d{\operatorname{\mathrm{d}}}

\let\leq=\leqslant 
\let\geq=\geqslant %

\newcommand{\EE}{\operatorname{\mathbb{E}}}

\newcommand{\RR}{\mathbb{R}}
\renewcommand{\SS}{\mathbb{S}}

\newcommand{\fbf}[1]{\mathbf{#1}}

\newcommand{\bA}{\fbf{A}}
\newcommand{\bB}{\fbf{B}}

\newcommand{\bD}{\fbf{D}}

\newcommand{\bH}{\fbf{H}}
\newcommand{\bI}{\fbf{I}}

\newcommand{\bP}{\fbf{P}}

\newcommand{\bR}{\fbf{R}}

\newcommand{\bT}{\fbf{T}}
\newcommand{\bU}{\fbf{U}}
\newcommand{\bV}{\fbf{V}}
\newcommand{\bW}{\fbf{W}}
\newcommand{\bX}{\fbf{X}}
\newcommand{\bY}{\fbf{Y}}
\newcommand{\bZ}{\fbf{Z}}

\newcommand{\ba}{\fbf{a}}
\newcommand{\bb}{\fbf{b}}

\newcommand{\be}{\mathrm{\fbf{e}}}

\newcommand{\bp}{\fbf{p}}

\newcommand{\bu}{\fbf{u}}
\newcommand{\bv}{\fbf{v}}
\newcommand{\bw}{\fbf{w}}
\newcommand{\bx}{\fbf{x}}
\newcommand{\by}{\fbf{y}}
\newcommand{\bz}{\fbf{z}}
\DeclareMathOperator*{\argmin}{argmin}

\renewcommand{\leq}{\leqslant}
\renewcommand{\geq}{\geqslant}

\newcommand{\cD}{\mathcal{D}}

\newcommand{\cF}{\mathcal{F}}
\newcommand{\cG}{\mathcal{G}}

\newcommand{\cJ}{\mathcal{J}}

\newcommand{\cL}{\mathcal{L}}

\newcommand{\cN}{\mathcal{N}}
\newcommand{\cO}{\mathcal{O}}
\newcommand{\cP}{\mathcal{P}}

\newcommand{\cX}{\mathcal{X}}

\newcommand{\bh}{\fbf{h}}

\newcommand{\<}{\left\langle}
\renewcommand{\>}{\right\rangle}

\renewcommand{\ge}{\geq}

\newcommand{\tr}{{\rm tr}}

\newcommand{\diag}[1]{{\rm diag}\left(#1\right)}

\newcommand{\abs}[1]{\left\vert#1\right\vert}
\newcommand{\norm}[2]{\left\Vert#1\right\Vert_{#2}}

\newtheorem{theorem}{Theorem}
\newtheorem{corollary}{Corollary}

\newtheorem{lemma}{Lemma}
\newtheorem{proposition}{Proposition}
\theoremstyle{definition}
\newtheorem{definition}{Definition}
\newtheorem{example}{Example}
\newtheorem{remark}{Remark}

\newcommand{\set}[1]{{\left\{ #1 \right\}}}

\newcommand{\paren}[1]{{\left( #1 \right)}}
\newcommand{\brac}[1]{{\left[ #1 \right] }}
\newcommand{\bzero}{{\mathbf 0}}

\newcommand{\E}{\mathbb{E}}

\newcommand{\normA}{\kappa_1}
\newcommand{\RF}{{\tt RF}}
\newcommand{\normP}{\kappa_2}

\DeclareMathOperator{\poly}{poly}

\let\vec\relax\DeclareMathOperator{\vec}{Vec}
\DeclareMathOperator{\Mat}{Mat}
\DeclareMathOperator{\sym}{Sym}

\title{Learning Hierarchical Polynomials with\\ Three-Layer Neural Networks}
\author{
  Zihao Wang\\
  Peking University\\
  \texttt{zihaowang@stu.pku.edu.cn}
  \and
  Eshaan Nichani\\
  Princeton University\\
  \texttt{eshnich@princeton.edu}
  \and
  Jason D. Lee\\
  Princeton University\\
  \texttt{jasonlee@princeton.edu}
}

\let\leq=\leqslant 
\let\geq=\geqslant %
\let\le=\leqslant 
\let\ge=\geqslant %

\date{\today\vspace{-1em}}

\ifdefined\usebigfont
\rene  wcommand{\mddefault}{bx}

\usepackage{times}
\usepackage[fontsize=13pt]{scrextend}
\makeatletter
\@ifpackageloaded{geometry}{\AtBeginDocument{\newgeometry{letterpaper,left=1.56in,right=1.56in,top=1.71in,bottom=1.77in}}}{\usepackage[letterpaper,left=1.56in,right=1.56in,top=1.71in,bottom=1.77in]{geometry}}
\AtBeginDocument{\newgeometry{letterpaper,left=1.56in,right=1.56in,top=1.71in,bottom=1.77in}}
\makeatother

\else
\usepackage{times}
\fi
\usepackage{subfigure}
\usepackage[square]{natbib}
\title{Learning Hierarchical Polynomials of Multiple Nonlinear Features with Three-Layer Networks}
\author{
Hengyu Fu\\
Peking University\\
\texttt{fhy2021@stu.pku.edu.cn}
\and
  Zihao Wang\\
  Stanford University\\
  \texttt{zihaow@stanford.edu}
  \and
  Eshaan Nichani\\
  Princeton University\\
  \texttt{eshnich@princeton.edu}
  \and
  Jason D. Lee\\
  Princeton University\\
  \texttt{jasonlee@princeton.edu}
}
\date{\today}

\begin{document}

\maketitle

\begin{abstract}
    In deep learning theory, a critical question is to understand how neural networks learn hierarchical features. In this work, we study the learning of hierarchical polynomials of \textit{multiple nonlinear features} using three-layer neural networks. We examine a broad class of functions of the form $f^{\star}=g^{\star}\circ \bp$, where $\bp:\mathbb{R}^{d} \rightarrow \mathbb{R}^{r}$ represents multiple quadratic features with $r \ll d$ and $g^{\star}:\mathbb{R}^{r}\rightarrow \mathbb{R}$ is a polynomial of degree $p$. This can be viewed as a nonlinear generalization of the multi-index model \citep{damian2022neural}, and also an expansion upon previous work that focused only on a single nonlinear feature, i.e. $r = 1$ \citep{nichani2023provable,wang2023learning}. 
    
    Our primary contribution shows that a three-layer neural network trained via layerwise gradient descent suffices for  
    \begin{itemize}\item complete recovery of the space spanned by the nonlinear features 
    \item efficient learning of the target function $f^{\star}=g^{\star}\circ \bp$ or transfer learning of $f=g\circ \bp$ with a different link function
    \end{itemize}
within $\widetilde{\cO}(d^4)$ samples and polynomial time.
For such hierarchical targets, our result substantially improves the sample complexity ${\Theta}(d^{2p})$ of the kernel methods, demonstrating the power of efficient feature learning. It is important to highlight that{  our results leverage novel techniques and thus manage to go beyond all prior settings} such as single-index and multi-index models as well as models depending just on one nonlinear feature, contributing to a more comprehensive understanding of feature learning in deep learning. 
\end{abstract}

\section{Introduction}
Deep neural networks have achieved remarkable empirical success across numerous domains of artificial intelligence \citep{alex2012,he2016deep}. This success can be largely attributed to their ability to extract latent features from real-world data and decompose complex targets into hierarchical representations, which improves test accuracy \citep{he2016deep} and allows efficient transfer learning \citep{devlin2018bert}. These feature learning capabilities are widely regarded as a core strength of neural networks over non-adaptive approaches such as kernel methods \citep{wei2020regularization,bai2020linearization}. 

Despite these empirical achievements, the feature learning capabilities of neural networks are less well understood from a theoretical point of view.
Previous work on feature learning has shown that two-layer neural networks can learn \textit{multiple} linear features of the input \citep{damian2022neural}, that is, multi-index models. However, the two-layer architecture inherently limits the network's ability to represent and learn nonlinear features \citep{daniely2017depthseparationneuralnetworks}. Given that many real-world scenarios involve diverse and nonlinear features, recent studies have shifted focus to investigating the learning of \emph{nonlinear} features using deeper neural networks. \citet{safran2022optimization, ren2023depth, nichani2023provable,wang2023learning} have demonstrated that three-layer networks, when trained via gradient descent, can efficiently learn \textit{hierarchical targets} of the form $h = g\circ p$, where $p$ represents certain types of features such as the norm $|\bx|$ or a quadratic form $\bx^{\top}\bA\bx$.  However, these studies are limited to relatively simple hierarchical functions and mainly focus on targets of a single feature. It remains unclear whether neural networks can efficiently learn a wider range of hierarchical functions, particularly those that depend on multiple nonlinear features. This leads us to the following central question:
\begin{center}
    \emph{Can neural networks adaptively identify \textbf{multiple nonlinear features} from the hierarchical targets by gradient descent,  thereby allowing an efficient learning for such targets?}
\end{center}

\subsection{Main Contributions}
In this paper, we provide strong theoretical evidence that three-layer neural networks have the ability to learn multiple hidden nonlinear features. Specifically, we study the problem of learning any hierarchical polynomial with multiple quadratic features using a three-layer network trained via layer-wise gradient descent. Our main contributions are summarized as follows:
\begin{itemize}
    \item \textbf{A Novel Analytic Framework for Multi-Nonlinear Feature Learning.} We demonstrate that when the target function belongs to a broad class of the form $f^{\star}=g^{\star}\circ\bp$, where $\bp:\RR^{d}\rightarrow\RR^{r}$ represents $r$ quadratic (nonlinear) features and $g^{\star}$ is a link function, the first step of gradient descent efficiently learns and recovers the space spanned by these nonlinear features $\bp$ within only $\widetilde{\cO}(d^4)$ samples. {{We remark that our proof techniques are also applicable to general nonlinear features. The core technical novelty is that we develop a novel and general universality argument (Lemma \ref{prop::extend chaterjee}) that bridges multi nonlinear feature models to multi-index models, which allows for an accurate reconstruction of the features through a simple linear transformation on the learned representations with small approximation error (Proposition \ref{prop::reconstructed feature main})}}
  
    \item \textbf{Improved Sample Complexity and Efficient Transfer Learning.} Leveraging the learned features in the first GD step, we prove that when the link function $g^{\star}$ is a polynomial of degree $p$, the gradient descent on the outer layer can achieves a vanishing generalization error with a small outer width and at most $\cO(r^{\cO(p)})$ additional training samples, removing the dependence on $d$ (Theorem \ref{thm:main_thm}).  This significantly improves upon the sample complexity of kernel methods, which require $\Theta(d^{2p})$ samples. Moreover, our analysis enables efficient transfer learning for any other target function of the form $f=g\circ\bp$ with a different link function $g$, which also only requires $\cO(r^{\cO(p)})$ additional samples.
    \end{itemize}
\subsection{Related Works}

\paragraph{Kernel Methods.}
Earlier research links the behavior of gradient descent (GD) on the entire network to its linear approximation near the initialization. In this scenario, neural networks act as kernels, known as the Neural Tangent Kernel (NTK). This connection bridges neural network analysis with established kernel theory and offers initial learning guarantees for neural networks~\citep{jacot2018neural, soltanolkotabi2018theoretical, du2018gradient, chizat2019lazy,arora2019exact}.  However, kernel theory fails to explain the superior empirical achievements of neural networks over kernel methods~\citep{arora2019exact, lee2020finite,e2020a}. Networks in the kernel regime \textbf{fail to learn features}~\citep{yang2021tensor}, not adaptable to hierarchical structures of real world targets. \citet{ghorbani2021linearized} proves that for uniformly distributed data on the sphere, the NTK method requires $\widetilde\Omega(d^k)$ samples to learn any polynomials of degree $k$ in $d$ dimensions, which is impractical when $k$ is large. Thus, a central question is how neural networks can detect and capture the underlying hierarchies in the target functions, which allows for a better generalization behavior versus kernel methods.

{
\paragraph{Learning Linear Features.}
Recent studies have demonstrated neural networks' capability to learn hierarchical functions of linear features more efficiently than kernel methods. Specifically, \citet{bietti2022learning, ba2022high} establish the efficient learning of single-index models, i.e., $f^{\star}(\bx)=g(\langle \bu,\bx\rangle)$. Furthermore, recent works \citet{damian2022neural, abbe2023sgd, dandi2023learning, bietti2023learning} further demonstrate that for isotropic data, two-layer or three-layer neural networks can effectively learn multi-index models of the form $f^*(\bx) = g(\bU\bx)$. These studies adopt certain modified training algorithms, such as layer-wise training. With sufficient feature learning, these networks can learn low-rank polynomials with a benign sample complexity of $\cO(d^{\cO(1)})$, which does not scale with the degree of the polynomial $g$.  Empirically, fully connected networks trained via gradient descent on image classification tasks also capture low-rank features \citep{lee2007sparse,radhakrishnan2022feature}. More recently, the learning of single-index and multi-index models is analyzed with more advanced algorithm framework or specified data structure. \cite{mousavihosseini2024learningmultiindexmodelsneural} considers learning general multi-index models with two-layer neural networks through a mean-field Langevin dynamics, \cite{dandi2024benefitsreusingbatchesgradient,lee2024neuralnetworklearnslowdimensional} goes beyond the traditional Correlational Statistical Query (CSQ) setting and consider algorithms that reuse samples for feature learning. \cite{mousavihosseini2023gradientbasedfeaturelearningstructured,NEURIPS2023_38a1671a,wang2024nonlinearspikedcovariancematrices} considers learning linear features with structured data (such as data with a spiked covariance) rather than the commonly considered isotropic one. \cite{cui2024asymptoticsfeaturelearningtwolayer,dandi2024randommatrixtheoryperspective} study the spectral structure revealed in the learned features with one huge gradient step through a spiked random feature model to understand the mechanism of feature learning in neural networks.

\paragraph{Learning Nonlinear Features.}

Previous studies indicate that neural networks can effectively learn specific hierarchies of nonlinear features. \citet{safran2022optimization} shows that GD can efficiently learn functions such as $\mathbf{1}_{\|\bx\| \geq \lambda}$ with a three-layer network. \citet{ren2023depth} demonstrates that $\mathrm{ReLU}(1 - \|\bx\|)$ can be learned by a multi-layer mean-field network. \cite{moniri2024theorynonlinearfeaturelearning} studies the nonlinear feature learning capabilities of two-layer neural networks with one step of gradient descent. \citet{allen2019can,allen2020backward} explore learning target functions of the form $p + \alpha g \circ p$ with $p$ being the underlying feature through a three-layer residual network, though they either need $\alpha=o_d(1)$ or cannot reach vanishing error. More recent works have addressed a broader class of nonlinear features compared with the previous research and demonstrate that three-layer neural networks can learn these hidden features efficiently. Specifically, \citet{nichani2023provable} demonstrates that a three-layer network trained with layer-wise GD algorithm effectively learns $g \circ p$ for a quadratic feature $p(\bx)=\bx^{\top}\bA\bx$ with an improved  sample complexity of $\widetilde\Theta(d^4)$. As a follow-up, \citet{wang2023learning} further demonstrates that such a network can in fact efficiently learn $g \circ p$ for $p$ within a broad subclass of degree $k$ polynomials and optimizes the sample complexity to $\widetilde{\cO}(d^k)$. 
}
However, all of these studies focus on a \textit{single nonlinear feature}, limiting their applicability to scenarios involving multiple features. Our work addresses this gap by establishing the first theoretical guarantee for efficiently learning hierarchical polynomials of \textit{multiple nonlinear features}, which significantly broadens the learnable function class and advances towards a better understanding of feature learning. 

    \section{Preliminaries}
\subsection{Notations}
We use bold letters to denote vectors and matrices. For a vector $\bv$, we denote its Euclidean norm by $\norm{\bv}{2}$. For a matrix $\bA$, we denote its operator and Frobenius norm as $\norm{\bA}{2}$ and $\norm{\bA}{\rm F}$, respectively. For any positive integer $n$, we denote $[n]=\{1,2,\dots,n\}$. Moreover, for any indexes $i$ and $j$, we denote $\delta_{ij}=1$ if $i=j$ and $0$ otherwise. We use  $\cO$, $\Theta$ and $\Omega$ to hide absolute constants. In addition, we denote $f\lesssim g$ when there exists some positive absolute constant $C$ with $f \leq Cg$. We use $\widetilde \cO$, $\widetilde\Theta$ and $\widetilde \Omega$ to ignore logarithmic terms. For a function $f: \mathcal{X}\rightarrow \RR$ and a distribution $v$ on $\mathcal{X}$, we denote $\norm{f}{L^p(\mathcal{X},v)}=\paren{\EE_{\bx\sim v}[\abs{f(\bx)}^p]}^{1/p}$. When the domain is clear from context, we write $\norm{f}{L^p}$ for simplicity. Finally, we write $\EE_{\bx}$ as the shorthand for $\EE_{\bx \sim v}$ sometimes.

\subsection{Problem Setup}

\paragraph{Data distribution} Our aim is to learn the target function $f^{\star}: \mathcal{X} \rightarrow \RR$, with $\mathcal{X}\subseteq \RR^{d}$ being the input space. Throughout the paper, we assume $\mathcal{X}= \mathbb{S}^{d-1}(\sqrt{d})$, that is, the sphere with radius $\sqrt{d}$ in $d$ dimensions. Also, we consider the data distribution to be the uniform distribution on the sphere, i.e.,  $\bx \sim {\rm Unif}(\cX)$, and we draw two independent datasets $\cD_1$, $\cD_2$,
each with $n_1$ and $n_2$ i.i.d. samples, respectively. Thus, we draw $n_1+n_2$ samples in total. 

\paragraph{Target function} For the target function $f^{\star}:\RR^{d}\rightarrow \RR$, we assume they are hierarchical functions of $r$ quadratic features
\begin{align*}
    f^{\star}(\bx)&=g^{\star}(\bp(\bx))=g^{\star}\left(\bx^{\top}\bA_{1}\bx,\bx^{\top}\bA_{2}\bx,\dots,\bx^{\top}\bA_{r}\bx\right).
\end{align*} 
This structure represents a broad class of functions where $\bp(\bx)=[\bx^{\top}\bA_{1}\bx,\bx^{\top}\bA_{2}\bx,\dots,\bx^{\top}\bA_{r}\bx]^{\top}$ represents $r$ quadratic features, and $g^{\star} :\RR^{r}\rightarrow \RR$ is a link function. Here we consider the case $r\ll d$. To simplify our analysis while maintaining generality, we make the following assumptions:
\begin{assumption}[Orthogonal quadratic features]\label{assump::features}
    For any $i,j \in[r]$, we suppose
    \begin{align*}
       \EE_{\bx}\brac{\bx^{\top}\bA_{i}\bx}=0,~~ \EE_{\bx}\brac{(\bx^{\top}\bA_{i}\bx)(\bx^{\top}\bA_{j}\bx)}=\delta_{ij} ~\text{ and }~ \norm{\bA_i}{op}\le \frac{\normA}{\sqrt{d}}.
    \end{align*}
    Here we assume $\normA={\rm poly}(\log d)$.
\end{assumption} The first assumption is equivalent to $\tr(\bA_i)=0$ for any $i\in[r]$. For $\bA_i$ such that $\tr(\bA_i) \neq 0$, we could simply subtract the mean of the feature to $\bA'_i=\bA_i-(\tr(\bA_i)/d)\cdot{\bI}$ so
 $$\bx^{\top}\bA'_i\bx=\bx^{\top}(\bA_i-(\tr(\bA_i)/d)\cdot{\bI})\bx=\bx^{\top}\bA_{i}\bx-\tr(\bA_i).$$ The second assumption on the feature orthonormality can be attained via linear transformation on the features, preserving the overall function class. The third assumption on the operator norm bound ensures that the features are balanced, which is common in the non-linear feature learning literature \citep{nichani2023provable,wang2023learning}. Moreover, we note that when the entries of $\bA_i$ are sampled i.i.d., the assumption is satisfied with high probability by standard random matrix arguments.
\begin{assumption}[Well-conditioned link function]\label{assump::target}
    For the link function $g^{\star}$, we assume $g^{\star}$ is a degree-$p$ polynomial with $\EE_{\bz}\brac{ g^{\star}(\bz)^2}=\Theta(1)$, where $\bz\sim \cN(\mathbf{0},\bI_r)$ and $p\in\mathbb{N}$ is a constant. Moreover, we assume the expected Hessian $\bH=\EE_{\bz}\brac{\nabla^2 g^{\star}(\bz)} \in \RR^{r\times r}$ is well-conditioned, i.e., there exists a constant $C_H$ such that $\lambda_{\min}(\bH) \ge \frac{C_H}{\sqrt{r}}$.
\end{assumption}
This assumption ensures the link function adequately emphasizes all $r$ features, preventing degeneracy to a lower-dimensional subspace. The second-moment condition is achievable through simple normalization.

\begin{assumption}[Prepocessed target function]\label{assump::preprocess target function}
    For the entire target function $f^{\star}$, we assume $\cP_{0}(f^{\star})=\EE_{\bx}[f^{\star}(\bx)]=0$ and $\norm{\cP_{2}(f^{\star})}{L^2}\le {\normP}/{\sqrt{d}}$. Here $\cP_{k}$ is the projection onto the function space of degree $k$ spherical harmonics on the sphere $\SS^{d-1}(\sqrt{d})$, and $\normP$ satisfies $\normP={\rm poly}(r,\log d)$.
\end{assumption}
We will give a rigorous definition of $\cP_{k}$ in Section \ref{sec::spherical}. This assumption is analogous to a preprocessing procedure conducted in \cite{damian2022neural}, which subtracts out the mean and linear component of the features from the target. The zero-mean condition ensures the network focuses on learning the function's variability rather than a constant offset. {{While  \cite{nichani2023provable,wang2023learning} assume the link function $g$ has non-zero linear component, we rather assume $g$ has a \textit{nearly zero linear component}, which prevents the target function from being dominated by a single linear combination of the quadratic features and keeps the learned representation space from collapsing to the one-dimenional space of that certain linear combination. This is an essential difference between single-feature and multi-feature learning, because  our assumptions ensure that the network genuinely learns to \textit{represent and distinguish all $r$ features} rather than conflate them, while assumptions in \cite{nichani2023provable,wang2023learning}  represent a \textit{degenerate case} that neural network may only learn the dominant linear combination of the  $r$ features. We provide examples and counterexamples as follows.}}

\begin{remark}
    These assumptions accommodate a wide range of target functions. For instance, $f^{\star}(\bx)=\frac{1}{\sqrt{r}}\sum_{k=1}^{r}\paren{\bx^{\top}\bA_{k}\bx}^2-\sqrt{r}$ satisfies Assumption \ref{assump::preprocess target function} with $\normP\lesssim \sqrt{r}\normA$ for any $\{\ba_k\}_{k\in[r]}$ under Assumption \ref{assump::features}. Moreover, for diagonal $\bA_{k}$ with $\bA_{k}={\rm diag}(\ba_k)$, where  $\ba_1,\ba_2,\dots,\ba_r$ are orthogonal zero-sum vectors with entries  $a_{k,i}\in \{\pm {c}/{\sqrt{d}}\}$, we can achieve $\normP=0$. Here $c=\Theta(1)$ is a normalizing constant. Notably, linear combinations of features like $f(
    \bx)=\frac{1}{\sqrt{r}}\sum_{k=1}^{r}\paren{\bx^{\top}\bA_{k}\bx}$ violate our assumptions, since  it represents a degenerate case with  $\norm{\cP_2(f)}{L^2}=\norm{f}{L^2}=\Theta(1)$.
\end{remark} 

\paragraph{Three-layer neural network} We adopt a standard three-layer neural network for learning the target functions. Let $m_1$, $m_2$ be the two hidden layer widths, and $\sigma_1$, $\sigma_2$ be two
activation functions. Our learner is a three-layer neural network parameterized by $\mathbf{\theta} = (\ba, \bW, \bb, \bV )$,
where $\ba \in \RR^{m_1}$, $\bW \in\RR^{m_1 \times m_2}$, $\bb \in \RR^{m_1}$, and $\bV \in \RR^{m_2 \times d}$. The network $f(\bx;\theta)$ is defined as
\begin{align}\label{equ::model}
f(\bx;\theta)=\frac{1}{m_1}\sum_{i=1}^{m_1}a_i\sigma_1\paren{\langle \bw_i,\sigma_2\paren{\bV\bx}\rangle+b_i}=\frac{1}{m_1}\sum_{i=1}^{m_1}a_i\sigma_1\paren{\langle \bw_i, \mathbf{h}^{(0)}(\bx)\rangle+b_i}.
\end{align}
Here, $\bw_i \in \RR^{m_2}$ is the $i$-th row of $\bW$ , and $\bh^{(0)}(\bx) := \sigma_2(\bV \bx)\in \RR^{m_2}$ is the random feature
embedding lying in the innermost layer. We initialize each row of $\bV$ to be drawn uniformly on the sphere of radius $\sqrt{d}$, i.e., $\bv_i^{(0)} \sim {\rm Unif}(\mathbb{S}^{d-1}(\sqrt{d}))$. For $\ba$, $\bb$ and $\bW$, we use a symmetric initialization so that $f(\bx;\theta^{(0)}) =0$ \citep{chizat2019lazy}. Explicitly, we assume that $m_1$ is an even number and for any $j \in [m_1/2]$, we initialize the paramters as
\begin{align*}
a_j^{(0)} = -a_{m_1-j}^{(0)} \sim {\rm Unif}\paren{{\{-1,1\}}},~~ \bw_j^{(0)} = \bw_{m_1-j}^{(0)}\sim\mathcal{N}(0,\epsilon\bI_{m_2}),~~\text{and}~~ b_j^{(0)} = b_{m_1-j}^{(0)}=0.
\end{align*}
Here $\epsilon>0$ is a hyperparameter to control the magnitude of the initial neurons. {{Different from \cite{nichani2023provable} where the weights $\bw_j$ are initialized at zeros, we require a \textit{random initialization}, which enables the learned weights to capture the multiple features in all directions instead of converging to a specific direction like the previous results for learning a single feature.}}

 For the activation functions $\sigma_1$ and $\sigma_2$, we have the following assumptions:
\begin{assumption}[Activation Function]\label{assump::activation}
    We take the outer activation function $\sigma_1$ and the inner activation function $\sigma_2$ as 
    \begin{align}\label{equ::activation}
        \sigma_1(t)=\left\{\begin{aligned} 2\abs{t}-1,~~~&\abs{t}\geq 1,\\
        t^2, ~~~~~~~~~&\abs{t}<1.
     \end{aligned}\right.~~~~\text{and} ~~~~\sigma_2(t)=\sum_{i=2}^{\infty}c_iQ_i(t),
    \end{align}
    where $Q_i(t)$ is the $i$-th degree Gegenbauer polynomial in the $d$-dimensional space. 
    Moreover, we assume there exist constants $C_{\sigma}$, $\alpha_{\sigma}$ such that $|\sigma_2(t)| \le C_{\sigma}$ for $\abs{t}\le {d}$, and $\EE_{\bx}\brac{{\sigma^{k}_2(\bx^{\top} \mathbf{1}_d)}}\le d^{-k}C_k$ for $k=2,4$. We assume $c_2=\Theta(1)$, and $C_2$, $C_4$ and $\set{c_i}_{i=2}^{\infty}$ are all constants independent of $n$, $d$, $m_1$ and $m_2$. 
\end{assumption}
We remark the outer activation $\sigma_1$ is a slightly modified version of the absolute value function $\abs{t}$, smoothed around the origin. The assumptions on $\sigma_2$ are based on the Gegenbauer expansion, often considered in the spherical analysis (introduced in Section \ref{sec::gegenbauer}). Compared to standard inner activations, we remove the constant term ($Q_0(t)=1$) and the linear term ($Q_1(t)=t/d$) to focus on learning nonlinear features without low-order interference. Importantly, these assumptions on activation functions maintain significant generality. The assumptions on magnitude and moments are satisfied by many common activation functions with appropriate scaling. {The core assumption in the Gegenbauer expansion is the non-zero component of $Q_2$, i.e., $c_2=\Theta(1)$, which we rely on for a subspace recovery of the $r$ quadratic features while other assumptions are made to simplify our analysis since other components in inner activation will lead to useless noises or biases in the weights after training. Moreover, if we consider higher degree nonlinear features such as degree $q$ polynomials, we expect that $\sigma_2$ has sufficient emphasis on $Q_q$ for efficient feature learning.}

 \begin{remark}
     $\sigma_2(t)=Q_2(t)=\frac{t^2-d}{d(d-1)}$ is an example of the inner activation function.
 \end{remark}

\paragraph{Training Algorithm} Following \citet{nichani2023provable}, our network is trained via layer-wise gradient descent with
sample splitting. Throughout the training process, we freeze the innermost layer weights $\bV$. In the first stage, the second layer weights $\bW$ are trained for one step with a specified learning rate $\eta_1$ and weight decay $\lambda_1$. In the second stage, we reinitialize the bias $\bb$ and train the outer layer
weights $\ba$ for $T-1$ steps. 
\paragraph{Transfer Learning} We remark that our algorithm allows transfer learning of a different target function $f$ that shares the same features of the original target:
\begin{align}  f^{\star}(\bx)\rightarrow f(\bx)=g\left(\bx^{\top}\bA_{1}\bx,\bx^{\top}\bA_{2}\bx,\dots,\bx^{\top}\bA_{r}\bx\right) \tag{transferred target}
\end{align}

In this case, we switch the target function from $f^{\star}=g^{\star}(\bp)$ to $f=g(\bp)$ in the second training stage. For the loss function, we use the standard squared loss:
\begin{align*}
    \hat{\cL}^{(1)}(\theta)&=\frac{1}{n_1}\sum_{\bx \in \cD_1} \paren{f(\bx;\theta)-f^{\star}(\bx)}^2,~~
    \hat{\cL}^{(2)}(\theta)=\begin{cases}
    \frac{1}{n_2}\sum_{\bx \in \cD_2} \paren{f(\bx;\theta)-f^{\star}(\bx)}^2~~\text{(original)},\\
        \frac{1}{n_2}\sum_{\bx \in \cD_2} \paren{f(\bx;\theta)-f(\bx)}^2~~~\text{(transferred)}.   
    \end{cases}
\end{align*}
This layer-wise training approach, combined with the ability to perform transfer learning, provides a powerful framework for learning and adapting to hierarchical functions with hidden features \citep{kulkarni2017layerwisetrainingdeepnetworks,damian2022neural,nichani2023provable}. The pseudocode for the entire training procedure is presented in Algorithm \ref{alg:: training algo}.

\begin{algorithm}
	\SetKwInOut{Input}{Input}
 \SetKw{Initialize}{initialize}
	\SetKwBlock{TrainW}{train $\bW$ on dataset $\cD_1$}{end}
 \SetKwBlock{Traina}{train $\ba$ on dataset $\cD_2$}{end}
	\SetKwBlock{Reinitialize}{re-initialize}{end}
	
	\Input{Learning rates $\eta_1,\eta_2$, weight decay $\lambda_1,\lambda_2$, parameter $\epsilon$, number of steps $T$}
 \Initialize{$\ba ,\bb,  \bW \text{ and } \bV.$}
	
	\TrainW{
	$\bW^{(1)} \leftarrow \bW^{(0)} - \eta_1 [\nabla_{\bW} \hat{\cL}^{(1)}(\theta) + \lambda_1 \bW^{(0)}]$
 }
	
	\Reinitialize{~~$b_{i}^{(1)} \sim {\rm Unif}({[-3,3]}), ~i\in[m_1]$\\
 $\ba^{(1)}, \bV^{(1)}\leftarrow\ba^{(0)},\bV^{(0)}$\\$\theta^{(1)} \leftarrow (\ba^{(1)},\bW^{(1)},\bb^{(1)},\bV^{(0)})$}
	\Traina{
	\For{$t=2$ \KwTo $T$}{
		$\ba^{(t)} \leftarrow \ba^{(t-1)} - \eta_2 [\nabla_{\ba} \hat{\cL}^{(2)}(\theta^{(t-1)}) + \lambda_2 \ba^{(t-1)}]$
	}}
	
	\Return Prediction function $f(\cdot;\theta^{(T)})$: $\bx \to  \frac{1}{m_1}
 \langle\ba^{(T)}, \sigma_1(\bW^{(1)} \mathbf{h}^{(0)}(\bx) + \bb^{(1)})\rangle$
	\caption{Layer-wise training algorithm}\label{alg:: training algo}
\end{algorithm}

\subsection{Technical Background: Analysis Over the Sphere}\label{sec::technical bg}
We briefly introduce spherical harmonics and Gegenbauer polynomials, which forms the foundation of our analysis over the sphere $\mathbb{S}^{d-1}(\sqrt{d})$. For more details, see Appendix \ref{appendix::mathsphere}.
\subsubsection{Spherical Harmonics}\label{sec::spherical}
Let $\tau_{d-1}$ be the uniform distribution on $\mathbb{S}^{d-1}(\sqrt{d})$. Consider functions in $L^2(\mathbb{S}^{d-1}(\sqrt{d}),\tau_{d-1})$, with scalar product and norm denoted as $\langle \cdot,\cdot \rangle_{L^2}$ and $\norm{\cdot}{L^2}$.
For $\ell \in \mathbb{Z}_{\ge 0}$, let $V_{d,\ell}$ be the linear space of homogeneous harmonic polynomials of degree $\ell$ restricted on $\mathbb{S}^{d-1}(\sqrt{d})$. The set $\{V_{d,\ell}\}_{\ell \ge 0}$ forms an orthogonal basis of the $L^2$ space, with dimension ${\rm dim}(V_{d,\ell})=\Theta(d^{\ell})$. For each $\ell \in \mathbb{Z}_{\ge 0}$, the spherical harmonics $\{Y_{\ell,j}\}_{j\in[B(d,\ell)]}$ form an orthonormal basis of $V_{d,\ell}$. Moreover, we denote by $\cP_k$ the orthogonal projections to $V_{d,k}$, which can be written as
\begin{align*}
    \cP_k(f)(\bx)=\sum_{\ell=1}^{B(d,k)}\langle f ,Y_{k,\ell}\rangle_{L^2}Y_{k,\ell}(\bx). 
\end{align*}
We also define $\cP_{\le \ell}\equiv \sum_{k=0}^{\ell}\cP_{k}$, $\cP_{> \ell}\equiv \bI-\cP_{\le \ell}$, $\cP_{< \ell}\equiv \cP_{\le \ell-1} $,  and $\cP_{\ge \ell}\equiv \cP_{> \ell-1} $.
\subsubsection{Gegenbauer Polynomials}\label{sec::gegenbauer}
Corresponding to the degree $\ell$ spherical harmonics in the $d$-dimension space, the $\ell$-th Gegenbauer polynomial $Q_\ell:[-d,d]\rightarrow \RR$ is a polynomial of degree $\ell$. The set $\{Q_\ell\}_{\ell \ge 0}$ forms an orthogonal basis on $L^2([-d,d],\widetilde{\tau}_{d-1})$, where $\widetilde{\tau}_{d-1}$ is the distribution of $\sqrt{d}\langle \bx,\be_1\rangle$ when $\bx \sim \tau_{d-1}$. In particular, these polynomials are normalized so that  $Q_{\ell}(d)=1$. We present the explicit forms of Gegenbauer polynomials of degree no more than $2$:
\begin{align*}
    Q_0(t)=1,~~~Q_1(t)=\frac{t}{d},~~\text{and}~~Q_2(t)=\frac{t^2-d}{d(d-1)}.
\end{align*}

Gegenbauer polynomials are directly related to spherical harmonics, leading to a number of elegant properties. We provide further details on these properties in Appendix \ref{appendix::mathsphere}.

\section{Main results}
The following is our main theorem, which bounds the population absolute loss of Algorithm \ref{alg:: training algo}:
\begin{theorem}\label{thm:main_thm}
Suppose $n_1 ,m_2 =\widetilde \Omega(d^4)$. Let $\hat \theta$ be the output of Algorithm \ref{alg:: training algo} after $T={\rm poly}(n_1,n_2,m_1,m_2,d)$ steps. Then, there exists a set of hyper-parameters $(\epsilon, \eta_1, \eta_2, \lambda_1, \lambda_2)$ such that, with high probability over the initialization of parameters and draws of $\mathcal{D}_1, \mathcal{D}_2$,  we have
\begin{align*}
\mathbb{E}_{\bx}\brac{\abs{f(\bx; \hat \theta) - f^{\star}(\bx)}}
&=\widetilde \cO \paren{ \underbrace{\sqrt{\frac{ r^{p} \normP^{2p}}{{\min(m_1,n_2)}} }}_{\text{Complexity of } g^{\star}}
+\underbrace{\sqrt{\frac{d^6 r^{p+1}}{m_2}}+ \sqrt{\frac{d^2 r^{p+1}}{n_1}} +\frac{ r^{p+2} }{d^{1/6}} }_{\text{Feature Learning Error}}}.
\end{align*}
Moreover, for any other degree $p$ polynomial $g:\RR^{r}\rightarrow \RR$ with $\norm{g}{L^2}\lesssim 1$, by substituting the target function $f^{\star}=g^{\star}\circ\bp$ by $f=g\circ\bp$ in the second training stage, we can achieve the same result for learning the new target function.
\end{theorem}
The full proof is provided in Appendix \ref{appendix::generalization}. To interpret the results, we provide the following discussion of Theorem \ref{thm:main_thm}. 
\paragraph{Feature learning error} This terms quantifies the  requirements on the first-stage sample complexity and the inner width to sufficiently capture the non-linear features. Given $d\gg r$, if the width $m_2=\widetilde\Omega(d^{6}r^{p+1})$ and the sample size $n_1=\widetilde\Omega(d^4+d^2r^{p+1})$, we can fully capture the underlying feature information and approximate any degree $p$ polynomials of the features.  We will demonstrate how Algorithm \ref{alg:: training algo}  learns these features through the learned representations in Proposition \ref{prop::reconstructed feature main} and express hierarchical polynomials in Proposition \ref{prop::construct random feature main}.

\paragraph{Complexity of $g^{\star}$} This term is the second-stage sample (and width) complexity given that the $r$ features have been fully captured in the first stage. Moreover, for a sufficiently preprocessed target function, i.e., $\normP=\cO(1)$,  we achieve the standard results of $\widetilde\cO(r^{p})$ complexity in learning a degree-$p$ polynomial in the $r$-dimensional space in the kernel regime.

\paragraph{Transfer learning} Leveraging the two-stage structure of training, we can learn a different target function in the second stage that shares the same features with the original target. This also supports the fact that we have fully captured the information of the $r$ nonlinear features in the first stage, making it possible for the efficient learning with a different polynomial head $g$. Moreover, by viewing the first stage as a pre-training process with $\widetilde\Omega(d^4+d^2r^{p+1})$ samples, only additional $\widetilde\cO(r^p\normP^{2p})$ samples are required to learn any degree $p$ polynomial of the features, which gets rid of the polynomial dependence on the ambient dimension of $d$.

\paragraph{Comparison with previous works} Compared with the sample complexity of $\widetilde\Omega(d^2r+dr^{p})$ in \citet{damian2022neural} for learning multi-index models, we have a similar polynomial dependence on $r$, and the dependence on $d$ increases from $d^2$ to $d^4$ because of the increased complexity of quadratic features rather than linear ones. Moreover, our approach significantly improves upon the $\Theta(d^{2p})$ sample complexity required by kernel methods to learn degree $p$ polynomials of quadratic features (i.e., degree $2p$ polynomials of the input). Crucially, our polynomial dependence on $d$ in the overall sample complexity is independent of the degree $p$ of the link function $g$.

{\paragraph{Near optimality of the sample complexity} We remark that our sample complexity of $\widetilde{\cO}({d^4})$ is nearly optimal with respect to $d$ for all algorithms that use one step of gradient descent for feature learning. Our assumptions on the target functions imply that the leap index\footnote{The leap index of a target function $f^{\star}$ is the first integer $\ell$ that $\mathcal{P}_{\ell}{f^{\star}}\neq 0$. Our assumptions imply a diminishing $\mathcal{P}_{<4}({f^{\star}})$ and a non-degenerate $\mathcal{P}_{4}({f^{\star}})$ as $d \rightarrow \infty$.} of our target functions  are basically 4 (more specifically, the second order information of $g\circ\bp$, where $\bp$ are quadratic features), and we also utilize $\mathcal{P}_4(f)$ for recovering the subspace of the $r$ quadratic features, which will be discussed in details in Section \ref{sec::proof sketch}. \cite{dandi2023twolayerneuralnetworkslearn} indicates that $\Omega(d^4)$ samples are required for an efficient learning of terms in $\mathcal{P}_4(f^{\star})$, which substantiates the near optimality of our result.}
    

\section{Proof Roadmap of Theorem \ref{thm:main_thm}}\label{sec::proof sketch}
The proof of Theorem \ref{thm:main_thm} unfolds in two training stages. First, by a novel universality argument (Lemma \ref{prop::extend chaterjee}), we show that after the first training stage, with sufficient training samples, the network learns to fully extract out the hidden features $\bp$
(Proposition \ref{prop::reconstructed feature main}). Next, we show that during the second stage, the network is capable of expressing the link function with a mild outer width $m_1$ (Proposition \ref{prop::construct random feature main}). We conclude the proof through standard Rademacher complexity analysis to quantify the generalization error of the second-stage model (detailed in Appendix \ref{appendix::generalization}).

\subsection{Stage 1: Learning the Features}
We provide a brief analysis on the learned representations after the first training stage. Denote $\bw_j=\epsilon^{-1}\bw_j^{(0)} \sim \cN(0,\bI_{m_2})$. According to Algorithm \ref{alg:: training algo}, by setting $\epsilon$ sufficiently small, after one-step gradient descent on $\bW$, we know for each $j\in[m_1]$,
\begin{align*}
\eta_1\nabla_{\bw_j^{(0)}} \cL(\theta^{(0)}) &= -\eta_1 \frac{a^{(0)}_j}{m_1}\cdot \frac{1}{n_1}\sum_{\bx\in\cD_1}f^*(\bx_i)\bh^{(0)}(\bx_i)\sigma_1'\left(\langle \epsilon\bw_j, \bh^{(0)}(\bx_i)\rangle\right)\\
&\underset{\epsilon \rightarrow 0}{\rightarrow} -\frac{2\epsilon\eta_1}{m_1} {a^{(0)}_j} \cdot {\frac{1}{n_1}\sum_{\bx\in\cD_1} f^*(\bx_i)\bh^{(0)}(\bx_i)\bh^{(0)}(\bx_i)^{\top}}\bw_j.
\end{align*}
By taking $\eta_1=\frac{m_1}{2\epsilon m_2}\cdot \eta$ for some $\eta>0$ to be chosen later and $\lambda_1=\eta_1^{-1}$, we have
\begin{align*}
\bw^{(1)}_j &= \bw^{(0)}_j -\eta_1 \left[\nabla_{\bw_j^{(0)}} \cL(\theta^{(0)}) + \lambda_1 \bw^{(0)}_j\right]
= \frac{\eta a^{(0)}_j}{m_2}\cdot\frac{1}{n_1}\sum_{\bx\in\cD_1} f^*(\bx_i)\bh^{(0)}(\bx_i)\bh^{(0)}(\bx_i)^{\top}\bw_j .
\end{align*}
Then for any second-stage training sample $\bx'\in\cD_2$, the inner-layer representation becomes
\begin{align*}
    \left\langle \bw_j^{(1)}, \sigma_2(\bV \bx') \right\rangle&=\frac{\eta a^{(0)}_j}{m_2}\left\langle \frac{1}{n_1}\sum_{\bx\in\cD_1} f^*(\bx_i)\bh^{(0)}(\bx_i)\bh^{(0)}(\bx_i)^{\top}\bw_j,\bh^{(0)}(\bx')  \right\rangle\\
    &=\eta a^{(0)}_j\cdot  \left\langle\bw_j,\underbrace{\frac{1}{n_1m_2}\sum_{\bx\in\cD_1}f^{\star}(\bx_i){\langle \bh^{(0)}(\bx_i),\bh^{(0)}(\bx')\rangle}\bh^{(0)}(\bx_i)}_{\bh^{(1)}(\bx')}\right\rangle.
\end{align*}
  Our main contribution in this part is that the first-step trained presentations representations $\mathbf{h}^{(1)}(\bx)$ approximately spans the space of the target features $(\bx^{\top}\bA_1\bx,\bx^{\top}\bA_2\bx,\dots,\bx^{\top}\bA_r\bx)$.
  Thus, the target features $\bp(\bx)$ can be reconstructed through a linear transformation from the learned representations $\mathbf{h}^{(1)}(\bx)$, which is formalized in the following proposition.

\begin{proposition}[Reconstruct the feature]\label{prop::reconstructed feature main}
   Suppose $m_2,n_1= \widetilde \Omega(d^{4})$. With high probability jointly on $\bV$ , $\cD_1$ and $\cD_2$, there exists a matrix $\bB^{\star}\in \RR^{r\times m_2}$ such that for any $\bx \in \cD_2$, we have
   \begin{align}\label{equ::prop 1 error}
       \norm{\bB^{\star}\mathbf{h}^{(1)}(\bx)-\bp(\bx)}{\rm 2} =\widetilde\cO \Bigg(\frac{d^3r}{\sqrt{m_2}}+ \frac{dr}{\sqrt{n_1}} +\frac{r^{\frac{p+5}{2}}}{d^{1/6}} \Bigg).
   \end{align}
\end{proposition}
The proof is provided in Appendix \ref{appendix proof reconstruct whole}. We summarize the main idea of the proof as follows.

\begin{figure}

    \centering
    \includegraphics[width=\linewidth]{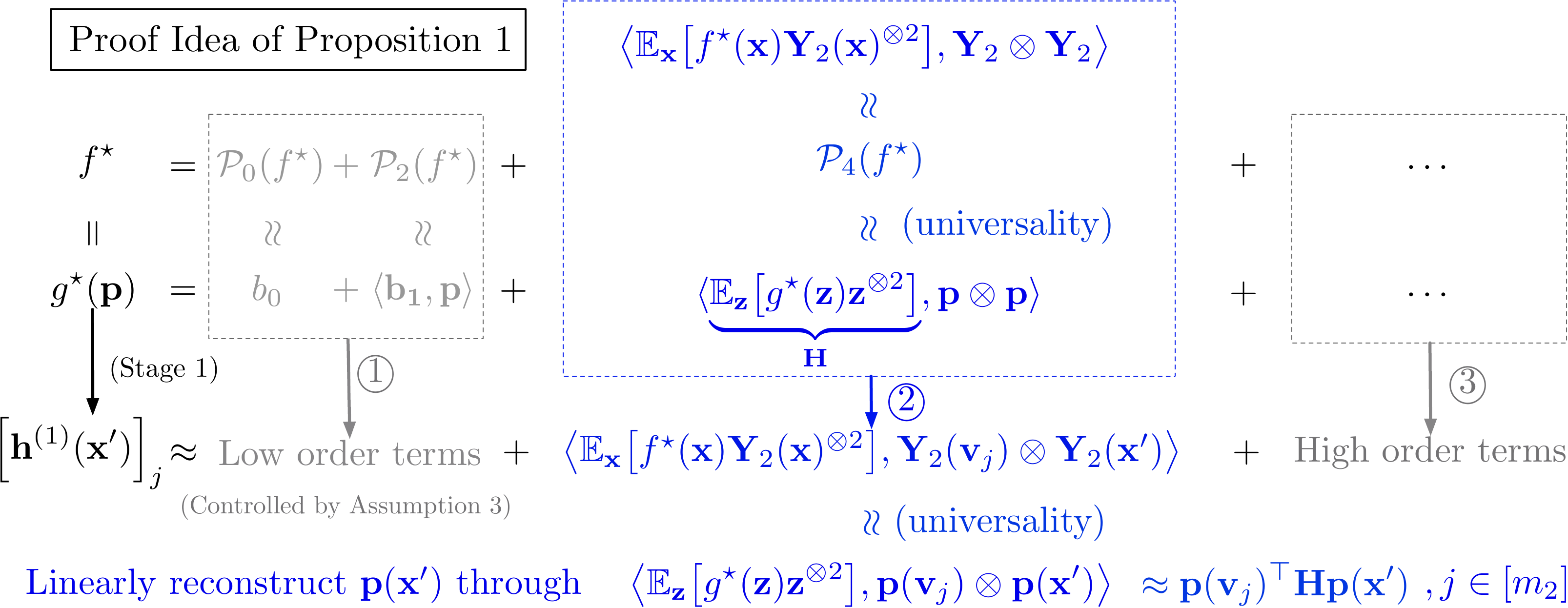}
    \caption{The proof idea of Proposition \ref{prop::reconstructed feature main}. { Block 1 characterizes the constant and linear terms of $g^{\star}$, which is approximately equivalent to the low-order terms $\cP_{< 4}(f^{\star})$ by our universality theory and results into biases in the learned weights $\bh^{(1)}(\bx'
    )$ after Stage 1. This bias is vanishing with $d \rightarrow \infty$ by our assumptions on $\cP_0({f^{\star}})$ and $\cP_2(f^{\star})$. Block 2 describes the second-order information of $g^{\star}$ (approximately $\cP_4(f^{\star})$), which is of the greatest importance and captured by the quadratic component $c_2Q_2(\cdot)$ in the inner activation $\sigma_2(\cdot)$ and converted into quantities spanned by the $r$ quadratic features $\bp$. Block 3 represents the remaining terms of $f^{\star}$, which leads to high-order nuisance in the learned weights, but still dominated by the second term due to Assumption \ref{assump::target} when $d$ is large, which enables us to utilize the terms in blue (resulted from Block 2) to reconstruct the features efficiently.}}  \label{figure::proof idea}
\end{figure}
\paragraph{Universality of features} The foundation of the proof lies in the universality result that the joint distribution of the multiple features $\bp$ is approximately multivariate standard Gaussian:
\begin{align*}
\left(\bx^{\top}\bA_{1}\bx,\bx^{\top}\bA_{2}\bx,\dots,\bx^{\top}\bA_{r}\bx\right) \overset{\rm d}{\approx} \cN \paren{\mathbf{0}_r,\bI_r},~~~d\gg r.
\end{align*}
It is worth mentioning that we provide a general universality theory that quantifies the difference between the distribution of any $r$-dimensional function (not limited in quadratic forms) and the $r$-dimensional Gaussian distribution, which is presented in Lemma \ref{prop::extend chaterjee}.
\begin{lemma}[Universality of vector-valued functions]\label{prop::extend chaterjee}
Suppose $\bX\sim \cN(\bzero,\bI_{d})$ is an $d$-dimensional standard Gaussian variable. If a function  $\mathbf{p}:\RR^{d}\rightarrow\RR^{r}$ satisfies $\EE_{\bX}\brac{\bp(\bX)}=\mathbf{0}_r$ and $\text{Cov}\paren{\bp(\bX),\bp(\bX)}=\bI_r$, then we have
\begin{align*}
    W_1(\text{Law}(\bp(\bX)),\cN(\bzero,\bI_{r}))\le \frac{4}{\sqrt{\pi}} \paren{\sum_{i=1}^{r}\EE\brac{\norm{\nabla p_i(\bX)}{2}^{4}}^{1/4} }\paren{\sum_{j=1}^{r}\EE\brac{\norm{\nabla^2 p_j(\bX)}{\rm op}^{4}}^{1/4}}.
\end{align*}
Here $\bp(\bx)=[p_1(\bx),p_2(\bx),\dots,p_r(\bx)]^{\top}$ and $W_1$ denotes the Wasserstein-$1$ distance.
\end{lemma}
The proof is provided in Appendix \ref{appendix::mathstein}. This lemma extends the previous universality results of univariate Gaussian approximation theory \citep{chatterjee2007fluctuationseigenvaluessecondorder} to the multivariate version and could be of independent interest for the field of high dimensional probability theory. As a corollary, when we take $\bp$ to be $r$ quadratic features satisfying Assumption \ref{assump::features}, we ensure the $W_1$ distance is bounded by $\widetilde\cO(r^2/\sqrt{d})$ (see Lemma \ref{lemma::Gaussian W1} in the appendix for the formal statement). This approximation error finally contributes to third term in the error bound of Proposition \ref{prop::reconstructed feature main} (Equation \eqref{equ::prop 1 error}).

\paragraph{Utilizing the second-order information of $g^{\star}$} Lemma \ref{prop::extend chaterjee} establishes a crucial link between our model and the multi-index model studied by \citet{damian2022neural}. This connection allows us to simplify the analysis on non-linear features and utilize the second-order information of the link function $g^{\star}$ to fully recover the feature space. In the context of multi-index models where $f^{\star}(\bx) = g^{\star}(\bp(\bx))$ with $\bp(\bx) =\bU\bx$, it has been shown that for a prepossessed target with a non-degenerate expected Hessian $\bH=\EE_{\bz}\brac{\nabla^2 g^{\star}(\bz)}$,  the learned representations, dominated by the degree $2$ component of $f^*$ which takes form $\EE_{\bx}\brac{f^{\star}(\bx)\bx^{\otimes 2}}\approx \bU^{\top} \bH \bU$, are spanned by $\{\bu_i \otimes \bu_j\}_{i,j\in[r]}$. Extending this to our setting with quadratic features and applying the universality argument from Lemma \ref{prop::extend chaterjee}, we demonstrate that the degree $4$ component of our $f^*$, namely $\mathbb{E}_{\bx}\left[f^*(\bx)\bY_2(\bx)^{\otimes 2}\right]$, is approximately spanned by the quantities $\{\bA_i \otimes \bA_j\}_{i,j\in[r]}$, which is formalized in Proposition \ref{prop::approximate stein lemma quad} in Appendix \ref{appendix::asymptotic analysis inner feature reconstruct}. Here $\bY_2(\bx)$ represents the tensorized quadratic spherical harmonics. Under Assumption \ref{assump::preprocess target function}, it turns out that after the first step of GD (Stage 1 of Algorithm \ref{alg:: training algo}), the learned representations are dominated by this degree $4$ component (Proposition \ref{prop::bounded feature} in Appendix \ref{sec::bounded feature}). This domination enables efficient recovery of the "span" of the hidden features $\bp$. For a visual representation of our proof strategy, we also present our main idea of the proof in Figure \ref{figure::proof idea}. Remarkably, we find that the reconstruction matrix admits a surprisingly simple form of $\bB^{\star}\propto \bH^{-1}[\bp(\bv_1),\bp(\bv_2),\dots,\bp(\bv_{m_2})]$. We provide empirical support for the effectiveness of this reconstruction through experiments in Section \ref{sec::experiments}.



 \subsection{Stage 2: Learning the Link Function}
 By the deduction above, after the first training stage, the model becomes a random-feature model \citep{NIPS2007_013a006f}:
\begin{align}
    f(\bx';\theta)=\frac{1}{m_1}\sum_{j=1}^{m_1}a_j\sigma_1\paren{\eta a^{(0)}_j \langle \bw_j, \mathbf{h}^{(1)}(\bx') \rangle+b^{(1)}_j}. \label{equ::modelafterGD}
\end{align}
Here $\theta=(\ba,\bW^{(1)},\bb^{(1)},\bV)$, with $\ba=[a_1,a_2,\dots,a_{m_1}]^{\top}\in \RR^{m_1}$ being the trainable parameters in the second stage. Leveraging the construction in Proposition \ref{prop::reconstructed feature main}, we can construct a corresponding weight vector $\ba$ in the outer layer to express the polynomial $g(\bB^{\star}\mathbf{h}^{(1)}(\bx))\approx g(\bp(\bx))$.
\begin{proposition}[Expressivity of the second-stage model]
\label{prop::construct random feature main}
Suppose $g$ is a degree $p$ polynomial with $\norm{g}{L^2}\lesssim 1$. Then there exists a learning rate $\eta$ such that, with high probability over $\cD_1$, $\cD_2$, $\bW$ and $\bV$, there exists $\ba^{\star}\in\RR^{m_1}$ such that the parameter $\theta^{\star}=(\ba^{\star},\bW^{(1)},\bb^{(1)},\bV)$ achieves a small empirical loss:
\begin{align*}
\frac{1}{n_2}\sum_{\bx \in \cD_2} \paren{f(\bx;\theta^{\star})-g(\bp(\bx))}^2&= \widetilde\cO\paren{\frac{\norm{\ba^{\star}}{2}^2}{m^2_1}+ \frac{d^6r^{p+1}}{m_2}+ \frac{d^2r^{p+1}}{n_1} +\frac{r^{{2p+4}} }{d^{1/3}}}.
\end{align*}
Here $\ba^{\star}$ satisfies ${\norm{\ba^{\star}}{2}^2 }=\widetilde \cO\left(m_1r^p\normP^{2p}\right).$
\end{proposition}
The proof is provided in Appendix \ref{appendix::proof outer rf main}. We provide following discussions. 

\paragraph{Error propagation} To explain the increased polynomial dependence on $r$, we remark that the approximation error in Proposition \ref{prop::reconstructed feature main} gets multiplied by the averaged Lipschitz smoothness of the link function $g$, which is upper bounded by $\cO(r^{\frac{p-1}{2}})$. This product is then squared due to the use of squared loss in Proposition \ref{prop::construct random feature main}. 

\paragraph{Reduced complexity of $\ba$} Moreover, we remark that the complexity of $\ba$, i.e., $\norm{\ba}{2}$, gets rid of the polynomial dependence on $d$, which is greatly reduced compared with a naive random-feature model that requires $\norm{\ba}{2}^2=\Theta(m_1d^{2p})$.  This directly saves the second-stage sample complexity $n_1$ and the outer width $m_1$, since $n_1,m_2 =\Theta(m_1^{-1}\norm{\ba^{\star}}{2}^2)$ is required for efficient approximation and generalization \citep{ghorbani2021linearized}. We also examine this reduced dependency by comparing our model with a naive random feature model in learning hierarchical target functions in Section \ref{sec::experiments}. 

\paragraph{Arbitrariness of $g$} Thanks to the two-stage architecture and the sufficient learning of the features, the choice on the link function $g$ can be an arbitrary degree $p$ polynomial, not limited to the truth target $g^{\star}$. This allows us to conduct transfer learning tasks in Stage 2 of Algorithm \ref{alg:: training algo}.

  Finally, by standard Rademacher complexity analysis on the random feature model presented in Appendix \ref{appendix::generalization}, we conclude our proof.

\section{Numerical Experiments}\label{sec::experiments}
We empirically verify Theorem \ref{thm:main_thm} and Proposition \ref{prop::reconstructed feature main}. We consider learning functions with $r=3$ quadratic features. Regarding the target function, we choose the target functions to be of the form
\begin{align}\label{equ::target exp}
    f^{\star}_{d,p}(\bx)= \frac{f_{d,p}(\bx)-\EE\brac{f_{d,p}(\bx)}}{\sqrt{{\rm Var}[f_{d,p}(\bx)]}},~~~\text{with}~~f_{d,p}(\bx)=\sum_{i=1}^{r}\paren{\bx^{\top}\bA_1\bx}^p,~p\in \mathbb{N}.
\end{align}

 For the underlying features, we take $\bp(\bx)=[\bx^{\top}\bA_1\bx,\bx^{\top}\bA_2\bx,\bx^{\top}\bA_3\bx]^{\top}$  with $\bA_k=\diag{c\cdot \ba_k}$,  and $c>0$ is a normalizing constant. To ensure the orthogonality of the features and $\tr(\bA_k)=0$, we choose the ambient dimension $d$ to be divisible by $4$ and take $\ba_k$ to be
\begin{align*}
    \ba_1=\Vec{\paren{[\mathbf{1},\mathbf{1},-\mathbf{1},-\mathbf{1}] }},~\ba_2=\Vec{\paren{[\mathbf{1},-\mathbf{1},\mathbf{1},-\mathbf{1}] }},~\text{and}~~\ba_3=\Vec{\paren{[\mathbf{1},-\mathbf{1},-\mathbf{1},\mathbf{1}] }}.
\end{align*}
Here $\mathbf{1}$ is a vector of ones in $d/4$ dimensions, and $c=\sqrt{\frac{d+2}{2d^2}}$ to ensure that $\EE_{\bx}\brac{(\bx^{\top}\bA_k\bx)^2}=1$ for each $k=1,2,3$.

For the network architecture, we choose $\sigma_1$ as per \eqref{equ::activation} and $\sigma_2=Q_2$, with network sizes set to $m_1 = 10000$ and $m_2 = 20000$. We compare our proposed model \eqref{equ::modelafterGD} (given by Algorithm \ref{alg:: training algo}) against the naive random-feature model defined as
\begin{align}
    f^{\RF}(\bx';\theta)=\frac{1}{m_1}\sum_{j=1}^{m_1}a_j\sigma_1\paren{\eta a^{(0)}_j \langle \bw_j, \mathbf{h}^{(0)}(\bx') \rangle+b^{(1)}_j}, \label{equ::RF}
\end{align}
where $\ba$ is the only trainable parameter throughout the training process. Our experiments involve learning $f^{\star}_{d,p}$ with $p=4$ and $d \in \{8,16,32\}$. To examine our model's transfer learning capabilities, we also train the model on an initial target function $f^{\star}_{d,2}$ with $d=16$ and $n_1=2^{16}$ in the first stage, then transfer to targets $f^{\star}_{d,p}$ with $p=4,6,8$. For each task, we explore a range of sample sizes from $2^8$ to $2^{16}$. The results of these experiments are presented in Figure \ref{fig::Comparison m2=}. 

\paragraph{Improved sample complexity and Polynomial dependence on $d$} The left panel of Figure \ref{fig::Comparison m2=} demonstrates that our model outperforms the naive random-feature model across all dimensions. As the dimension $d$ increases, both models show larger test errors, but our model exhibits less sensitivity to $d$. This aligns with our theoretical analysis in Theorem \ref{thm:main_thm} that the sample complexity of kernel methods should be $\Omega(d^{2p-4})$ times greater than that of our model. Moreover, we redraw Figure \ref{fig::Comparison m2=} by plotting the test error against $\log_d n$. As shown in Figure \ref{fig::redraw comparison}, the loss curves for our model (Algorithm \ref{alg:: training algo}) align closely for different values of $d$, indicating that it achieves low error rates with only $\widetilde{\cO}(d^4)$ samples. In stark contrast, the naive random feature model exhibits significant separation between curves for different $d$ values, requiring more than $\widetilde{\cO}(d^{4})$ samples to achieve comparable error rates. This graphical evidence powerfully demonstrates how our approach eliminates the dependence on dimension $\Theta(d^{2p})$ presented in kernel methods, resulting in substantially improved sample complexity in high-dimensional settings.


\paragraph{Efficient transfer learning} The right panel of Figure \ref{fig::Comparison m2=} showcases our algorithm's strong transfer learning capabilities. our algorithm successfully learns all three transferred target functions with benign second-stage sample complexity. Notably, as the degree $p$ increases, the test error grows no faster than $r^{p}$, which is significantly slower than $d^{2p}$. This supports our theoretical result that the second-stage sample complexity depends on the number of features $r$ rather than the ambient dimension $d$, underscoring our model's strong transfer learning capabilities.


\begin{figure}
    \hspace{-2.5em}
    \subfigure[Comaprison between Algorithm \ref{alg:: training algo} and the random-feature model]{\includegraphics[width=0.75\linewidth]{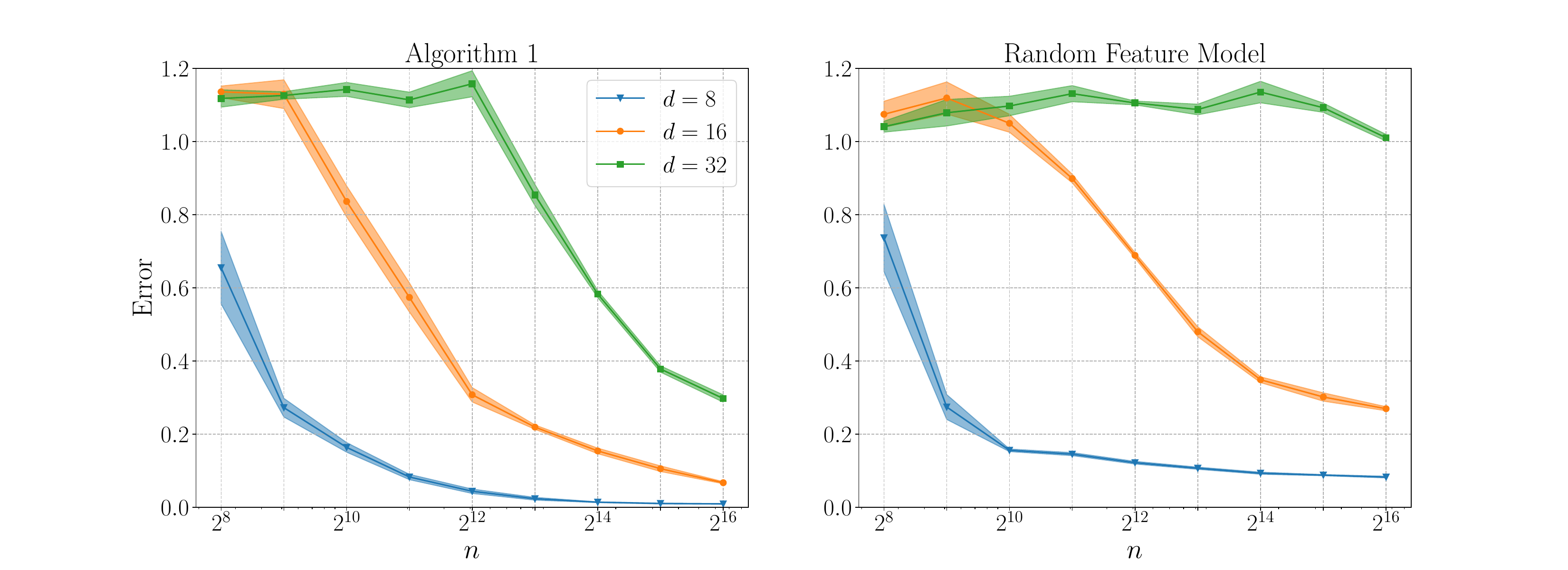}}\hspace{-3.2em}
    \subfigure[Performance of transfer learning]{\includegraphics[width=0.38\linewidth]{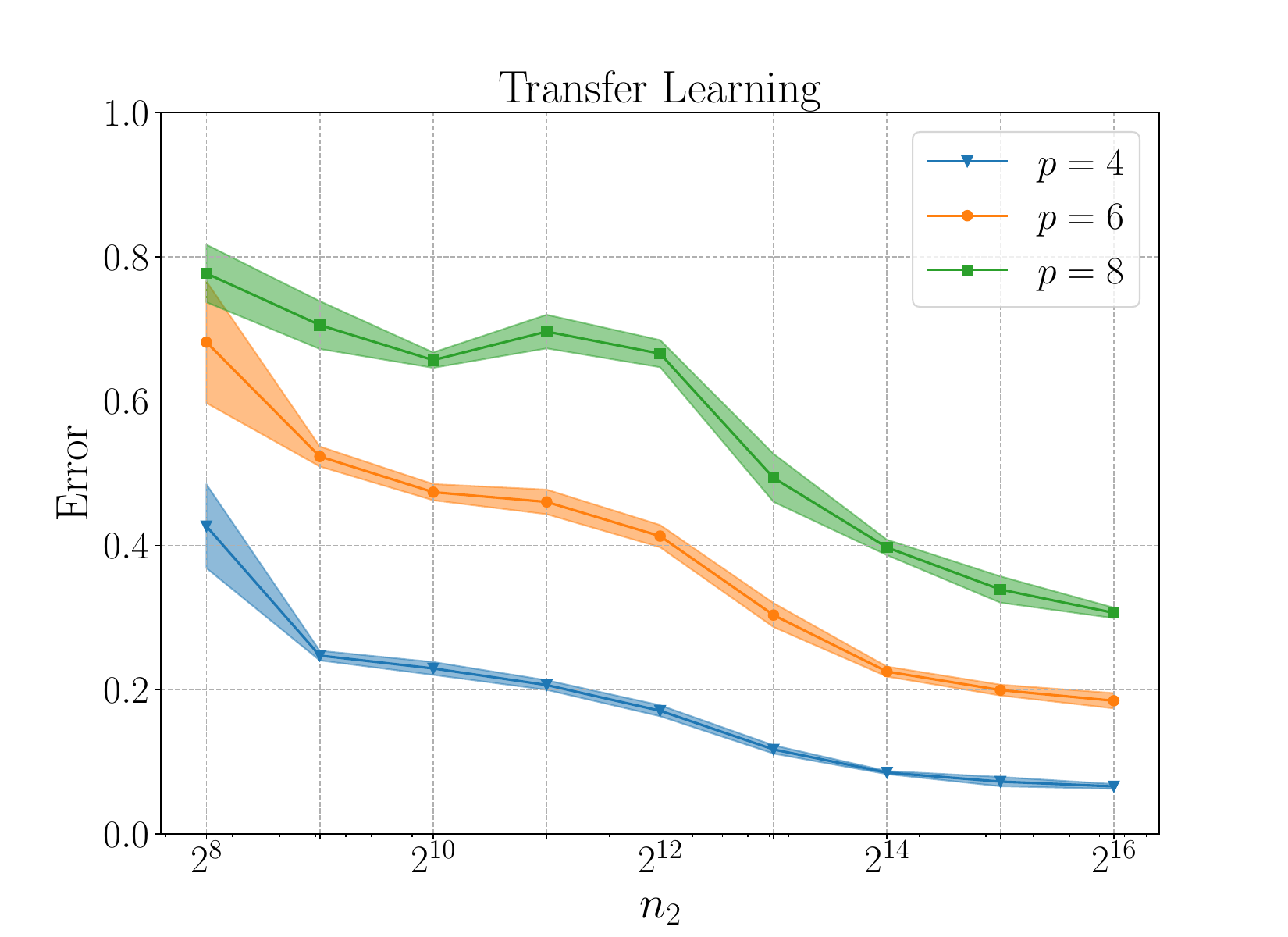}}
    \caption{For the left panel, Algorithm \ref{alg:: training algo} uses two equally sized datasets, while the random feature model uses the full dataset. For the right panel, we conduct transfer learning with  $n_1=2^{16}$ pretraining samples and plot the dependence on $n_2$. The figure reports the mean and normalized standard error of the test error using 10,000 fresh samples, based on $5$ independent experimental instances.}\label{fig::Comparison m2=} 
\end{figure}
 \paragraph{Accurate reconstruction of quadratic features} To further demonstrate our model's feature learning capabilities, we extract the learned features $\bh^{(1)}$ after the first training stage of Algorithm \ref{alg:: training algo}, using $f_{16,2}$ as the target. We then reconstruct these features using a linear transformation $\bB^{\star}\in \RR^{r\times m_2}$, as described in Proposition \ref{prop::reconstructed feature main}. We examine how reconstruction accuracy changes with first-stage sample sizes. Figure \ref{fig:feature reconstruction} shows the correlation between true and reconstructed features for each sample size. As $n_1$ increases, all features are better approximated simultaneously. Notably, $d^4$ samples prove sufficient to reconstruct the features with high accuracy, supporting our model's effective feature learning ability.
 \begin{figure}
    \centering
    \includegraphics[width=1\linewidth]{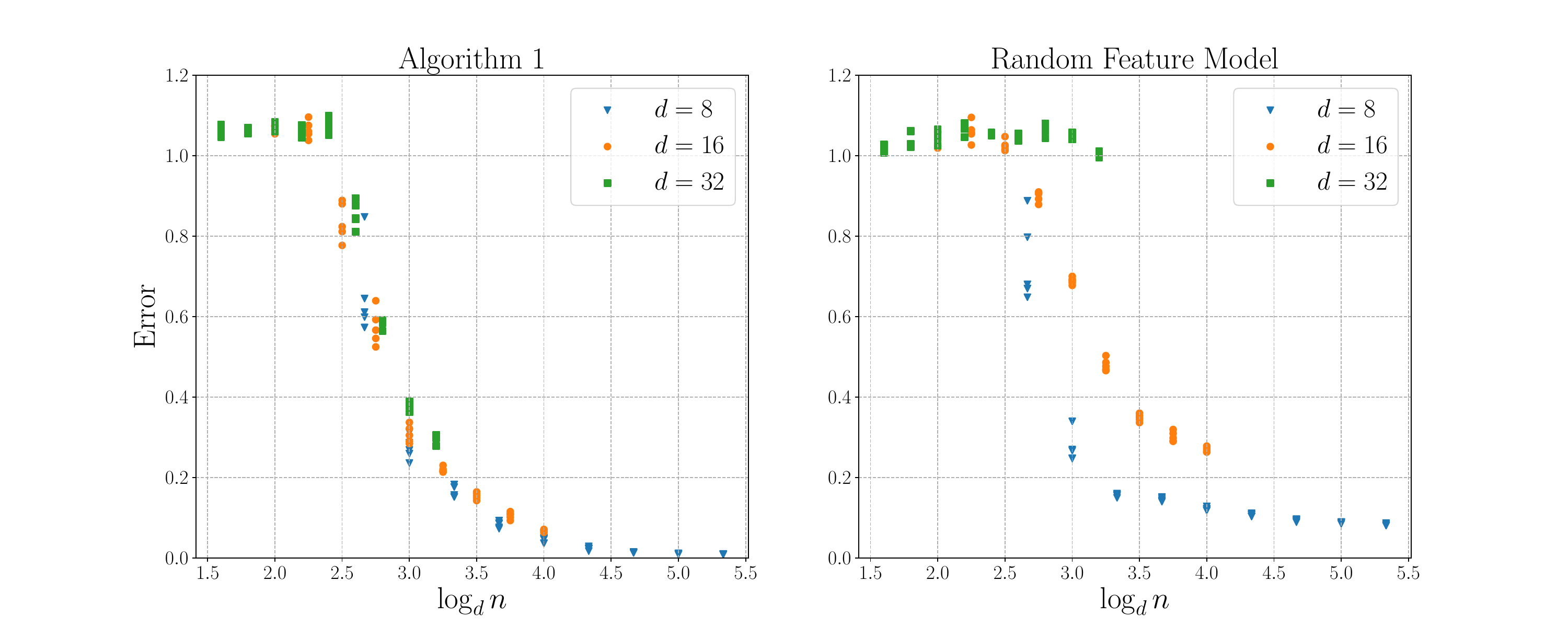}
    \caption{Test error of Algorithm \ref{alg:: training algo} and the naive random feature models with x-axis being the relative sample complexity $(\log_d n)$. We plot the test error of $5$ independent instances for each $d\in\{8,16,32\}$.}
    \label{fig::redraw comparison}
\end{figure}
\begin{figure}
\hspace{-2.2em}\includegraphics[width=1.1\linewidth]{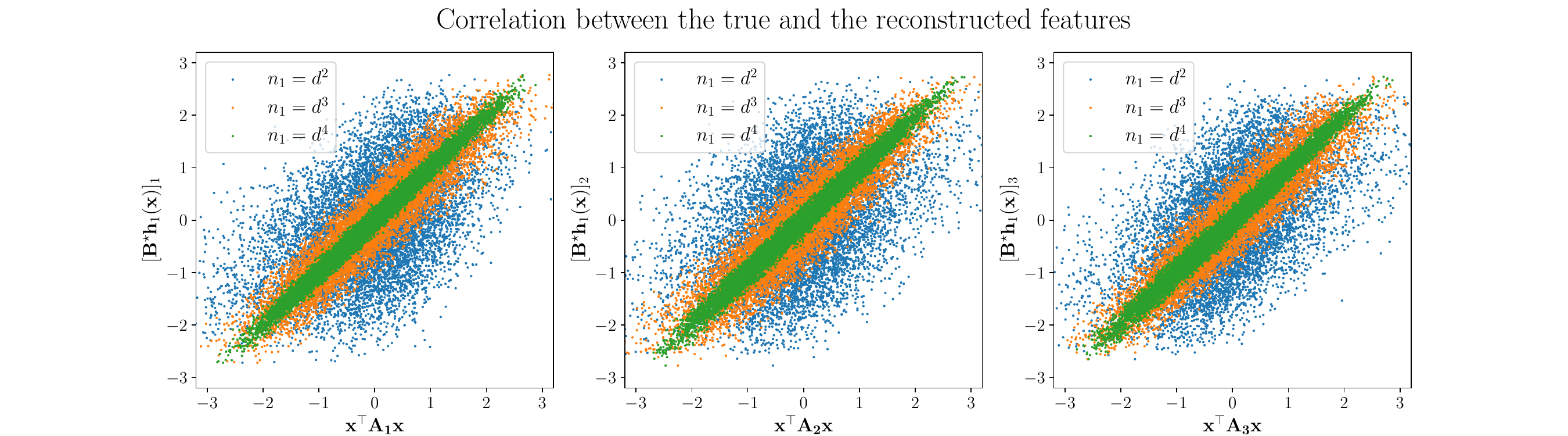}
    \caption{The linear correlation between the three true features and their corresponding reconstructed features for varying first-stage sample sizes $n_1$. The reconstructed features are standardized to match the variance of the true features. For $i=1,2,3$, the $i$-th scatter plot represents $10,000$ test sample points of $([\bB^{\star}\bh^{(1)}(\bx)]_i,\bx^{\top}\bA_i\bx)$ for $n_1 \in \{d^2, d^3, d^4\}$, where $d=16$.}
    \label{fig:feature reconstruction}
\end{figure}


\section{Conclusions and Discussions}
{
\paragraph{Comparison with \cite{nichani2023provable,wang2023learning}} As discussed under assumptions \ref{assump::preprocess target function} and the initialization of our neural networks, our work differs significantly in the targets of interests, the parametrization of neural networks, the mathematical strategies, and the intuitions behind the results. Our assumptions ensure a nearly zero linear component and a non-degenerate second order term of the link function $g$ which significantly contrasts the assumptions posed in  \cite{nichani2023provable,wang2023learning} that emphasize the linear component. Our random initialization (rather than a deterministic initialization used in the aforementioned two works) in the weights of the three-layer neural networks allows the learned weights to capture multiple features in all directions simultaneously after training rather than converge to a single direction. We develop a novel universality result to relate multiple nonlinear features to multivariate Gaussian, while these two works adopt existing result of the approximate Stein's lemma which only applies to single nonlinear feature. Most importantly, subspace recovery is completely different from and also significantly harder than single feature recovery considered in \cite{nichani2023provable,wang2023learning}.

\paragraph{Conclusions} In this work, we have shown the provable capabilities of three-layer networks in efficiently learning targets of multiple quadratic features. Leveraging a novel universality result, we have shown that one gradient step suffices for a full recovery of the subspace spanned by multiple quadratic features. In addition, leveraging the learned features, we have demonstrated the transfer learning capabilities of this three-layer neural network with a constant polynomial sample complexity guarantee.  To the best of our knowledge, this is the first theoretical result of efficiently learning such a board target function class of multiple nonlinear features with neural networks. We have made a great improvement on the sample complexity by highlighting feature learning compared to kernel methods. 

\paragraph{Future works} First, it may be possible that the sample complexity bound of $\widetilde\cO(d^4)$ could be improved to the information-theoretic optimal sample complexity ${\cO}(d^2)$ in learning general hierarchical polynomials of quadratic features. We think that this result may be achieved when we consider more advanced algorithms that utilize the samples more thoroughly such as using multiple steps of GD, which could be a great future extension of our work. Moreover, our methodology is not inherently limited to quadratic features. The principles shown in Figure \ref{figure::proof idea} and techniques developed here give a foundation for understanding the learning of even more complex function classes. Another natural future direction of our work is to understand whether and when our results can be generalized to learning multiple high-degree features. }

\clearpage
\bibliography{ref}
\clearpage
\appendix
\begin{center}
    \noindent\rule{\textwidth}{4pt} \vspace{-0.2cm}
    
    \LARGE \textbf{Appendix} 
    
    \noindent\rule{\textwidth}{1.2pt}
\end{center}

\startcontents[sections]
\printcontents[sections]{l}{1}{\setcounter{tocdepth}{2}}
\clearpage
\section{Techinical Background}

\subsection{Asymptotic Notation}

Throughout the proof we will let $C$ be a fixed but sufficiently large constant.
\begin{definition}[high probability events]
    Let $\iota = C \log(dn_1n_2m_1m_2)$. We say that an event happens \emph{with high probability} if it happens with probability at least $1-\poly(d,n_1,n_2,m_1,m_2)e^{-\iota}$.
\end{definition}
\begin{example}
	If $z \sim N(0,1)$ then $\abs{z} \le \sqrt{2\iota}$ with high probability.
\end{example}

Note that high probability events are closed under union bounds over sets of size $\poly(d,n_1,n_2,m_1,m_2)$, such as $\cD_1$, $\cD_2$ and $\{\bw_j\}_{j\in[m_1]}$. We will also assume throughout the paper that $\iota \le C^{-1} d$.

\subsection{Multivariate Gaussian Approximation}\label{appendix::multivariate gaussian approximation}\label{appendix::mathstein}
In this section, we assume that $\bX \sim \cN(\bzero,\bI_{d})$ and aim to establish an upper bound of Wasserstein distance between the distribution of ${\mathbf{p}(\bX)}$ and the standard $r$-dimensional Gaussian distribution, i.e., Lemma \ref{prop::extend chaterjee}.

To prove Lemma \ref{prop::extend chaterjee}, we introduce Stein's method \citep{ross2011fundamentals} for multivariate Gaussian approximation. We will use the following additional notations.
\begin{itemize}
\item $\cG f(\bx):=\int_{0}^{\infty} {\EE_{\bZ \sim \cN(\bzero,\bI)}\brac{f\paren{e^{-t}\bx+\sqrt{1-e^{-2t}\bZ} } - f(\bZ)} }\d t$ denotes the potential operator of $f$.
\item $\cJ(\bp):=[\nabla p_1,\nabla p_2,...,\nabla p_r]^{\top} \in \RR^{r\times n}$ denotes the Jacobian matrix of $\bp$.
\end{itemize}
Now we state the supporting lemmas to prove Lemma \ref{prop::extend chaterjee}.
\begin{lemma}[Corollary 9.12 in \cite{vanprobability}]\label{lemma::Stein's method}
For any probability measure $\mu$ in $\RR^{r}$, we have
\begin{align*}
    W_1(\mu,\cN(\bzero,\bI_{r})) \le \sup\limits_{\norm{\nabla g}{} \le 1, \norm{\nabla^2 g }{}\le \sqrt{\frac{2}{\pi}} } \EE_{\bY\sim\mu}\brac{\Delta g(\bY) -\left\langle\nabla g(\bY),\bY \right\rangle}.
\end{align*}
\end{lemma}
\begin{lemma}[Lemma 9.21 in \cite{vanprobability}]\label{lemma::stein cov}
    Suppose $\bX=(X_1,X_2,...,X_d)\sim \cN(\bzero,\bI_{n})$ is an $d$-dimensional standard Gaussian variable. Then for any functions $g:\RR^{d}\rightarrow\RR$ and $h:\RR^{d}\rightarrow \RR$, we have
    \begin{align*}
\text{Cov}\paren{g(\bX),h(\bX)}=\EE_{\bX}\brac{\left\langle\nabla g(\bX),\nabla \cG h(\bX) \right\rangle}
\end{align*}
\end{lemma}
With the lemmas above, we begin our proof of Lemma \ref{prop::extend chaterjee}.
\begin{proof}[Proof of Lemma \ref{prop::extend chaterjee}]
   By invoking Lemma \ref{lemma::Stein's method} with $\mu=\text{Law}(\bp)$ and $\bY=\bp(\bX)$, for any $g(\by):\RR^{r}\rightarrow\RR$ with $\norm{\nabla g}{} \le 1$ and $ \norm{\nabla^2 g }{}\le \sqrt{\frac{2}{\pi}} $, we aim to bound 
    \begin{align*}
        \underbrace{\EE_{\bX}\brac{\Delta g(\bp(\bX)) -\left\langle\nabla g(\bp(\bX)),\bp(\bX) \right\rangle}}_{\spadesuit}= \sum_{i=1}^{r}\EE_{\bX}\brac{\frac{\partial^2 g}{\partial y_i^2}\bigg |_{\by=\bp(\bX)}-p_i(\bX)\frac{\partial g}{\partial y_i}\bigg |_{\by=\bp(\bX)}}.
    \end{align*}
    Since for any $i\in[r]$, $\EE\brac{p_i(\bX)}=0$, we have 
    \begin{align*}
        \EE\brac{p_i(\bX)\frac{\partial g}{\partial y_i}\bigg |_{\by=\bp(\bX)}}&=\text{Cov}\paren{p_i(\bX),\frac{\partial g}{\partial y_i}\bigg |_{\by=\bp(\bX)}}\\
        &=\EE\brac{\left\langle \nabla_{\bx}\frac{\partial g}{\partial y_i}\bigg |_{\by=\bp(\bX)}, \nabla_{\bx}\cG p_{i}(\bX) \right\rangle}\\
        &=\EE\brac{\left\langle \sum_{j=1}^{r}\frac{\partial^2 g}{\partial y_i\partial y_j}\bigg |_{\by=\bp(\bX)}\nabla_{\bx} p_j(\bX), \nabla_{\bx}\cG p_{i}(\bX) \right\rangle}\\
        &=\sum_{j=1}^{r}\EE\brac{ \frac{\partial^2 g}{\partial y_i\partial y_j}\bigg |_{\by=\bp(\bX)}\left\langle\nabla_{\bx} p_j(\bX), \nabla_{\bx}\cG p_{i}(\bX) \right\rangle},
    \end{align*}
where the second equality follows from Lemma \ref{lemma::stein cov} and we obtain the third equality by the chain rule. Thus, we have
\begin{align}\label{equ::expect W1 bound}
\spadesuit=\EE\brac{\left\langle \nabla^2 g(\bp(\bX)),\bI_r -\cJ(\bp(\bX))\cJ(\cG \bp(\bX))^{\top} \right\rangle}.
\end{align}
For a special case, for any $i,j\in[r]$, we take $g(\by)=y_iy_j$ in \eqref{equ::expect W1 bound}, obtaining that
\begin{align*}
    \EE\brac{\left\langle\nabla_{\bx} p_j(\bX), \nabla_{\bx}\cG p_{i}(\bX) \right\rangle}=\left\{ \begin{matrix}
    \EE\brac{2p_j(\bX)p_i(\bX)}=0,~~i\ne j,\\
    \EE\brac{2p_i^2(\bX)}-1=1,~~i=j.
    \end{matrix}\right.
\end{align*}
Thus, $\EE\brac{\bI_r -\cJ(\bp(\bX))\cJ(\cG \bp(\bX))^{\top}}=\bzero_{r\times r}$. Since $\norm{\nabla^2 g}{}\le \sqrt{\frac{2}{\pi}}$, we have $\abs{\brac{\nabla^2 g}_{i,j}}\le \sqrt{\frac{2}{\pi}}$ for any $i,j\in[r]$. We can therefore estimate
\begin{align*}
    W_1(\text{Law}(\bp(\bX)),\cN(\bzero,\bI_{r}))&\le \sqrt{\frac{2}{\pi}} \sum_{i,j\in[r]}\EE\brac{\abs{\delta_{i,j}-\left\langle\nabla_{\bx} p_j(\bX), \nabla_{\bx}\cG p_{i}(\bX) \right\rangle}}\\
    &\le \sqrt{\frac{2}{\pi}} \sum_{i,j\in[r]}\text{Var}\brac{\left\langle\nabla_{\bx} p_j(\bX), \nabla_{\bx}\cG p_{i}(\bX) \right\rangle}^{1/2}\\
    &\le \sqrt{\frac{2}{\pi}} \sum_{i,j\in[r]} \EE\brac{\norm{\nabla_{\bx} \left\langle\nabla_{\bx} p_j(\bX), \nabla_{\bx}\cG p_{i}(\bX) \right\rangle}{}^2}^{1/2},
\end{align*}
where we invoke Poincar\'e inequality in the last inequality. For any $i,j\in[r]$, we have
\begin{align*}
    &\quad\EE\brac{\norm{\nabla_{\bx} \left\langle\nabla_{\bx} p_j(\bX), \nabla_{\bx}\cG p_{i}(\bX) \right\rangle}{}^2}\\
    &=\EE\brac{\norm{\nabla^2_{\bx}p_j(\bX)\nabla \cG p_{i}(\bX) + \nabla_{\bx}p_j(\bX)\nabla^2\cG p_{i}(\bX)}{}^2}\\ 
    &\le 2\EE\brac{\norm{\nabla^2_{\bx}p_j(\bX)\nabla \cG p_{i}(\bX)}{}^2}+2\EE\brac{\norm{\nabla_{\bx}p_j(\bX)\nabla^2\cG p_{i}(\bX)}{}^2}\\
    &\le 2\EE\brac{\norm{\nabla^2p_j}{}^{4}}^{1/2} \EE\brac{\norm{\nabla \cG p_{i}}{}^{4}}^{1/2}+2\EE\brac{\norm{\nabla p_j}{}^{4}}^{1/2} \EE\brac{\norm{\nabla^2 \cG p_{i}}{}^{4}}^{1/2}\\
    &\le 2\EE\brac{\norm{\nabla^2p_j}{}^{4}}^{1/2} \EE\brac{\norm{\nabla p_{i}}{}^{4}}^{1/2}+2\EE\brac{\norm{\nabla p_j}{}^{4}}^{1/2} \EE\brac{\norm{\nabla^2 p_{i}}{}^{4}}^{1/2}.
\end{align*}
The last inequality follows from the inequality in Page 308 in \cite{vanprobability}. By adding up all the terms along $i$ and $j$, we have
\begin{align*}
   &\quad W_1(\text{Law}(\bp(\bX)),\cN(\bzero,\bI_{r}))\\
    &\le\sqrt{\frac{2}{\pi}}\sum_{i,j\in[r]} \sqrt{2\EE\brac{\norm{\nabla^2p_j}{}^{4}}^{1/2} \EE\brac{\norm{\nabla p_{i}}{}^{4}}^{1/2}+2\EE\brac{\norm{\nabla p_j}{}^{4}}^{1/2} \EE\brac{\norm{\nabla^2 p_{i}}{}^{4}}^{1/2}}\\
    &\le\frac{2}{\sqrt{\pi}}\sum_{i,j\in[r]} \paren{\EE\brac{\norm{\nabla^2p_j}{}^{4}}^{1/4} \EE\brac{\norm{\nabla p_{i}}{}^{4}}^{1/4}+\EE\brac{\norm{\nabla p_j}{}^{4}}^{1/4} \EE\brac{\norm{\nabla^2 p_{i}}{}^{4}}^{1/4}}\\
    &=\frac{4}{\sqrt{\pi}} \paren{\sum_{i=1}^{r}\EE\brac{\norm{\nabla p_i}{}^{4}}^{1/4} }\paren{\sum_{j=1}^{r}\EE\brac{\norm{\nabla^2 p_j}{}^{4}}^{1/4}}.
\end{align*}
We complete our proof.
\end{proof}
\subsection{Hypercontractivity of Polynomials}\label{subappendix: gaussian hypercontractivity}

The following Lemma is cited from \citet{mei2021learning} and is designed for uniform distribution on the sphere in $d$ dimension.
\begin{lemma}\label{lemma: sphere hypercontractivity}
For any $\ell \in \mathbb{N}$ and $f \in L^2(\mathbb{S}^{d-1})$ to be a degree $\ell$ polynomial, for any $q \geq 2$, we have
$$
\left(\EE_{\bz\sim \operatorname{Unif}(\SS^{d-1}(\sqrt{d}))}\left[f(\bz)^q\right]\right)^{2/q} \le (q-1)^{\ell} \EE_{\bz\sim \operatorname{Unif}(\SS^{d-1}(\sqrt{d}))}\left[f(\bz)^2\right].
$$
\end{lemma}
We remark that the results above are also multiplicative.
\begin{lemma}\label{lemma: sphere hypercontractivity multi}
For any $\ell \in \mathbb{N}$ and $f \in L^2((\mathbb{S}^{d-1})^{k})$ to be a degree $\ell$ polynomial in the components of each $\bz_1,\bz_2,\dots,\bz_k$, for any $q \geq 2$, we have
$$
\left(\EE_{\bz\sim \operatorname{Unif}\paren{\SS^{d-1}(\sqrt{d})}^k}\left[f(\bz)^q\right]\right)^{2/q} \le (q-1)^{k\ell} \EE_{\bz\sim \operatorname{Unif}\paren{\SS^{d-1}(\sqrt{d})}^k}\left[f(\bz)^2\right].
$$
Here $\bz=\Vec\paren{[\bz_1,\bz_2,\dots,\bz_k]}$.
\end{lemma}

For the case where the input distribution is standard Gaussian in $d$ dimension (denoted as $\gamma$), we have the next Lemma from Theorem 4.3, \citet{Prato2007WickPI}. 
\begin{lemma}\label{lemma: gaussian hypercontractivity}
For any $\ell \in \mathbb{N}$ and $f \in L^2(\gamma)$ to be a degree $\ell$ polynomial, for any $q \geq 2$, we have
$$
\mathbb{E}_{\bz \sim \gamma}\left[f(\bz)^q\right] \le \cO_{q,\ell}(1)\left(\mathbb{E}_{\bz \sim \gamma}\left[f(\bz)^2\right]\right)^{q / 2}.
$$
where we use $\cO_{q,\ell}(1)$ to denote some universal constant that only depends on $q,\ell$.
\end{lemma}

Moreover, we introduce lemmas to control the deviation of random variables which polynomially depend on some Gaussian random variables.
We will use a slightly modified version of Lemma 30 from \citet{damian2022neural}.
\begin{lemma}\label{lemma:polynomial concentration}
Let $g$ be a polynomial of degree $p$ and $\bx\sim \cN(\bzero,\bI_d)$. Then there exists an absolute positive constant $C_p$ depending only on $p$ such that for any $\delta>1$,
$$
\mathbb{P}\left[|g(\bx)-\mathbb{E}[g(\bx)]| \geq \delta \sqrt{\operatorname{Var}(g(\bx))}\right] \le 2 \exp \left(-C_p \delta^{2 / p}\right).
$$
\end{lemma}
We also have the spherical version of Lemma \ref{lemma:polynomial concentration}.
\begin{lemma}\label{lemma::poly concentration sphere}
Let $g$ be a polynomial of degree $p$ and $\bx\sim \SS^{d-1}(\sqrt{d})$. Then there exists an absolute positive constant $C_p$ depending only on $p$ such that for any $\delta>1$,
$$
\mathbb{P}\left[|g(\bx)-\mathbb{E}[g(\bx)]| \geq \delta \sqrt{\operatorname{Var}(g(\bx))}\right] \le 2 \exp \left(-C_p \delta^{2 / p}\right).
$$
\end{lemma}

Thus, for a degree-p polynomial $g$, we have $g(\bx)\lesssim \iota^{p/2}\norm{g}{L^2}$ with high probability.
\subsection{Moments and Factorization of Polynomials}
In this section, we present formulae for calculating moments of Gaussian or spherical variables, cited from \citet{damian2022neural}.
\begin{lemma}[Expectations of Gaussian tensors] For $\bw \in \cN(\mathbf{0}_d,\bI_d)$ and $k\in \mathbb{N}$, we have
\begin{align*}
    \EE_{\bw}\brac{\bw^{\otimes 2k}}=(2k-1)!!{\rm Sym}(\bI_d^{\otimes k})
\end{align*}
Here ${\rm Sym}(\bT)$ is the symmetrization of a $k$-tensor $\bT\in (\RR^{d})^{\otimes k}$ across all $k$ axes.
\end{lemma}
Leveraging this calculation, we can factorize any polynomial $g$ into inner products between high-order tensors and bound the Frobenius norm of the tensors.
\begin{lemma}(Lemma 21 in \cite{damian2022neural})\label{lemma::bound F norm tensor}
    Given Let $g:\RR^{r}\rightarrow \RR$ be an degree-$p$ polynomial. Then there exists $\bT_0,\bT_1,\dots,\bT_p$ such that 
    \begin{align*}
        g(\bz)=\sum_{k=0}^{p}\left\langle \bT_k,\bz^{\otimes k} \right \rangle~~~\text{with}~~~\norm{\bT_k}{\rm F}\lesssim \norm{g}{L^2}r^{\frac{p-k}{4}}, ~~k=0,1,\dots,p.
    \end{align*}
    Here $\norm{g}{L^2}=\EE_{\bz \sim \cN(\mathbf{0},\bI)}\brac{g^2(\bz)}$.
\end{lemma}
 As a corollary, we then have $\nabla g(\bz) = \sum_{k \ge 1}^p k\bT_k(z^{\otimes k-1})$ and 
\begin{align}\label{equ::bound grad g}
\norm{\nabla g(\bz)}{} &
\le \sum_{k=1}^p k\norm{\bT_k}{\rm F}\norm{\bz}{}^{k-1}
\lesssim \norm{g}{L^2}\sum_{k=1}^p kr^{\frac{p-k}{4}}\norm{\bz}{}^{k-1}.
\end{align}
For a spherical variable $\bx \sim {\rm Unif}(\SS^{d-1}(\sqrt{d}))$ we can also compute its moments.
\begin{lemma}[Expectations of Spherical tensors]
    For $\bx \sim {\rm Unif}(\SS^{d-1}(\sqrt{d}))$ and $k\in \mathbb{N}$, we have
    \begin{align*}
        \EE_{\bz}\brac{\bz^{\otimes 2k}}=d^{k}\cdot \frac{\EE_{\bw\sim\cN(\mathbf{0}_d,\bI_d)}\brac{\bw^{\otimes 2k}}}{\EE_{v \sim \chi(d)}\brac{v^{2k}}},
    \end{align*}
    where $\chi(d)$ represents the chi-distribution with the degree of freedom being $d$, and its moments can be computed as
    \begin{align*}
        \EE_{v \sim \chi(d)}\brac{v^{2k}}=\prod_{j=0}^{k-1}(d+2j)=\Theta(d^k).
    \end{align*}
\end{lemma}
As an example, the moments of spherical quadratic forms $\bx^{\top}\bA\bx$ can be computed explicitly as
\begin{align*}
    \EE_{\bx}\brac{\bx^{\top}\bA\bx}={\rm tr}(\bA),~~\text{and}~~\EE_{\bx}\brac{(\bx^{\top}\bA\bx)(\bx^{\top}\bB\bx)}=\frac{d}{d+2}\cdot \paren{{\rm tr}(\bA){\rm tr}(\bB)+2\langle \bA,\bB\rangle}.
\end{align*}
Thus, to satisfy Assumption \ref{assump::features}, we require ${\rm tr}(\bA_k)=0$, $\norm{\bA_k}{\rm F}=\sqrt{(d+2)/(2d)}$ and $\langle \bA_k,\bA_\ell\rangle=0$ for any $k,\ell\in[r]$.

\subsection{Spherical Harmonics and Gegenbauer Polynomials}\label{appendix::mathsphere}

We introduce some facts of spherical harmonics and Gegenbauer polynomials, with the first four properties from \citet{ghorbani2021linearized} and the last one from \citet{Koornwinder_2018}.
\begin{enumerate}
\item  For $\bx,\by\in \SS^{d-1}(\sqrt d)$,
\begin{align}
\abs{Q_j(\langle\bx,\by\rangle)}\le Q_j(d)=1.
\label{eq:Gegenbauer bound}
\end{align}
\item  For $\bx, \by \in \SS^{d-1}(\sqrt d)$,
\begin{align}
\< Q_j^{}(\< \bx, \cdot\>), Q_k^{}(\< \by, \cdot\>) \>_{L^2} = \frac{1}{B(d,k)}\delta_{jk}  Q_k^{}(\< \bx, \by\>).  \label{eq:ProductGegenbauer}
\end{align}
Here $B(d,k)$ denotes the dimension of subspace of degree $k$ spherical harmonics
\begin{align*}
    B(d,k):={\rm dim}(V_{d,k})=\frac{2k+d-2}{k}\binom{k+d-3}{k-1}=\Theta(d^{k}).
\end{align*}

\item For $\bx, \by \in \SS^{d-1}(\sqrt d)$,
\begin{align}
Q_k^{}(\< \bx, \by\> ) = \frac{1}{B(d, k)} \sum_{i =1}^{ B(d, k)} Y_{k,i}^{}(\bx) Y_{k,i}^{}(\by). \label{eq:GegenbauerHarmonics}
\end{align}
\item For any $k\in\mathbb{N}_{\ge 1}$,
\begin{align}
\frac{t}{d}\,  Q_k^{}(t) = \frac{k}{2k + d - 2} Q_{k-1}^{}(t) + \frac{k + d - 2}{2k + d - 2} Q_{k+1}^{}(t). \label{eq:RecursionG}
\end{align}
\item For any $i,j\in \mathbb{N}$,
\begin{align}
Q_i(t)Q_j(t)&=\sum_{k=0}^{\min(i,j)}{b^{(i,j)}_{i+j-2k}}\binom{i}{k}\binom{j}{k}k! Q_{i+j-2k}(t).
\end{align}
\end{enumerate}
Here, we have
\begin{align*}
   b^{(i,j)}_{i+j-2k}&=\frac{2(i+j-2k)+d-2}{d-2}\cdot\frac{\paren{{(d-2)}/{2}}_{k}\paren{{(d-2)}/{2}}_{i-k}\paren{{(d-2)}/{2}}_{j-k}\paren{{d-2}}_{i+j-k}}{\paren{d-2}_{i}\paren{d-2}_{j}\paren{{d}/{2}}_{i+j-k}}.\nonumber
\end{align*}
We note that $(z)_k=z(z+1)\cdots(z+k-1)=\Gamma(z+k)/\Gamma(z)$ is the Pochhammer symbol. Given any $i$ and $j$, we have $d^{k}b^{(i,j)}_{i+j-2k}\rightarrow 1$ when $d\rightarrow \infty$. We derive a quantitative bound on the scale of $b^{(i,j)}_{i+j-2k}$ in Lemma  \ref{lemma::coeff bound linearization}.

\begin{lemma}\label{lemma::coeff bound linearization}
   For any $i,j\ge k\ge0$, denote 
$$c_{i+j-2k}^{(i,j)}=\frac{\paren{{(d-2)}/{2}}_{k}\paren{{(d-2)}/{2}}_{i-k}\paren{{(d-2)}/{2}}_{j-k}\paren{{d-2}}_{i+j-k}}{\paren{d-2}_{i}\paren{d-2}_{j}\paren{{d}/{2}}_{i+j-k}}.$$
Then, when $d\ge 4$, it holds that $c_{i+j-2k}^{(i,j)}\le \frac{1}{(d-2)_k}$.
\end{lemma}

\begin{proof}[Proof of Lemma \ref{lemma::coeff bound linearization}]
    Note that when $d \geq 4$, i.e., $d-2\geq d/2$,
\begin{align*}
    \frac{c_{(i+1)+j-2k}^{(i+1,j)}}{c_{i+j-2k}^{(i,j)}}&=\frac{\paren{\frac{d-2}{2}+i-k}\paren{d-2+i+j-k}}{\paren{d-2+i}\paren{\frac{d}{2}+i+j-k}}~~~(\text{monotone decreasing with}~j)\\
    &\le \frac{\paren{\frac{d-2}{2}+i-k}\paren{d-2+i}}{\paren{d-2+i}\paren{\frac{d}{2}+i}}\\
    &< 1.
\end{align*}
Thus, we have $c_{(i+1)+j-2k}^{(i+1,j)} \le c_{i+j-2k}^{(i,j)}$. Similarly, we have $c_{i+(j+1)-2k}^{(i,j+1)} \le c_{i+j-2k}^{(i,j)}$. Consequently, for any $i,j \geq k$, we have 
\begin{align*}
    c_{i+j-2k}^{(i,j)}&\le  c_{k+k-2k}^{(k,k)}\\
    &=\frac{\paren{{(d-2)}/{2}}_{k}\paren{{(d-2)}/{2}}_{0}\paren{{(d-2)}/{2}}_{0}\paren{{d-2}}_{k}}{\paren{d-2}_{k}\paren{d-2}_{k}\paren{{d}/{2}}_{k}}\\
    &=\frac{\paren{{(d-2)}/{2}}_{k}}{\paren{d-2}_{k}\paren{{d}/{2}}_{k}}\\
    &=\frac{(d-2)/2}{\paren{d-2}_{k}(d/2 +k-1)}\\
    &\le \frac{1}{\paren{d-2}_{k}}.
\end{align*}
The proof is complete.
\end{proof}

\section{Approximation Theory of the Inner Layer}
Since we focus on the first training stage throughout this section, we denote $n=n_1$ for notation simplicity when the context is clear, and let the training set be $\cD_1=\{\bx_1,\bx_2,\dots,\bx_n\}$.
\subsection{Asymptotic Analysis of the Learned Feature}\label{appendix::asymptotic analysis inner feature reconstruct}
In this subsection, we analyse the learned feature $\bh^{(1)}(\bx')$ in the asymptotic way, i.e., $m_2,n\rightarrow \infty$. Note that we can rewrite the learned feature as
\begin{align*}
    \bh^{(1)}(\bx')&=\frac{1}{nm_2}\sum_{i=1}^{n}f^{\star}(\bx_i){\langle \bh^{(0)}(\bx_i),\bh^{(0)}(\bx')\rangle}\bh^{(0)}(\bx_i)\\
    &=\frac{1}{n}\sum_{i=1}^{n}f^{\star}(\bx_i) K^{(0)}_{m_2}(\bx,\bx')\bh^{(0)}(\bx_i).
\end{align*}
where the initial kernel $ K^{(0)}_{m_2}(\bx,\bx')$ is defined as
\begin{align*}
    K^{(0)}_{m_2}(\bx,\bx')&=\frac{1}{m_2}\langle \sigma_2\paren{\bV\bx},\sigma_2\paren{\bV\bx'}\rangle \approx {\EE_{\bv}\brac{\sigma_2(\bv^{\top}\bx)\sigma_2(\bv^{\top}\bx')}}.
\end{align*}
In this case, we have for any $j\in[m_2]$,
\begin{align*}
    [\bh^{(1)}(\bx')]_j&=\frac{1}{n}\sum_{i=1}^{n}K^{(0)}_{m_2}(\bx_i,\bx')\sigma_2(\bv_j^{\top}\bx_i)
    \overset{m_2,n\rightarrow \infty}{\rightarrow} \EE_{\bx}\brac{f^{\star}(\bx)K^{(0)}(\bx,\bx')\sigma_2(\bv_j^{\top}\bx)}.
\end{align*}
Here the infinite-inner-width kernel $K^{(0)}$ is defined as
\begin{align*}
    K^{(0)}(\bx,\bx'):=\EE_{\bv}\brac{\sigma_2(\bv^{\top}\bx)\sigma_2(\bv^{\top}\bx')}=\sum_{i=2}^{\infty}\frac{c_i^2}{B(d,i)}Q_i(\bx^{\top}\bx').
\end{align*}
Recall that $Q_2(t)=\frac{t^2-d}{d(d-1)}$, so $Q_2(\bv_{j}^{\top}\bx)=\langle \bx\bx^{\top}-\bI, \bv_j\bv_j^{\top}-\bI \rangle/(d(d-1))$. Let's focus on the contribution of the quadratic term $Q_2$ in $ K^{(0)}(\bx,\bx')$ and $\sigma_2(\bv_j^\top\bx)$, which is
\begin{align*}
    &\quad \frac{c_2^3}{B(d,2)}\cdot \EE_{\bx}\brac{f^{\star}(\bx){Q_2(\bx^{\top}\bx')} Q_2(\bv_j^{\top}\bx)}\\
    &=\frac{c_2^3}{B(d,2)d^2(d-1)^2}\cdot \EE_{\bx}\brac{f^{\star}(\bx) \langle \bx\bx^{\top}-\bI, \bv_j\bv_j^{\top}-\bI\rangle\langle \bx\bx^{\top}-\bI, \bx'\bx'^{\top}-\bI\rangle}\\
    &\approx \frac{1}{d^6} \left\langle \EE_{\bx}\brac{f^{\star}(\bx) (\bx\bx^{\top}-\bI)^{\otimes 2}}, ( \bv_j\bv_j^{\top}-\bI)\otimes (\bx'\bx'^{\top}-\bI)  \right\rangle.
\end{align*}
The following proposition provides an approximation of the tensor $\EE_{\bx}\brac{f^{\star}(\bx) (\bx\bx^{\top}-\bI)^{\otimes 2}}$, which lays the foundation of our feature reconstruction theory.
\begin{proposition}\label{prop::approximate stein lemma quad}
    Consider two linear operators $T$ and $T^{\star}$ that map $\RR^{d \times d}$ to $\RR^{d \times d}$ and satisfy
    \begin{align}
        &T(\bW)=\EE_{\bx}\left[f^{\star}(\bx)\langle \bW,\bx\bx^{\top}-\bI \rangle (\bx\bx^{\top}-\bI)\right],~~\text{and}~\label{equ::definition T}\\
        &T^{\star}(\bW)=\sum_{k=1}^{r}\frac{1}{\norm{\bA_k}{F}^2}\langle \bW,\bA_k \rangle \sum_{j=1}^{r} \bH_{k,j}\bA_j \label{equ::definition T star}
    \end{align}
    for any $\bW \in \RR^{d\times d}$, 
    where $\bH$ is the expected Hessian matrix $\bH=\EE_{\bz \sim \cN(\bzero_r,\bI_r)}\brac{\nabla^{2}g^{\star}(\bz)}$. Then for any $\bW \in \RR^{d\times d}$ , we have $$\norm{T(\bW)-T^{\star}(\bW)}{\rm F}\lesssim d^{-1/6}Lr^{2}\normA\log^2 d\cdot\norm{\bW}{\rm F}.$$
    Here $L=\widetilde\cO(r^{\frac{p-1}{2}})$ is the Lipschitz constant of $g^{\star}$ that holds with high probability.
\end{proposition}
The proof is provided in Appendix \ref{appendix::approxiamte stein lemma}. This proposition shows that, when $\bH$ is well-conditioned and $d\gg r$, $T$ can fully recover the space spanned by $\bA_1,\bA_2,\dots,\bA_r$,  which enables us to reconstruct the features efficiently. Specifically, when taking $\bW_k=\sum_{j=1}^{r}[\bH^{-1}]_{k,j} \bA_j$ for any $k\in[r]$, we have $T(\bW_k)\approx T^{\star}(\bW_k)= \bA_k$. 

Now we consider the construction of $\bB^{\star}$. If we set $\bB^{\star}=\frac{1}{m_2}[\bp_0(\bv_1),\dots,\bp_0(\bv_{m_2})]$ for some vector-valued function $\bp_0:\RR^{d}\rightarrow \RR^{r}$ and denote $\bp_0(\bv)=[p_{0,1}(\bv),\dots,p_{0,r}(\bv)]^{\top}$, we directly have for any $k\in[r]$,
\begin{align*}
    [\bB^{\star} \bh^{(1)}(\bx')]_{k}&{\approx}\frac{1}{m_2}\sum_{j=1}^{m_2}\EE_{\bx}\brac{f^{\star}(\bx)K^{(0)}(\bx,\bx')\sigma_2(\bv_j^{\top}\bx)p_{0,k}(\bv_j)}\\
    &\approx \frac{1}{d^6}\EE_{\bv}\brac{\left\langle \EE_{\bx}\brac{f^{\star}(\bx) (\bx\bx^{\top}-\bI)^{\otimes 2}}, p_{0,k}(\bv)(\bv\bv^{\top}-\bI)\otimes (\bx'\bx'^{\top}-\bI)  \right\rangle}.\\
    &\approx \frac{1}{d^6} \left\langle  T^{\star}\paren{ \EE_{\bv}\brac{p_{0,k}(\bv)(\bv\bv^{\top}-\bI)}} ,\bx'\bx'^{\top}-\bI\right\rangle.
\end{align*}
Thus, it suffices to solve 
\begin{align*}
    T^{\star}\paren{ \EE_{\bv}\brac{p_{0,k}(\bv)(\bv\bv^{\top}-\bI)}} \propto \bA_k,~~~k=1,2,\dots,r,
\end{align*}
which is equivalent to solving
\begin{align*}
    \EE_{\bv}\brac{p_{0,k}(\bv)(\bv\bv^{\top}-\bI)}\propto \bW_k=\sum_{j=1}^{r}[\bH^{-1}]_{k,j} \bA_j.
\end{align*}
Since we have $\EE_{\bv}\brac{(\bv^{\top}\bA_k\bv)(\bv\bv^{\top}-\bI)}\propto \bA_k$, we can explicitly construct $p_{0,k}(\bv)$ as 
\begin{align*}
    p_{0,k}(\bv)\propto \sum_{j=1}^{r}[\bH^{-1}]_{k,j} \bv^{\top}\bA_j\bv,~~{\rm i.e.,}~~\bp_0(\bv)\propto \bH^{-1}\bp(\bv).
\end{align*}
 Thus, with a well conditioned $\bH$, we can fully reconstruct the features.

\subsubsection{Proof of Proposition \ref{prop::approximate stein lemma quad}}\label{appendix::approxiamte stein lemma}
To prove Proposition \ref{prop::approximate stein lemma quad}, it suffices to prove that the approximation error $$R(\bW,\bV)=\langle T(\bW)-T^{\star}(\bW), \bV \rangle \lesssim d^{-1/6}Lr^2\normA\log^2d $$  holds for any test matrix $\bV$ with $\norm{\bV}{\rm F}=1$. We rely on the following three lemmas.

\begin{lemma}[Bound $R(\bA_i,\bA_j)$]\label{lemma::Ai Aj}
     For any $i,j\in[r]$, we have
    \begin{align*}    \Bigg|\EE_{\bx}&\brac{g^{\star}(\bx^{\top}\bA_1\bx,\dots,\bx^{\top}\bA_r\bx)\langle \bA_i,\bx\bx^{\top}-\bI\rangle \langle \bA_j,\bx\bx^{\top}-\bI\rangle}-\EE_{\bz \sim \cN(\bzero_r,\bI_r)}\brac{\nabla^{2}g(\bz)}_{i,j}\Bigg| \le \frac{Lr^2\normA \log^2 d}{\sqrt{d}}.
    \end{align*}
    Here $L=C_{g}Rr^{\frac{p-1}{2}}$ is the Lipschitz constant of $g^{\star}$ that holds with high probability.
\end{lemma}

Following the proof above, we have the following more general lemma.
\begin{lemma}[Bound $R(\bA_i,\bB)$]\label{lemma::AiB strong}
     For any matrix $\bB\in\RR^{d\times d}$ satisfying $\EE\brac{\bx^{\top}\bB\bx}=0$, $\EE\brac{\paren{\bx^{\top}\bB\bx}^2}=1$ and $\langle \bB,\bA_i \rangle=0$ for any $i=1,2,\dots, r$, we have
    \begin{align*}    \abs{\EE\brac{g^{\star}(\bx^{\top}\bA_1\bx,\dots,\bx^{\top}\bA_r\bx)\langle \bA_i,\bx\bx^{\top}-\bI\rangle \langle \bB,\bx\bx^{\top}-\bI\rangle}} \lesssim d^{-1/4}Lr^2\normA\log^2 d.
    \end{align*}
\end{lemma}
\begin{lemma}[Bound $R(\bB_1,\bB_2)$]\label{lemma::B1B2}
    For any two matrices $\bB_1, \bB_2\in\RR^{d\times d}$ satisfying $\EE\brac{\bx^{\top}\bB_j\bx}=0$, $\EE\brac{\paren{\bx^{\top}\bB_j\bx}^2}=1$ and $\langle \bB_j,\bA_i \rangle=0$ for any $j=1,2$ and $i=1,2,\dots, r$, we have
    \begin{align*}    \abs{\EE\brac{g^{\star}(\bx^{\top}\bA_1\bx,\dots,\bx^{\top}\bA_r\bx)\langle \bB_1,\bx\bx^{\top}-
    \bI\rangle \langle \bB_2,\bx\bx^{\top}-\bI\rangle}} \lesssim d^{-1/6}Lr^2\normA\log^2 d.
    \end{align*}
   Here $L\lesssim \iota r^{\frac{p-1}{2}}$ is a Lipschitz constant satisfying $\norm{\nabla g^{\star}(\bp(\bx))}{\rm 2} \le L$  with high probability.
\end{lemma}
The proof of the three lemmas is provided in Appendix \ref{appendix:: omiited proof approximate stein}. With the lemmas above, we begin our proof of Proposition \ref{prop::approximate stein lemma quad}.
\begin{proof}[Proof of Proposition \ref{prop::approximate stein lemma quad}]
    Given any $\bW\in\RR^{d\times d}$, we assume $\norm{\bW}{\rm F}=1$ without loss of generality. Let's decompose $\bW$ as
    \begin{align*}
        \bW=\sum_{k=1}^{r}\lambda_{k}\bA_k  +{\lambda_{r+1}} \frac{\bI_d}{\sqrt{d}}+\lambda_{r+2}\bB, ~~\text{where}~~ \langle \bB, \bA_k\rangle= \langle \bB, \bI_d\rangle=0,~~k\in[r].
    \end{align*}
    Here the coefficients $\set{\lambda_k}_{k=1}^{r+2}$ satisfy $\sum_{k=1}^{r+2}\lambda_k^2\lesssim1$, so $\sum_{k=1}^{r+2}\abs{\lambda_k}\lesssim \sqrt{r}$. Since $T^{\star}(\bB)=T(\bI)=T^{\star}(\bI)=\bzero_{d\times d}$, we have
    \begin{align*}
        \norm{T(\bW)-T^{\star}(\bW)}{F}&\le \sum_{k=1}^{r}\abs{\lambda_k }\norm{T(\bA_k)-T^{\star}(\bA_k)}{F} \\&+ \abs{\lambda_{r+1}}\norm{T\left(\frac{\bI_d}{\sqrt{d}}\right)-T^{\star}\left(\frac{\bI_d}{\sqrt{d}}\right)}{F} +\abs{\lambda_{r+2}}\norm{T(\bB)-T^{\star}(\bB)}{F}\\
        &=\sum_{k=1}^{r}\abs{\lambda_k }\norm{T(\bA_k)-T^{\star}(\bA_k)}{F}+\abs{\lambda_{r+2}}\norm{T(\bB)}{F}. 
    \end{align*}
Since both $T(\bA_k)$ and $T^{\star}(\bA_k)$ are traceless, by Lemma \ref{lemma::Ai Aj} and \ref{lemma::AiB strong}, we have for any $k=1,2,\cdots,r$,
\begin{align*}
    \norm{T(\bA_k)-T^{\star}(\bA_k)}{F}&=\max_{\norm{\bV}{}=1,\tr(\bV)=0}\langle T(\bA_k)-T^{\star}(\bA_k), \bV \rangle \\&\lesssim \sqrt{r}\cdot d^{-1/2}Lr^2\normA\log^2 d+d^{-1/4}Lr^2\normA\log^2 d\\
    &\lesssim d^{-1/4}Lr^2\normA\log^2 d.
\end{align*}
This is because we can decompose $\bV=\sum_{k=1}^r c_k\bA_k +c_{r+1}\bB'$ with $\langle \bB', \bA_k\rangle=0$ and apply the two lemmas to obtain the results above.
Similarly, by Lemma \ref{lemma::AiB strong} and \ref{lemma::B1B2}, we have 
\begin{align*}
    \norm{T(\bB)}{F}&=\max_{\norm{\bV}{}=1,\tr(\bV)=0}\langle T(\bB), \bV \rangle\\& \lesssim \sqrt{r}\cdot d^{-1/4}Lr^{2}\normA\log^2 d+d^{-1/6}Lr^{2}\normA\log^2 d \\
        &\lesssim d^{-1/6}Lr^{2}\normA\log^2 d .
\end{align*}
Thus, we have 
 \begin{align*}
        \norm{T(\bW)-T^{\star}(\bW)}{F}
        &\lesssim \sum_{k=1}^{r+2}\abs{\lambda_k}d^{-1/4}Lr^{2}\log^2 d + \abs{\lambda_{r+2}}d^{-1/6}Lr^{2}\normA\log^2 d \\
        &\lesssim d^{-1/6}Lr^{2}\normA\log^2 d.
    \end{align*}
Here we invoke $\sum_{k=1}^{r+2}\abs{\lambda_k}\lesssim \sqrt{r}=o_d(1)$ in the last inequality. The proof is complete.
\end{proof}

\subsubsection{Omitted Proofs in Appendix \ref{appendix::approxiamte stein lemma}}\label{appendix:: omiited proof approximate stein}

The following lemmas lay the foundation for our approximation process.
\begin{lemma}\label{lemma::Gaussian W1}
Suppose Assumption \ref{assump::features} holds. Then the Wasserstein-1 distance between the distribution of $(\bx^{\top}\bA_1\bx,\dots,\bx^{\top}\bA_r\bx)$ and standard Gaussian $\cN(\bzero_r,\bI_r)$ can be bounded by
\begin{align}   W_1\paren{(\bx^{\top}\bA_1\bx,\dots,\bx^{\top}\bA_r\bx),\cN(\bzero_r,\bI_r)} \lesssim \frac{r^2\normA}{\sqrt{d}}. \label{equ::W1}
\end{align}
Moreover, for any orthogonal unit vectors $\bu_1,\bu_2,\dots \bu_s \in \RR^{d}$, we have a similar bound of
\begin{align}
W_1\paren{(\bx^{\top}\bA_1\bx,\cdots,\bx^{\top}\bA_r\bx,\bu_1^{\top}\bx,\bu_2^{\top}\bx,\cdots,\bu_s^{\top}\bx),\cN(\bzero_{r+s},\bI_{r+s})} \lesssim \frac{(r+s)^2\normA}{\sqrt{d}}. \label{equ::W1 u}
\end{align}
\end{lemma}
\begin{proof}[Proof of Lemma \ref{lemma::Gaussian W1}]
    For a fixed matrix $\bA_i$, define the function $f_i(\bz) = d\frac{\bz^{\top}\bA_i\bz}{\norm{\bz}{}^2}$  and let $\bx = \frac{\bz\sqrt{d}}{\norm{\bz}{}}$. Observe that when $\bz \sim \mathcal{N}(0, \bI)$, we have $\bx \sim \text{Unif}(\mathcal{S}^{d-1}(\sqrt{d}))$. Therefore $[f_i(\bz)]_{i\in[r]}$ is equal in distribution to $[\bx^{\top}\bA\bx]_{i \in[r]}$.  We have for any $i\in[r]$,
    \begin{align*}
        \nabla f_i(\bz) = 2d\left(\frac{\bA_i\bz}{\norm{\bz}{}^2} - \frac{\bz^{\top}\bA_i\bz\cdot \bz}{\norm{\bz}{}^4}\right)
    \end{align*}
    and
    \begin{align*}
        \nabla^2 f_i(\bz) = 2d\left(\frac{\bA_i}{\norm{\bz}{}^2} - \frac{2\bA_i\bz\bz^{\top}}{\norm{\bz}{}^4} - \frac{2\bz\bz^{\top}\bA_i}{\norm{\bz}{}^4} - 2\frac{\bz^{\top}\bA_i\bz}{\norm{\bz}{}^4}\bI + 4\frac{\bz^{\top}\bA_i\bz \bz\bz^{\top}}{\norm{\bz}{}^6}\right).
    \end{align*}
    Thus, we have
    \begin{align*}
        \norm{\nabla f_i(\bz)}{} \le 2d\left(\frac{\norm{\bA_i\bz}{}}{\norm{\bz}{}^2} + \frac{\abs{\bz^{\top}\bA_i\bz}}{\norm{\bz}{}^3}\right) \le \frac{\sqrt{d}}{\norm{\bz}{}}\cdot\norm{\bA_i\bx}{}  + \frac{\abs{\bx^{\top}\bA_i\bx}}{\norm{\bz}{}}.
    \end{align*}
    and
    \begin{align*}
        \norm{\nabla^2 f_i(\bz)}{op} \lesssim \frac{d}{\norm{\bz}{}^2}\norm{\bA_i}{op}.
    \end{align*}
    Since $\norm{\bz}{}^2$ is distributed as a chi-squared random variable with $d$ degrees of freedom, and thus
    \begin{align*}
        \EE\left[\norm{\bz}{}^{-2k}\right] &= \frac{1}{\prod_{j = 1}^k(d - 2j)}.
    \end{align*}
    Therefore, we have
    \begin{align*}
        \mathbb{E}\left[\norm{\nabla^2f_i(\bz)}{op}^4\right]^{1/4} \lesssim d\norm{\bA_i}{op}\mathbb{E} \left[ \norm{\bz}{}^{-8}\right ]^{1/4} \lesssim \norm{\bA_i}{op}.
    \end{align*}
    Then, using the fact that $\bx$ and $\norm{\bz}{}$ are independent,
     \begin{align*}
        \mathbb{E}\left[\norm{\nabla f_i(\bz)}{}^4\right]^{1/4} \lesssim \sqrt{d}\mathbb{E}\left[\norm{\bz}{}^{-4}\right]^{1/4} \mathbb{E}\left[\norm{\bA_i\bx}{}^4]^{1/4}\right] + \mathbb{E}\left[\norm{\bz}{}^{-4}\right]^{1/4}\mathbb{E}\left[(\bx^T\bA_i\bx)^4\right]^{1/4} \lesssim 1.
    \end{align*}
    Thus by Lemma \ref{prop::extend chaterjee} we have
    \begin{align*}
    &\quad W_1\paren{(\bx^{\top}\bA_1\bx,\dots,\bx^{\top}\bA_r\bx),\cN(\bzero_r,\bI_r)}\\
&=W_1\paren{f_1(\bz),\dots,f_r(\bz),\cN(\bzero_r,\bI_r)}\\
        &\lesssim \frac{4}{\sqrt{\pi}} \paren{\sum_{i=1}^{r}\EE\brac{\norm{\nabla f_i(\bz)}{}^{4}}^{1/4} }\paren{\sum_{j=1}^{r}\EE\brac{\norm{\nabla^2 f_j(\bz)}{}^{4}}^{1/4}}\\
        &\lesssim \frac{r^2\normA}{\sqrt{d}}.
    \end{align*}

Now let's focus on the function $g_j(\bz)=\sqrt{d}\frac{\bu_j^{\top}\bz}{\norm{\bz}{}}$. It holds that 
\begin{align*}
    \nabla g_j(\bz)=\sqrt{d}\paren{\frac{\bu_j}{\norm{\bz}{}}-\frac{\bu_j^{\top}\bz \cdot \bz}{\norm{\bz}{}^3} }
\end{align*}
and 
\begin{align*}
    \nabla^2 g_j(\bz)=\sqrt{d}\paren{-\frac{\bu_j\bz^{\top}+\bz\bu_j^{\top}}{\norm{\bz}{}^3}+3\frac{\bu_j^{\top}\bz\cdot\bz\bz^{\top}}{\norm{\bz}{}^5}}.
\end{align*}
Thus, we have 
\begin{align*}
    \norm{\nabla g_j(\bz)}{}\le \frac{2\sqrt{d}}{\norm{\bz}{}}~~~\text{and}~~~\norm{\nabla^2 g_j(\bz)}{op}\le \frac{5\sqrt{d}}{\norm{\bz}{}^2},
\end{align*}
which directly gives rise to 
\begin{align*}
    \mathbb{E}\left[\norm{\nabla g_j(\bz)}{}^4\right]^{1/4} \lesssim 1~~~\text{and}~~~\mathbb{E}\left[\norm{\nabla^2 g_j(\bz)}{op}^4\right]^{1/4}\lesssim \frac{1}{\sqrt{d}}.
\end{align*}
Again by Lemma \ref{prop::extend chaterjee}, we have
    \begin{align*}
    &\quad W_1\paren{(\bx^{\top}\bA_1\bx,\cdots,\bx^{\top}\bA_r\bx,\bu_1^{\top}\bx,\bu_2^{\top}\bx,\cdots,\bu_s^{\top}\bx),\cN(\bzero_{r+s},\bI_{r+s})}\\
&=W_1\paren{f_1(\bz),\dots,f_r(\bz),g_1(\bz),\cdots,g_s(\bz),\cN(\bzero_{r+s},\bI_{r+s})}\\
        &\lesssim \frac{4}{\sqrt{\pi}} \paren{\sum_{i=1}^{r}\EE\brac{\norm{\nabla f_i(\bz)}{}^{4}}^{1/4} \sum_{i=1}^{s}\EE\brac{\norm{\nabla g_i(\bz)}{}^{4}}^{1/4} }\\
        &\quad \cdot\paren{\sum_{j=1}^{r}\EE\brac{\norm{\nabla^2 f_j(\bz)}{}^{4}}^{1/4}+\sum_{j=1}^{s}\EE\brac{\norm{\nabla^2 g_j(\bz)}{}^{4}}^{1/4}}\\
        &\lesssim \frac{(r+s)^2\normA}{\sqrt{d}}.
    \end{align*}
The proof is complete.
\end{proof}
With the lemma above, we begin our proof of Lemma \ref{lemma::Ai Aj}.
\begin{proof}[Proof of Lemma \ref{lemma::Ai Aj}]
    For $\bz \in \RR^{d}$, define $H(\bz)=g(\bz)z_iz_j$. Then by Stein's Lemma, we have $$\EE_{\bz \sim \cN(\bzero_r,\bI_r)}\brac{{H}(\bz)}=\EE_{\bz \sim \cN(\bzero_r,\bI_r)}\brac{\nabla^2g(\bz)}_{i,j}+\delta_{i,j}\EE_{\bz \sim \cN(\bzero_r,\bI_r)}\brac{g(\bz)}.$$
   Moreover, let $R>0$ be a truncation radius and we define $\overline{H}(\bz)=H(\text{clip}(\bz,R))$. Here the clipping function is defined as $$\text{clip}(\bz,R)_i=\max(\min(z_i,R),-R),~~i=1,2,\dots,r.$$
   
By \eqref{equ::bound grad g}, we know $g({\rm clip}(\bz,R))$ is $\cO(Rr^{\frac{p-1}{2}})$-Lipschitz  continuous, so $\overline{H}$ has a Lipschitz constant of $\cO(R^3r^{\frac{p-1}{2}})$. Thus, by Lemma \ref{lemma::Gaussian W1}, we have
    \begin{align*}
&\abs{\EE_{\bx}\brac{g({\rm clip}\paren{\bx^{\top}\bA_1\bx,\dots,\bx^{\top}\bA_r\bx),R}}-\EE_{\bz \sim \cN(\bzero_r,\bI_r)}\brac{g({\rm clip}(\bz,R))}}\lesssim \frac{Rr^{\frac{p-1}{2}}\cdot r^2\normA}{\sqrt{d}},\\
        &\abs{\EE\brac{\overline{H}(\bx^{\top}\bA_1\bx,\dots,\bx^{\top}\bA_r\bx)}-\EE_{\bz \sim \cN(\bzero_r,\bI_r)}\brac{\overline{H}(\bz)}} \lesssim \frac{R^3r^{\frac{p-1}{2}}\cdot r^2\normA}{\sqrt{d}}.
    \end{align*}
    Since $\EE_{\bx}[(\bx^{\top}\bA_k\bx)^2]=1$ for any $k\in[r]$, by Lemma \ref{lemma::poly concentration sphere}, choosing $R=C\log d$ for an appropriate constant $C$ can ensure that 
    \begin{align*}
        &\abs{\EE\brac{\overline{H}(\bx^{\top}\bA_1\bx,\dots,\bx^{\top}\bA_r\bx)-{H}(\bx^{\top}\bA_1\bx,\dots,\bx^{\top}\bA_r\bx)}}\le \frac{1}{d},\\
        &\abs{\EE_{\bx}\brac{g({\rm clip}\paren{\bx^{\top}\bA_1\bx,\dots,\bx^{\top}\bA_r\bx),R}-g\paren{\bx^{\top}\bA_1\bx,\dots,\bx^{\top}\bA_r\bx}}}\le \frac{1}{d},\\
        &\abs{\EE_{\bz \sim \cN(\bzero_r,\bI_r)}\brac{H(\bz)-\overline{H}(\bz)}} \le \frac{1}{d},\\
        & \abs{\EE_{\bz \sim \cN(\bzero_r,\bI_r)}\brac{g({\rm clip}(\bz,R))-g(\bz)}} \le \frac{1}{d}.
    \end{align*}
    Altogether, we have
    \begin{align*}    \abs{\EE\brac{g^{\star}(\bx^{\top}\bA_1\bx,\dots,\bx^{\top}\bA_r\bx)\langle \bA_i,\bx\bx^{\top}-\bI\rangle \langle \bA_j,\bx\bx^{\top}-\bI\rangle}-\EE_{\bz \sim \cN(\bzero_r,\bI_r)}\brac{\nabla^{2}g(\bz)}_{i,j}} \le \frac{Lr^2\normA \log^2 d}{\sqrt{d}}.
    \end{align*}
    Here $L=C_{g}Rr^{\frac{p-1}{2}}$ for some constant $C_g>0$ is the Lipschitz constant of $g^{\star}$ that holds with high probability. The proof is complete.
\end{proof}

Following the above proof and replacing $\bA_i$ and $\bA_j$ by any other traceless matrices $\bB_1$ and $\bB_2$ that are orthogonal to all $\bA_k$, we directly have the following corollary:
\begin{corollary}
\label{lemma::AiB weak}
    For any two matrices $\bB_1, \bB_2 \in\RR^{d\times d}$ satisfying $\EE\brac{\bx^{\top}\bB_j\bx}=0$ and $\langle \bB_j,\bA_i \rangle=0$ for any $i=1,2,\dots, r$ and $j=1,2$, we have
    \begin{align*}    &\quad\abs{\EE\brac{g^{\star}(\bx^{\top}\bA_1\bx,\dots,\bx^{\top}\bA_r\bx)\langle \bA_i,\bx\bx^{\top}-\bI\rangle \langle \bB_j,\bx\bx^{\top}-\bI\rangle}} \\
    &\lesssim \paren{\frac{r\normA}{\sqrt{d}}+\frac{\norm{\bB_j}{op}}{\norm{\bB_j}{F}}}\norm{\bB_j}{F}Lr\log^2 d,
    \end{align*}
    for any $j=1, 2$, and
    \begin{align*}
    &\quad\abs{\EE\brac{g^{\star}(\bx^{\top}\bA_1\bx,\dots,\bx^{\top}\bA_r\bx)\langle \bB_1,\bx\bx^{\top}-\bI\rangle \langle \bB_2,\bx\bx^{\top}-\bI\rangle}} \\&\lesssim \paren{\frac{r\normA}{\sqrt{d}}+\frac{\norm{\bB_1}{op}}{\norm{\bB_1}{F}}+\frac{\norm{\bB_2}{op}}{\norm{\bB_2}{F}}}\norm{\bB_1}{F}\norm{\bB_2}{F}Lr\log^2 d.
    \end{align*}
\end{corollary}
Also, by \eqref{equ::W1 u} we know for any unit vector $u \in \RR^{d}$,  $(\bx^{\top}\bA_1\bx,\bx^{\top}\bA_2\bx,\dots,\bx^{\top}\bA_r\bx,\bu^{\top}\bx)$ is approximately Gaussian when $d$ is sufficiently large, which gives rise to the following lemma by the same deduction.
\begin{corollary}
\label{lemma::Ai u}
    For any $i\in[r]$ and unit vector $\bu_1,\bu_2 \in \RR^{d}$ and matrix $\bB$ satisfying the same requirements in Lemma \ref{lemma::AiB strong}, we have
    \begin{align*}    &\abs{\EE\brac{g^{\star}(\bx^{\top}\bA_1\bx,\dots,\bx^{\top}\bA_r\bx)\left\langle \bA_i,\bx\bx^{\top}-\bI\right\rangle((\bu_j^{\top}\bx)^2-1)}} \lesssim \frac{Lr^2\normA \log^2 d}{\sqrt{d}},\\
    &\abs{\EE\brac{g^{\star}(\bx^{\top}\bA_1\bx,\dots,\bx^{\top}\bA_r\bx)\left\langle \bB,\bx\bx^{\top}-\bI\right\rangle((\bu_j^{\top}\bx)^2-1)}} \lesssim \paren{\frac{r\normA}{\sqrt{d}}+\frac{\norm{\bB_j}{op}}{\norm{\bB_j}{F}}}\norm{\bB_j}{F}Lr\log^2 d
    \end{align*}
    for any $j=1,2$, and we further have
    \begin{align*}
        \abs{\EE\brac{g^{\star}(\bx^{\top}\bA_1\bx,\dots,\bx^{\top}\bA_r\bx)((\bu_1^{\top}\bx)^2-1)((\bu_2^{\top}\bx)^2-1)}} \lesssim \frac{Lr^2\normA \log^2 d}{\sqrt{d}}.
    \end{align*}
\end{corollary}
With the lemmas above, we can derive a stronger version of Corollary \ref{lemma::AiB weak}, i.e., Lemma \ref{lemma::AiB strong}, in which the error gets rid of the dependence on $\norm{\bB}{op}$.

\begin{proof}[Proof of Lemma \ref{lemma::AiB strong}]
    Let $\tau > 1/\sqrt{d}$ be a threshold to be determined later. Decompose $\bB$ as follows:
    \begin{align*}
        \bB = \sum_{j=1}^d \lambda_j \bu_j\bu_j^{\top} = \sum_{\abs{\lambda_j} > \tau} \lambda_j \left(\bu_j \bu_j^{\top} - \frac{1}{d}\bI\right) - \sum_{k=1}^{r}\frac{1}{\norm{\bA_k}{F}^2} \sum_{\abs{\lambda_j} > \tau}\lambda_j \bu_j^{\top}\bA_k\bu_j \cdot \bA_k + \widetilde \bB,
    \end{align*}
    where $\set{u_j}_{i=1}^{d}$ are orthogonal unit vectors and
    \begin{align*}
        \widetilde \bB = \sum_{\abs{\lambda_j} \le \tau} \lambda_j\bu_j\bu_j^{\top} + \bI \cdot \frac{1}{d}\sum_{\abs{\lambda_j} > \tau}\lambda_j + \sum_{k=1}^{r}\frac{1}{\norm{\bA_k}{F}^2} \sum_{\abs{\lambda_j} > \tau}\lambda_j \bu_j^{\top}\bA_k\bu_j \cdot \bA_k.
    \end{align*}
    By construction, we have
    \begin{align*}
        \tr(\widetilde \bB) = \sum_{\abs{\lambda_j} \le \tau} \lambda_j + \sum_{\abs{\lambda_j} > \tau}\lambda_j = \sum_{j \in [d]} \lambda_j = 0.
    \end{align*}
    Moreover, for any $k \in[r]$, we have
    \begin{align*}
        \langle \widetilde \bB, \bA_k \rangle = \sum_{\abs{\lambda_j} \le \tau} \lambda_j \bu_j^{\top}\bA_k\bu_j + \sum_{\abs{\lambda_j} > \tau}\lambda_j \bu_j^{\top}\bA_k\bu_j = \langle \bA_k, \bB \rangle = 0.
    \end{align*}
    Therefore by Lemma \ref{lemma::AiB weak},  we have
    \begin{align}
&\quad\abs{\EE\brac{g^{\star}(\bx^{\top}\bA_1\bx,\dots,\bx^{\top}\bA_r\bx)\langle \bA_i,\bx\bx^{\top}-\bI\rangle \langle \widetilde \bB,\bx\bx^{\top}-\bI\rangle}} \nonumber\\
&\lesssim \paren{\frac{r\normA}{\sqrt{d}}+\frac{\norm{\widetilde \bB}{op}}{\norm{\widetilde \bB}{F}}}\norm{\widetilde \bB}{F}Lr\log^2 d .\label{equ::bound tilde B}
    \end{align}
    Since $\sum_{j=1}^{d}\lambda^2_j=\norm{\bB}{\rm F}^{2}=(d+2)/(2d
)=\cO(1)$,  there are at most $O(\tau^{-2})$ indices $j$ satisfying $\abs{\lambda_j} > \tau$, which gives rise to
    \begin{align*}
        \sum_{\abs{\lambda_j} > \tau}\abs{\lambda_j} \lesssim \sqrt{\tau^{-2} \cdot \sum_{\abs{\lambda_j} > \tau}\abs{\lambda_j}^2} \le \tau^{-1}.
    \end{align*}
    Thus, we can bound the Frobenius norm of $\widetilde \bB$ by
    \begin{align*}
        \norm{\widetilde \bB}{F}^2 &\lesssim \sum_{\abs{\lambda_j}\le \tau}\lambda_j^2 + \frac{1}{d}\left(\sum_{\lambda_j > \tau}\lambda_j\right)^2 + \norm{\sum_{k=1}^{r}\frac{1}{\norm{\bA_k}{F}^2}\sum_{\abs{\lambda_j} > \tau}\lambda_j \bu_j^{\top}\bA_k\bu_j\cdot \bA_k}{F}^2\\
        &= \norm{\widetilde \bB}{F}^2 \sum_{\abs{\lambda_j}\le \tau}\lambda_j^2 + \frac{1}{d}\left(\sum_{\lambda_j > \tau}\lambda_j\right)^2 + \sum_{k=1}^{r}\paren{\sum_{\abs{\lambda_j} > \tau}\lambda_j \bu_j^{\top}\bA_k\bu_j}^2 \\
        &\lesssim  \sum_{\abs{\lambda_j}\le \tau}\lambda_j^2 +\left(\frac{1}{d}+\sum_{k=1}^{r}\norm{\bA_k}{op}^2\right) \paren{\sum_{\abs{\lambda_j} > \tau} \lambda_j}^2\\
        &\lesssim 1 + \frac{r\normA^2}{d\tau^2}.
    \end{align*}
    Thus, we have $ \norm{\widetilde \bB}{F}\lesssim 1+\frac{\sqrt{r}\normA}{\sqrt{d}\tau}$
    and
    \begin{align*}
        \norm{\widetilde \bB}{op} &\le \tau + \left(\frac{1}{d} + \sum_{k=1}^r\norm{\bA_k}{op}\right)\abs{\sum_{\abs{\lambda_j} > \tau}\lambda_ju_j^{\top}Au_j}\\
        &\lesssim \tau + \frac{r\normA}{d\tau}.
    \end{align*}
    Thus, plugging the norm bounds into \eqref{equ::bound tilde B}, we obtain that
    \begin{align*}
    \abs{\EE\brac{g^{\star}(\bx^{\top}\bA_1\bx,\dots,\bx^{\top}\bA_r\bx)\langle \bA_i,\bx\bx^{\top}-\bI\rangle \langle \widetilde \bB,\bx\bx^{\top}-\bI\rangle}} 
    \lesssim \left(
     \frac{r \normA}{\sqrt{d}} + \frac{r^{3/2}\normA^2}{d\tau}+\tau + \frac{r\normA}{d\tau}\right)Lr \log^2 d.
    \end{align*}  
    Next, applying Corollary \ref{lemma::Ai u}  with $\bu=\bu_1$, $\bu_2$, $\dots$, $\bu_d$, we have
    \begin{align*}    \abs{\EE\brac{g^{\star}(\bx^{\top}\bA_1\bx,\dots,\bx^{\top}\bA_r\bx)\left\langle \bA_i,\bx\bx^{\top}-\bI\right\rangle\langle\bu_j\bu_j^{\top}-
    \bI/d,\bx\bx^{\top}-\bI\rangle}} \lesssim \frac{Lr^2\normA \log^2 d}{\sqrt{d}},~~~\forall j \in[d].
    \end{align*}
    Thus, we have
    \begin{align*}
&\quad \abs{\EE\brac{g^{\star}(\bx^{\top}\bA_1\bx,\dots,\bx^{\top}\bA_r\bx)\left\langle \bA_i,\bx\bx^{\top}-\bI\right\rangle \left\langle  \sum_{\abs{\lambda_j} > \tau}\lambda_j\left(\bu_j\bu_j^{\top} - \frac{1}{d}\bI\right),\bx\bx^{\top}-\bI\right\rangle}} \\
         &\le \sum_{\abs{\lambda_j} > \tau}\abs{\lambda_j}\frac{Lr^2\normA \log^2 d}{\sqrt{d}}\\
        &\lesssim \frac{Lr^2\normA \log^2 d}{\sqrt{d}\tau}.
    \end{align*}
    Besides, by Lemma \ref{lemma::Ai Aj}, we have
    \begin{align*}
    &\quad \abs{\EE\brac{g^{\star}(\bx^{\top}\bA_1\bx,\dots,\bx^{\top}\bA_r\bx)\left\langle \bA_i,\bx\bx^{\top}-\bI\right\rangle \left\langle \sum_{k=1}^{r}\frac{1}{\norm{\bA_k}{F}^2} \sum_{\abs{\lambda_j} > \tau}\lambda_j \bu_j^{\top}\bA_k\bu_j \cdot \bA_k ,\bx\bx^{\top}-\bI\right\rangle}}\\
    &\lesssim \sum_{k=1}^{r}\abs{\sum_{\abs{\lambda_j} > \tau}\lambda_j \bu_j^{\top}\bA_k\bu_j} \left(\EE_{\bz \sim \cN(\bzero_r,\bI_r)}\brac{\nabla^2 g(\bz)}_{k,i} +\frac{Lr^2\normA \log^2 d}{\sqrt{d}} \right)\\
    &\lesssim \sum_{k=1}^{r} L\abs{\sum_{\abs{\lambda_j} > \tau}\lambda_j \bu_j^{\top}\bA_k\bu_j} \\ &\lesssim \frac{Lr\normA}{\sqrt{d}\tau}.
    \end{align*}
    Altogether, we have
    \begin{align*}
&\quad\abs{\EE\brac{g^{\star}(\bx^{\top}\bA_1\bx,\dots,\bx^{\top}\bA_r\bx)\langle \bA_i,\bx\bx^{\top}-\bI\rangle \langle \bB,\bx\bx^{\top}-\bI\rangle}}\\
        &\lesssim \left(
     \frac{r \normA}{\sqrt{d}} + \frac{r^{3/2}\normA^2}{d\tau}+\tau + \frac{r\normA}{d\tau}\right)Lr \log^2 d+ \frac{Lr^2\normA \log^2 d}{\sqrt{d}\tau}+ \frac{Lr\normA}{\sqrt{d}\tau}\\
        &\lesssim d^{-1/4}Lr^2\normA\log^2 d. 
    \end{align*}
    where we set $\tau = \normA d^{-1/4}$. The proof is complete.
\end{proof}
Following the proof above, we can complete the proof of Lemma \ref{lemma::B1B2}.
 \begin{proof}[Proof of Lemma \ref{lemma::B1B2}]
Similar to the proof of Lemma \ref{lemma::AiB strong}, we decompose $\bB_1$ and $\bB_2$ as follows:
  \begin{align*}
        \bB_i &= \sum_{j=1}^d \lambda_{i,j} \bu_{i,j}\bu_{i,j}^{\top}\\
        &= \sum_{\abs{\lambda_{i,j}} > \tau} \lambda_{i,j} \left(\bu_{i,j} \bu_{i,j}^{\top} - \frac{1}{d}\bI\right) - \sum_{k=1}^{r}\frac{1}{\norm{\bA_k}{F}^2} \sum_{\abs{\lambda_{i,j}} > \tau}\lambda_{i,j} \bu_{i,j}^{\top}\bA_k\bu_{i,j} \cdot \bA_k + \widetilde \bB_i,
    \end{align*}
    where $\set{u_{i,j}}_{j=1}^{d}$ are orthogonal unit vectors for $i=1,2$, respectively, and
    \begin{align*}
        \widetilde \bB_i = \sum_{\abs{\lambda_{i,j}} \le \tau} \lambda_{i,j}\bu_{i,j}\bu_{i,j}^{\top} + \bI \cdot \frac{1}{d}\sum_{\abs{\lambda_{i,j}} > \tau}\lambda_{i,j} + \sum_{k=1}^{r}\frac{1}{\norm{\bA_k}{F}^2} \sum_{\abs{\lambda_{i,j}} > \tau}\lambda_{i,j} \bu_{i,j}^{\top}\bA_k\bu_{i,j} \cdot \bA_k.
    \end{align*}
    Then following the proof of Lemma \ref{lemma::AiB strong}, we know for any $i=1,2$ and $k=1,2,\dots,r$,
    \begin{align*}
        \tr\paren{\widetilde \bB_{i}}=\langle \widetilde \bB_i, \bA_k \rangle=0,~~\norm{\widetilde \bB_{i}}{\rm F} \lesssim 1 + \frac{\sqrt{r}\normA}{\sqrt{d}\tau},  ~~\norm{\widetilde \bB_{i}}{\rm op} \lesssim \tau+\frac{r\normA}{d\tau}, ~~\sum_{\abs{\lambda_{i,j}} > \tau}\abs{\lambda_{i,j}}\lesssim \tau^{-1} .
    \end{align*}
    Let's denote the bi-linear operator $\Gamma(\cdot,\cdot):\RR^{d\times d}\times \RR^{d\times d} \rightarrow \RR$ being 
    \begin{align}\label{equ::define operator}
        \Gamma(\bA,\bB)=\EE\brac{f^{\star}(\bx)\langle\bA,\bx\bx^{\top}-\bI\rangle\langle\bB,\bx\bx^{\top}-\bI\rangle}.
    \end{align}
    By Corollary \ref{lemma::AiB weak} and the proof of Lemma \ref{lemma::AiB strong}, we have 
  \begin{align}
&\abs{\Gamma(\widetilde \bB_1,\widetilde \bB_2) }  \lesssim\left(\frac{r\normA}{\sqrt{d}}\left(1+\frac{r\normA^2}{d\tau^2}\right)+\left(\tau+\frac{r\normA}{d\tau}\right)\left(1+\frac{\sqrt{r}}{\sqrt{d}\tau}\right)\right)Lr \log^2 d. \label{equ::tildeb1b2} \\
& \abs{\Gamma\paren{\widetilde \bB_i, \sum_{\abs{\lambda_{-i,j}} > \tau}\lambda_{-i,j}\left(\bu_{-i,j}\bu_{-i,j}^{\top} - \bI/d\right)} }\nonumber \\&\hspace{5em}\lesssim \tau^{-1}\left(
     \frac{r \normA}{\sqrt{d}} + \frac{r^{3/2}\normA^2}{d\tau}+\tau + \frac{r\normA}{d\tau}\right)Lr \log^2 d   \label{equ::tildeB1 u}     \\
        & \abs{\Gamma\paren{\widetilde \bB_i, \sum_{k=1}^{r}\frac{1}{\norm{\bA_k}{F}^2} \sum_{\abs{\lambda_{-i,j}} > \tau}\lambda_{-i,j} \bu_{-i,j}^{\top}\bA_k\bu_{-i,j} \cdot \bA_k }}\nonumber \\&\hspace{5em} \lesssim \tau^{-1}\left(
     \frac{r \normA}{\sqrt{d}} + \frac{r^{3/2}\normA^2}{d\tau}+\tau + \frac{r\normA}{d\tau}\right)Lr \log^2 d   . \label{equ::tildeB1 A}
    \end{align}
    Here $-i$ means $2$ when $i=1$ and $1$ when $i=2$.
Moreover, by Lemma \ref{lemma::Ai Aj}, we have 
  \begin{align}
        & \quad \abs{\Gamma\paren{\sum_{k=1}^{r}\frac{1}{\norm{\bA_k}{F}^2} \sum_{\abs{\lambda_{1,j}} > \tau}\lambda_{1,j} \bu_{1,j}^{\top}\bA_k\bu_{1,j} \cdot \bA_k , \sum_{k=1}^{r}\frac{1}{\norm{\bA_k}{F}^2} \sum_{\abs{\lambda_{2,j}} > \tau}\lambda_{2,j} \bu_{2,j}^{\top}\bA_k\bu_{2,j} \cdot \bA_k }} \nonumber \\ 
        &\lesssim \paren{\sum_{k=1}^{r} \abs{\sum_{\abs{\lambda_{1,j}} > \tau}\lambda_{1,j} \bu_{1,j}^{\top}\bA_k\bu_{1,j}}}\paren{\sum_{k=1}^{r} \abs{\sum_{\abs{\lambda_{2,j}} > \tau}\lambda_{2,j} \bu_{2,j}^{\top}\bA_k\bu_{2,j}}}\left(1 +\frac{Lr^2\normA \log^2 d}{\sqrt{d}} \right) \nonumber\\
        &  \lesssim \frac{r^2\normA^2}{d\tau^2} \label{equ::AA},
    \end{align}
and 
\begin{align}
        & \quad \abs{\Gamma\paren{\sum_{k=1}^{r}\frac{1}{\norm{\bA_k}{F}^2} \sum_{\abs{\lambda_{i,j}} > \tau}\lambda_{i,j} \bu_{i,j}^{\top}\bA_k\bu_{i,j} \cdot \bA_k ,\sum_{\abs{\lambda_{-i,j}} > \tau}\lambda_{-i,j}\left(\bu_{-i,j}\bu_{-i,j}^{\top} - \bI/d\right) } }\nonumber \\ 
        &\lesssim \paren{\sum_{k=1}^{r} \abs{\sum_{\abs{\lambda_{i,j}} > \tau}\lambda_{i,j} \bu_{i,j}^{\top}\bA_k\bu_{i,j}}}\paren{\sum_{\abs{\lambda_{-i,j}} > \tau}\abs{\lambda_{-i,j}}}\frac{Lr^2\normA \log^2 d}{\sqrt{d}}\nonumber \\
        &\lesssim \frac{r\normA}{\sqrt{d}\tau^2}\cdot \frac{Lr^2\normA \log^2 d}{\sqrt{d}}.
    \end{align}
Finally, we have 

    \begin{align}
        & \quad \abs{\Gamma\paren{\sum_{\abs{\lambda_{1,j}} > \tau}\lambda_{1,j}\left(\bu_{1,j}\bu_{1,j}^{\top} - \bI/d\right),\sum_{\abs{\lambda_{2,j}} > \tau}\lambda_{2,j}\left(\bu_{2,j}\bu_{2,j}^{\top} - \bI/d\right) } }\nonumber \\ 
        &\lesssim \paren{\sum_{\abs{\lambda_{1,j}} > \tau}\abs{\lambda_{1,j}}}\paren{\sum_{\abs{\lambda_{2,j}} > \tau}\abs{\lambda_{2,j}}}\frac{Lr^2\normA \log^2 d}{\sqrt{d}}.\nonumber \\
        &\lesssim \frac{1}{\tau^2}\cdot \frac{Lr^2\normA \log^2 d}{\sqrt{d}} .\label{equ::u1u2}
    \end{align}
Summing \eqref{equ::tildeb1b2} to \eqref{equ::u1u2} altogether, we have 
\begin{align*}
    \abs{\Gamma(\bB_1,\bB_2)}&\lesssim \left(\frac{r\normA}{\sqrt{d}}\left(1+\frac{r\normA^2}{d\tau^2}\right)+\left(\tau+\frac{r\normA}{d\tau}\right)\left(1+\frac{\sqrt{r}}{\sqrt{d}\tau}\right)\right)Lr \log^2 d\\
    &\quad+\tau^{-1}\left(
     \frac{r \normA}{\sqrt{d}} + \frac{r^{3/2}\normA^2}{d\tau}+\tau + \frac{r\normA}{d\tau}\right)Lr \log^2 d  \\
    &\quad+\frac{r}{\sqrt{d}\tau} \cdot \frac{Lr^2\normA \log^2 d}{\sqrt{d}}+\frac{r^2\normA^2}{d\tau^2}+\frac{r\normA}{\sqrt{d}\tau^2}\cdot \frac{Lr^2\normA \log^2 d}{\sqrt{d}}+ \frac{1}{\tau^2}\cdot \frac{Lr^2\normA \log^2 d}{\sqrt{d}}\\
    &\lesssim  d^{-1/6}Lr^2\normA\log^2 d,
\end{align*}
where we take $\tau=\normA d^{-1/6}$. The proof is complete.
   \end{proof}

\subsection{Boundedness of the learned feature}\label{sec::bounded feature}
In this section, we aim to upper bound the magnitude of the learned feature $\langle \bw, \mathbf{h}^{(1)}(\bx') \rangle$. Since we focus on the first training stage throughout this section, we denote $n=n_1$ for notation simplicity when the context is clear, and let the training set be $\cD_1=\{\bx_1,\bx_2,\dots,\bx_n\}$. We have the following proposition:
\begin{proposition}\label{prop::bounded feature}
   Suppose $m_2\geq d^{4}C_{\sigma}^4$ and $n \geq C\iota^2d^2$ for some sufficiently large $C$. With high probability jointly on $\bV$ and the training dataset $\cD_1$, and with probability at least $1-4n\exp(-\iota^2/2)$ on $\bw$, for any $\bx' \in \cD_2$, we have
   \begin{align*}
       \abs{\langle \bw, \mathbf{h}^{(1)}(\bx') \rangle} \lesssim \frac{\iota^{p+2}}{d^3}+\frac{\iota^{p+2}\sqrt{m_2}}{d^4\sqrt{n}}+\frac{\sqrt{m_2}\iota^3 \log^2(m_2n_2)}{d^6}\cdot\paren{\norm{\cP_{>2}(f^{\star})}{L^2}+\sqrt{d}\norm{\cP_{2}(f^{\star})}{L^2}}.
   \end{align*}
\end{proposition}
As a corollary, when $m_2 \gtrsim d^4\iota^{2p+4}$, $n \gtrsim d^{4}\iota^{2p+4}$ and $\norm{\cP_{2}(f^{\star})}{L^2} \lesssim \frac{\normP}{\sqrt{d}}$,  we have for any $\bx'\in \cD_2$,
\begin{align}\label{equ::feature bounds final}
    \frac{1}{\sqrt{m_2}}\abs{\langle \bw, \mathbf{h}^{(1)}(\bx') \rangle} \lesssim \frac{\normP\iota^{{5}}}{d^6}.
\end{align}
Thus, by taking the learning rate $\eta=Cm_2{-1/2}\normP^{-1}\iota^{-5}d^6$ for an appropriate constant $C>0$, we can ensure that $\abs{\eta\langle \bw, \mathbf{h}^{(1)}(\bx') \rangle}\le 1 $ with high probability.
\begin{proof}[Proof of Proposition \ref{prop::bounded feature}]
Note that
    \begin{align*}
        \frac{1}{m_2}\langle \bw, \mathbf{h}^{(1)}(\bx') \rangle
&=\frac{1}{m_2n}\sum_{i=1}^{n}\sum_{j=1}^{m_2}w_jf^{\star}(\bx_i)K^{(0)}_{m_2}(\bx_i,\bx')\sigma_2\paren{\bv_j^{\top}\bx_i}\\
        &=\frac{1}{m_2}\sum_{j=1}^{m_2}\frac{1}{n}\sum_{i=1}^{n}f^{\star}(\bx_i)K^{(0)}_{m_2}(\bx_i,\bx')w_j\sigma_2\paren{\bv_j^{\top}\bx_i}.
    \end{align*}
We do a decomposition as follows 
\begin{align*}
    &\quad\frac{1}{m_2}\sum_{j=1}^{m_2}\frac{1}{n}\sum_{i=1}^{n}f^{\star}(\bx_i)K^{(0)}_{m_2}(\bx_i,\bx')w_j\sigma_2\paren{\bv_j^{\top}\bx_i}\\&=\underbrace{\frac{1}{m_2}\sum_{j=1}^{m_2}\frac{1}{n}\sum_{i=1}^{n}f^{\star}(\bx_i)\paren{K^{(0)}_{m_2}(\bx_i,\bx')-K^{(0)}(\bx_i,\bx')}w_j\sigma_2\paren{\bv_j^{\top}\bx_i}}_{A_1}\\
    &\quad +\underbrace{\frac{1}{m_2}\sum_{j=1}^{m_2}\frac{1}{n}\paren{\sum_{i=1}^{n}f^{\star}(\bx_i)K^{(0)}(\bx_i,\bx')w_j\sigma_2\paren{\bv_j^{\top}\bx_i}-\EE_{\bx}\brac{f^{\star}(\bx)K^{(0)}(\bx,\bx')w_j\sigma_2\paren{\bv_j^{\top}\bx}}}}_{A_2}\\
    &\quad +\underbrace{\frac{1}{m_2}\sum_{j=1}^{m_2}w_j\EE_{\bx}\brac{f^{\star}(\bx)K^{(0)}(\bx,\bx')\sigma_2\paren{\bv_j^{\top}\bx}}}_{A_3}.
\end{align*}
We consider derive an upper bound on $A_1$, $A_2$ and $A_3$, respectively.

\begin{lemma}[Bound $A_1$]\label{lemma:: bound A1}
    Suppose $m_2\geq d^{4}C_{\sigma}^4$. With high probability jointly on $\bV$ and the training dataset $\cD_1$, and with probability at least $1-2n\exp(-\iota^2/2)$ on $\bw$, for any $\bx' \in \cD_2$, we have $$\abs{A_1}\lesssim \frac{\iota^{p+2}}{m_2d^3}.$$
\end{lemma}

\begin{lemma}[Bound $A_2$]\label{lemma:: bound A2}
    Suppose $m_2\geq d^{4}C_{\sigma}^4$ and $n \geq C\iota^2 d^2$ for some sufficiently large $C$. With high probability on the training dataset $\cD_1$, for any $\bx' \in \cD_2$, we have $$\abs{A_2}\lesssim \frac{\iota^{p+2}}{d^4\sqrt{m_2n}}.$$
\end{lemma}
\begin{lemma}[Bound $A_3$]\label{lemma:: bound A3}
    Suppose $m_2\geq d^{4}C_{\sigma}^4$ for some sufficiently large $C$. With high probability jointly on $\bV$ and the training dataset $\cD_1$, and with probability at least $1-2n_2\exp(-\iota^2/2)$ on $\bw$, we have \begin{align*}
   \abs{A_3}\lesssim \frac{\iota^3 \log^2(m_2n_2)}{\sqrt{m_2}d^6}\cdot\paren{\norm{\cP_{>2}(f)}{L^2}+\sqrt{d}\norm{\cP_{2}(f)}{L^2}}.
\end{align*}
Similarly, for a single point $\bx'$, with high probability on $\bV$ and $\cD_1$, with probability $1-2\exp(-\iota^2/2)$ on $\bw$, we have 
\begin{align*}
    \abs{A_3} \lesssim \frac{\iota^3 \log^2(m_2n_2)}{\sqrt{m_2}d^6}\cdot\paren{\norm{\cP_{>2}(f)}{L^2}+\sqrt{d}\norm{\cP_{2}(f)}{L^2}}.
\end{align*}
\end{lemma}
The proof of the three lemmas are provided in Appendix \ref{sec::proof bounded feature}. Combining the results in the three lemmas above directly concludes our proof.
\end{proof}

\subsubsection{Omitted proofs for Proposition \ref{prop::bounded feature}}\label{sec::proof bounded feature}
\begin{proof}[Proof of Lemma \ref{lemma:: bound A1}]
We can rewrite $A_1$ as
\begin{align*}
    \abs{A_1} =\frac{1}{n}\sum_{i=1}^{n}f^{\star}(\bx_i)\paren{K^{(0)}_{m_2}(\bx_i,\bx')-K^{(0)}(\bx_i,\bx')}\frac{1}{m_2}\sum_{j=1}^{m_2}w_j\sigma_2\paren{\bv_j^{\top}\bx_i}
\end{align*}
Since $\frac{1}{\sqrt{m_2}}\sum_{j=1}^{m_2}w_j\sigma_2\paren{\bv_j^{\top}\bx} \sim \cN\paren{0,\frac{1}{m_2}\sum_{j=1}^{m_2}\sigma_2\paren{\bv_j^{\top}\bx}^2}$, we know given any $\bx$ and $\bV$, with probability at least $1-2\exp(-\iota^2/2)$ on $\bw$, we have
\begin{align*}
    \abs{\frac{1}{m_2}\sum_{j=1}^{m_2}w_j\sigma_2\paren{\bv_j^{\top}\bx}}\le \frac{\iota}{\sqrt{m_2}} \cdot \sqrt{ \frac{1}{m_2}\sum_{j=1}^{m_2}\sigma_2\paren{\bv_j^{\top}\bx}^2}
\end{align*}
Moreover, by \eqref{equ::initial kernel concentration} in the proof of Lemma \ref{lemma::concentration initial kernel}, we know  for any $\bx$ and $t>0$, we have 
\begin{align*}
&\quad\Pr\brac{ \frac{1}{m_2}\sum_{j=1}^{m_2}\paren{\sigma_2\paren{\bv_j^{\top}\bx}^2-\EE_{\bv_j}\brac{\sigma_2(\bv_j^{\top}\bx)^2}}\geq\sqrt{\frac{t}{m_2}}}\\
&=\Pr\brac{\abs{K^{(0)}_{m_2}(\bx,\bx)-K^{(0)}(\bx,\bx)} \geq \sqrt{\frac{t}{m_2}}}
       \\
       &\le2\exp\paren{\frac{-{t}/2}{\frac{C_4}{d^4}+\frac{C_{\sigma}^2}{3}\sqrt{\frac{t}{m_2}}}}. 
\end{align*}
Altogether, when $m_2\geq d^{4}C_{\sigma}^4$, by taking $t=C^2\iota^2/d^{4}$ for sufficiently large $C$ and union bounding over the dataset $\cD_1$, we can ensure that with probability at least $1-n\exp(-\iota)$ on $\bV$ and at least $1-2n\exp(-\iota^2/2)$ on $\bw$, i.e., high probability on $\bw$, $\bV$, we have 
\begin{align}
     \abs{\frac{1}{m_2}\sum_{j=1}^{m_2}w_j\sigma_2\paren{\bv_j^{\top}\bx}}\le \frac{C\iota}{\sqrt{m_2}}\sqrt{\frac{C_2}{d^2}+\frac{\iota}{\sqrt{m_2}d^2}}\le \frac{\iota C_3}{\sqrt{m_2}d}, ~~~\forall \bx \in \cD_1. \label{equ:: bound A1 1}
\end{align}
Here $C_3$ is a constant. We denote this joint event by $E_1$. 
On the other hand, by Lemma \ref{lemma::concentration initial kernel}, with high probability on $\bV$, we have for any $\bx_i\in \cD_1$ and $\bx' \in \cD_2$,
\begin{align}
   \abs{ K^{(0)}_{m_2}(\bx_i,\bx')-K^{(0)}(\bx_i,\bx')} \le \frac{\iota}{\sqrt{m_2}d^2}. 
\end{align}
We denote this event by $E_2$.
Last, we truncate the range of the target function $f$. Denoting the truncation radius as $R=(C\eta)^{p}$ for a sufficient large constant $C$ and $\eta=\log(dm_1m_2n_1n_2)$  $\Pr\brac{\abs{f(\bx)}\geq R} \le 2e^{-2\eta}$ (this could be guaranteed by Lemma \ref{lemma: sphere hypercontractivity}). Given $n$ i.i.d. samples $\bx_1,\bx_2\dots,\bx_n \sim \mathbb{S}^{d-1}(\sqrt{d})$, we have 
\begin{align}\label{event::E3}
    \Pr\brac{\abs{f(\bx_i)}\le R, \forall i\in[n]} \geq 1- 2ne^{-2\eta}.
\end{align}
Thus, with high probability on the dataset $\cD$, we have $\abs{f(\bx)}\le \iota^p$ for any $\bx \in \cD_1$. We denote this event by $E_3$. Thus, combining  \eqref{equ:: bound A1 1} \eqref{equ:: bound A1 2} and the truncation radius of $f$, we directly have 
\begin{align*}
    \abs{A_1} \le \iota^p \cdot \frac{\iota}{\sqrt{m_2}d^2} \cdot \frac{\iota C_3}{\sqrt{m_2}d}=\frac{C_3\iota^{p+2}}{m_2d^3}.
\end{align*}
with high probability (under events $E_1$, $E_2$ and $E_3$). The proof is complete.
\end{proof}

\begin{proof}[Proof of Lemma \ref{lemma:: bound A2}]
We rewrite $A_2$ as 
\begin{align*}
    A_2&=\frac{1}{n}\sum_{i=1}^{n}\EE_{\bx}\brac{f^{\star}(\bx_i)K^{(0)}(\bx_i,\bx') \frac{1}{m_2}\sum_{j=1}^{m_2}w_j\sigma_2\paren{\bv_j^{\top}\bx_i}-f^{\star}(\bx)K^{(0)}(\bx,\bx') \frac{1}{m_2}\sum_{j=1}^{m_2}w_j\sigma_2\paren{\bv_j^{\top}\bx}}.
\end{align*}

Denote $Y(\bx)=f^{\star}(\bx)K^{(0)}(\bx,\bx') \frac{1}{m_2}\sum_{j=1}^{m_2}w_j\sigma_2\paren{\bv_j^{\top}\bx}$. By the proof of bounding $A_1$, we could choose the truncation radius as $R=(C\eta)^p$ such that $\abs{f(\bx)}\le R$ for all $\bx \in \cD$ with high probability ($1-2ne^{-2\eta}$) on the dataset $\cD$. Now we denote a truncated version of $Y$ by 
$$\widetilde{Y}(\bx)=f^{\star}(\bx)\mathbf{1}\set{f^{\star}(\bx)\le R}K^{(0)}(\bx,\bx')\frac{1}{m_2}\sum_{j=1}^{m_2}w_j\sigma_2\paren{\bv_j^{\top}\bx} \mathbf{1}\left\{\frac{1}{m_2}\sum_{j=1}^{m_2}w_j\sigma_2\paren{\bv_j^{\top}\bx}\le \frac{\iota C_3}{\sqrt{m_2}d}\right\}.$$
Here $C_3$ is a constant defined in \eqref{equ:: bound A1 1}. Now, we decompose the concentration error as
\begin{align*}
    \frac{1}{n}\sum_{i=1}^{n} Y(\bx_i)-\EE_{\bx}\brac{Y(\bx)}&=\underbrace{\frac{1}{n}\sum_{i=1}^{n} \paren{Y(x_i)-\widetilde Y(x_i)}}_{\cL_0} +\underbrace{\frac{1}{n}\sum_{i=1}^{n} \paren{\widetilde Y(x_i)- \EE_{x_i}\brac{\widetilde Y(x_i)}}}_{\cL_1}\\
    &\quad+\underbrace{\frac{1}{n}\sum_{i=1}^{n} \paren{ \EE_{x_i}\brac{\widetilde Y(x_i)}-\EE_{x_i}\brac{ Y(x_i)}}}_{\cL_2}.
\end{align*}
We know with probability at least $1- 2ne^{-2\eta}$ on $\cD$, $\cL_0=0$.
\paragraph{Bounding $\cL_1$.} We attempt to use Bernstein's type bound. First we derive a uniform upper bound of $\widetilde Y(\bx)$. By the definition, we have 
\begin{align*}
    \abs{\widetilde Y(\bx)}&\le R\abs{K^{(0)}(\bx,\bx')}\abs{\frac{1}{m_2}\sum_{j=1}^{m_2}w_j\sigma_2\paren{\bv_j^{\top}\bx} \vee \frac{\iota C_3}{\sqrt{m_2}d}}\\&\le R\cdot \frac{C_2}{d^2}\frac{\iota C_3}{\sqrt{m_2}d}\\&=\frac{RC_2C_3\iota}{d^3\sqrt{m_2}}.
\end{align*}
Then, we bound the second moments of $\widetilde Y(\bx)- \EE_{\bx}\brac{\widetilde Y(\bx)}$, which is 
\begin{align*}
    \Var\brac{\widetilde Y(\bx)}&\le \EE_{\bx}\brac{\widetilde Y^2_k(\bx)}\\& \le r^2\normA\EE_{\bx}\brac{\paren{K^{(0)}(\bx,\bx')}^2\paren{\frac{1}{m_2}\sum_{j=1}^{m_2}w_j\sigma_2\paren{\bv_j^{\top}\bx} \vee \frac{\iota C_3}{\sqrt{m_2}d} }^2}\\
    &\le r^2\normA \EE_{\bx}\brac{\paren{K^{(0)}(\bx,\bx')}^2}\paren{\frac{\iota C_3}{\sqrt{m_2}d}}^2\\
    &\le r^2\normA \sum_{k=2}^{\infty}\frac{c_k^4}{B(d,k)^3} \cdot \paren{\frac{\iota C_3}{\sqrt{m_2}d}}^2\\
    &\le \frac{C_4r^2\normA\iota^2}{m_2d^8}.
\end{align*}
Here $C_4$ is a constant.
Thus, by Bernstein's inequality, we have
\begin{align*}
    \Pr\brac{\abs{\cL_{1}}\geq \frac{R\iota}{d^4 \sqrt{m_2}}\sqrt{\frac{t}{n}}} \le 
    2\exp\paren{\frac{-\frac{t}{2}}{C_4 +\frac{C_2C_3}{3}\sqrt{\frac{d^2t}{n}}}}.
\end{align*}
Thus, when $n \geq C_{\epsilon}\iota^2d^2$, by taking $t=\iota^2$ and $R=(C\eta)^p\le \iota^p$, with high probability on $\cD$, $\bV$ and $\bw$, we have
\begin{align*}
    \abs{\cL_{1}}\le \frac{\iota^{p+2}}{d^4 \sqrt{m_2n}}.
\end{align*}

\paragraph{Bounding $\cL_2$.} It suffices to bound
\begin{align*}
    &\quad\abs{\EE_{\bx}\brac{\widetilde Y_k(\bx)}-\EE_{\bx}\brac{ Y_k(\bx)}}\\
    &\le \EE_{\bx}\brac{\abs{f^{\star}(\bx)} \abs{K^{(0)}(x,x^{\prime})\frac{1}{m_2}\sum_{j=1}^{m_2}w_j\sigma_2\paren{\bv_j^{\top}\bx}}\mathbf{1}\left\{f^{\star}(\bx)>R ~\text{ or }~\frac{1}{m_2}\sum_{j=1}^{m_2}w_j\sigma_2\paren{\bv_j^{\top}\bx}> \frac{\iota C_3}{\sqrt{m_2}d}\right\} }\\
    &\le \EE_{\bx}\brac{(f^{\star}(\bx))^2}^{\frac{1}{2}}\Pr\brac{f^{\star}(\bx)>R ~\text{ or }~\frac{1}{m_2}\sum_{j=1}^{m_2}w_j\sigma_2\paren{\bv_j^{\top}\bx}> \frac{\iota C_3}{\sqrt{m_2}d}}^{\frac{1}{4}}\EE_{\bx}\brac{K^{(0)}(\bx,\bx')^4}^{\frac{1}{4}}\frac{\tau C_3}{\sqrt{m_2}d}\\
    &\lesssim \frac{ (\exp(-\eta)+\exp(-\iota))\iota C_3}{d^3\sqrt{m_2}}.
\end{align*}
Taking $\eta \geq 2\log n+2\log d +\log(C_3) $ and $\iota \geq C\eta$, we ensure that $\cL_2 \le {\iota}/{(d^4\sqrt{m_2n})}$. Altogether, with high probability (event $E_{3}$) on $\cD$, we have
\begin{align*}
    \abs{A_2}= \abs{\frac{1}{n}\sum_{i=1}^{n} Y(\bx_i)-\EE_{\bx}\brac{Y(\bx)}} \le \frac{2\iota^{p+2}}{d^4\sqrt{m_2n}}.
\end{align*}
The proof is complete.
\end{proof}

\begin{proof}[Proof of Lemma \ref{lemma:: bound A3}]
    We remember that 
\begin{align*}
   A_3=\frac{1}{m_2}\sum_{j=1}^{m_2}w_j\EE_{\bx}\brac{f^{\star}(\bx)K^{(0)}(\bx,\bx')\sigma_2\paren{\bv_j^{\top}\bx}}=\frac{1}{m_2}\sum_{j=1}^{m_2}w_j h(\bv_j,\bx'),
\end{align*}
where $h(\bv,\bx')=\EE_{\bx}\brac{f^{\star}(\bx)K^{(0)}(\bx,\bx')\sigma_2\paren{\bv^{\top}\bx}}$.
To bound $h(\bv,\bx')$ uniformly, we have the following lemma:
\begin{lemma}\label{lemma::bound h}
    With high probability on $\bV$ and the datasets $\cD_1$ and $\cD_2$, we have for any $j\in[m_2]$ and $\bx'\in \cD_2$, 
\begin{align*}
    \abs{h(\bv_j,\bx')}\lesssim \frac{\iota^2\log^2(m_2n_2)}{d^6} \cdot\paren{\norm{\cP_{>2}(f)}{L^2}+\sqrt{d}\norm{\cP_{2}(f)}{L^2}}
\end{align*}
\end{lemma}
The proof of Lemma \ref{lemma::bound h} is deferred to the end of this section. Thus, condition on the event above, by invoking the upper bound of Gaussian tail and uniformly bounding over $\bx'\in \cD_2$, we have with probability $1-2n\exp(-\iota^2/2)$ on $\bw$, for any $\bx' \in \cD_2$, we have 
\begin{align*}
    \abs{A_3}\le \frac{\iota}{\sqrt{m_2}} \sqrt{\frac{1}{m_2}\sum_{j=1}^{m_2}h^2(\bv_j,\bx')} \lesssim \frac{\iota^3 \log^2(m_2n_2)}{\sqrt{m_2}d^6}\cdot\paren{\norm{\cP_{>2}(f)}{L^2}+\sqrt{d}\norm{\cP_{2}(f)}{L^2}}.
\end{align*}
Also, for a single point $\bx'$, with probability $1-2\exp(-\iota^2/2)$ on $\bw$, we have 
\begin{align*}
    \abs{A_3} \lesssim \frac{\iota^3 \log^2(m_2n_2)}{\sqrt{m_2}d^6}\cdot\paren{\norm{\cP_{>2}(f)}{L^2}+\sqrt{d}\norm{\cP_{2}(f)}{L^2}}.
\end{align*}
The proof is complete.
\end{proof}

\begin{proof}[Proof of Lemma \ref{lemma::bound h}]
Recall that the activation function $\sigma_2$ admits a Gegenbauer expansion
\begin{align*}
    \sigma_2(t)=\sum_{i=2}^{\infty}c_i Q_i(t).
\end{align*}  
Let's fix $\bx'$ and $\bv$. Note that we can decompose $h(\bv,\bx')$ as
\begin{align}
{h(\bv,\bx')}&={\EE_{\bx}\brac{f^{\star}(\bx)K^{(0)}(\bx,\bx')\sigma_2\paren{\bv^{\top}\bx}}}\nonumber\\
    &={\EE_{\bx}\brac{f^{\star}(\bx)\sum_{i=2}^{\infty}\frac{c^2_iQ_i(\bx^{\top}\bx')}{B(d,i)}\sum_{j=2}^{\infty}c_jQ_j\paren{\bx^{\top}\bv}}}\nonumber\\
    &=\sum_{i=2}^{\infty}\sum_{j=2}^{\infty}{\EE_{\bx}\brac{f^{\star}(\bx)\frac{c^2_i}{B(d,i)^2}\left\langle \bY_i(\bx),\bY_i(\bx')\right\rangle\cdot\frac{c_j}{B(d,j)}\left\langle \bY_j(\bx),\bY_j(\bv)\right\rangle}}\nonumber\\
    &=\sum_{i=2}^{\infty}\sum_{j=2}^{\infty}\frac{c^2_ic_j}{B(d,i)^2B(d,j)} {\left\langle { \bY_i(\bx')\otimes\bY_j(\bv),\EE_{\bx}\brac{f^{\star}(\bx) \bY_i(\bx)\otimes\bY_j(\bx)}}\right\rangle}. \nonumber\\
    &=:\sum_{i=2}^{\infty}\sum_{j=2}^{\infty}\frac{c^2_ic_j}{B(d,i)^2B(d,j)} h_{i,j}(\bv,\bx').\nonumber
\end{align}
By the definition of $h_{i,j}(\bv,\bx')$, we have
\begin{align}
&\quad\EE_{\bv,\bx'}\brac{h_{i,j}^2(\bv,\bx')}\\&=\EE_{\bv,\bx'}\brac{\left\langle { \bY_i(\bx')\otimes\bY_j(\bv),\EE_{\bx}\brac{f^{\star}(\bx) \bY_i(\bx)\otimes\bY_j(\bx)}}\right\rangle^2} \nonumber\\
&={\left\langle {\EE_{\bx'}\brac{f^{\star}(\bx') \bY_i(\bx')\otimes\bY_j(\bx')},\EE_{\bx}\brac{f^{\star}(\bx) \bY_i(\bx)\otimes\bY_j(\bx)}}\right\rangle} \nonumber\\
&=B(d,i)B(d,j)\EE_{\bx,\bx'}\brac{f(\bx)f(\bx')Q_i(\bx^{\top}\bx')Q_j(\bx^{\top}\bx')}\nonumber\\
&=B(d,i)B(d,j)\EE_{\bx,\bx'}\brac{f(\bx)f(\bx')\sum_{k=0}^{\min(i,j)}{b^{(i,j)}_{i+j-2k}}Q_{i+j-2k}(\bx^{\top}\bx')}\nonumber\\
&=B(d,i)B(d,j)\EE_{\bx,\bx'}\brac{f(\bx)f(\bx')\sum_{k=0}^{\min(i,j)}\frac{b^{(i,j)}_{i+j-2k}}{B(d,i+j-2k)}\left\langle\bY_{i+j-2k}(\bx),\bY_{i+j-2k}(\bx')\right\rangle}\nonumber\\
&=B(d,i)B(d,j)\sum_{k=0}^{\min(i,j)}\frac{b^{(i,j)}_{i+j-2k}}{B(d,i+j-2k)}\norm{\EE_{\bx}\brac{f(\bx)\bY_{i+j-2k}(\bx)}}{F}^2\nonumber\\
&=B(d,i)B(d,j)\sum_{k=0}^{\min(i,j)}\frac{b^{(i,j)}_{i+j-2k}}{B(d,i+j-2k)}\norm{\brac{\cP_{i+j-2k}(f)}}{L^2}^2. \label{equ::second moment hij}
\end{align}

Since $h_{i,j}(\bv,\bx')$ is a degree 
$i$ polynomial of $\bx'$ and a degree $j$ polynomial of $\bv$, by Lemma \ref{lemma: sphere hypercontractivity multi}, we have for any $q\geq 2$,
\begin{align*}
    \mathbb{E}_{\bv,\bx'}\left[\abs{h_{i,j}(\bv,\bx')}^q\right]^{2/q} 
    &\le(q-1)^{i+j}\mathbb{E}_{\bv,\bx'}\left[h_{i,j}(\bv,\bx')^2\right]
\end{align*}
Let $\delta=\paren{2e\iota\log \paren{m_2n_2}}^{(i+j)/2}$ for some $\iota>1$, taking $q=1+e^{-1}\delta^{2/(i+j)}$ and Markov inequality, we have
\begin{align*}
    \Pr\brac{\abs{h_{i,j}(\bv,\bx')}\geq \delta \sqrt{\mathbb{E}_{\bv,\bx'}\left[h_{i,j}(\bv,\bx')^2\right] }} &\le \frac{\mathbb{E}_{\bv,\bx'}\left[\abs{h_{i,j}(\bv,\bx')}^q\right]}{\paren{\delta \sqrt{\mathbb{E}_{\bv,\bx'}\left[h_{i,j}(\bv,\bx')^2\right] }}^{\infty}}\\
    &\le\paren{q-1}^{(i+j)q/2}\delta^{-q}\\
    &= \exp\paren{-\frac{i+j}{2}\paren{1+\frac{2e\iota\log \paren{m_2n_2}}{e}} }\\
    &=(m_2n_2)^{-\iota(i+j)}\exp(-(i+j)/2).
\end{align*}
Thus, with probability at least $-(m_2n_2)^{1-\iota(i+j)}\exp(-(i+j)/2)$, 
\begin{align*}
    h_{i,j}(\bv_i,\bx') &\le \paren{2e\iota\log \paren{m_2n_2}}^{(i+j)/2}\sqrt{\mathbb{E}_{\bv,\bx'}\left[h_{i,j}(\bv,\bx')^2\right] }\\
    &\le \paren{2e\iota\log \paren{m_2n_2}}^{(i+j)/2}\sqrt{\sum_{k=0}^{[(i+j)/2]}\frac{B(d,i)B(d,j)b^{(i,j)}_{i+j-2k}}{B(d,i+j-2k)}\norm{\brac{\cP_{i+j-2k}(f)}}{L^2}^2}.
\end{align*}
In the second inequality we invoke \eqref{equ::second moment hij}. Summing over $i$ and $j$ gives rise to
\begin{align*}   &\quad \abs{h(\bv,\bx')}\\&=\abs{\sum_{i=2}^{\infty}\sum_{j=2}^{\infty}\frac{c^2_i{c_j}}{B(d,i)^2B(d,j)} h_{i,j}(\bv,\bx')}\\
&\le \sum_{i=2}^{\infty}\sum_{j=2}^{\infty}\frac{c^2_i\abs{c_j}\paren{2e\iota\log \paren{m_2n_2}}^{(i+j)/2}}{B(d,i)^2B(d,j)} \sqrt{\sum_{k=0}^{[(i+j)/2]}\frac{B(d,i)B(d,j)b^{(i,j)}_{i+j-2k}}{B(d,i+j-2k)}\norm{\brac{\cP_{i+j-2k}(f)}}{L^2}^2}\\
&\le {\sqrt{\sum_{i=2}^{\infty}\sum_{j=2}^{\infty}\frac{c^2_i\abs{c_j}\paren{2e\iota\log \paren{m_2n_2}}^{(i+j)/2}}{B(d,i)^2B(d,j)^{1/2}}}}\\&\quad\cdot{\sqrt{\sum_{i=2}^{\infty}\sum_{j=2}^{\infty}\frac{c^2_i\abs{c_j}\paren{2e\iota\log \paren{m_2n_2}}^{(i+j)/2}}{B(d,i)^2B(d,j)^{3/2}}\sum_{k=0}^{[(i+j)/2]}\frac{B(d,i)B(d,j)b^{(i,j)}_{i+j-2k}}{B(d,i+j-2k)}\norm{\brac{\cP_{i+j-2k}(f)}}{L^2}^2}}\\
&\lesssim\frac{\iota\log(m_2n_2)}{d^{5/2}}\cdot\sqrt{\sum_{\ell=0}^{\infty}\frac{1}{B(d,\ell)}\sum_{i+j-\ell ~\text{even}}^{2 \le i,j}\frac{c_i^2\abs{c_j}b^{(i,j)}_{\ell}\paren{2e\iota\log \paren{m_2n_2}}^{(i+j)/2}}{B(d,i)B(d,j)^{1/2}}\norm{\brac{\cP_{\ell}(f)}}{L^2}^2}.
\end{align*}
In the second inequality we invoke Cauchy inequality. Then by plugging the bound on $b^{(i,j)}_\ell$ in Lemma \ref{lemma::coeff bound linearization}, we have
\begin{align*}
    &\quad\abs{h(\bv,\bx')}\\&\lesssim \frac{\iota\log(m_2n_2)}{d^{5/2}}\cdot\sqrt{\sum_{\ell=0}^{\infty}\frac{4(2\ell+d-2)}{B(d,\ell)(d-2)}\sum_{i+j-\ell =2k}^{ i,j\geq \max(k,2)}\frac{c_i^2\abs{c_j}\paren{2e\iota\log \paren{m_2n_2}}^{(i+j)/2}}{B(d,i)B(d,j)^{1/2}(d-2)_k}\binom{i}{k}\binom{j}{k}k !\norm{\brac{\cP_{\ell}(f)}}{L^2}^2}\\
&\lesssim \frac{\iota\log(m_2n_2)}{d^{5/2}}\cdot \left(\frac{\iota\log(m_2n_2)}{d^{7/2}}\cdot \norm{\cP_{>2}(f)}{L^2} + \frac{\iota \log(m_2n_2)}{d^3}\cdot \norm{\cP_{2}(f)}{L^2} \right)\\
&=\frac{\iota^2\log^2(m_2n_2)}{d^6} \cdot \norm{\cP_{>2}(f)}{L^2}  + \frac{\iota^2\log^2(m_2n_2)}{d^{11/2}} \cdot \norm{\cP_{2}(f)}{L^2}.
\end{align*}

The probability of this event is at least 
\begin{align*}
    1-\sum_{i=2}^{\infty}\sum_{j=2}^{\infty}(m_2n_2)^{-\iota(i+j)}\exp((i+j)/2)=1-\frac{m_2n_2\paren{-(m_2n_2)^{-\iota}e^{-1/2}}^2}{(m_2n_2)^{4\iota}e^2},
\end{align*}
which is a high probability event when uniformly bounding over $\bx' \in \cD_2$ and $\bv=\bv_1,\bv_2\dots,\bv_{m_2}$. The proof is complete.
\end{proof}

\subsection{Proof of Proposition \ref{prop::reconstructed feature main}}\label{appendix proof reconstruct whole}
\subsubsection{The Formal Statement of Proposition \ref{prop::reconstructed feature main} and the Corollary}\label{sec::reconstruction appendix}
Let's consider a formal version of Proposition \ref{prop::reconstructed feature main}. We remind the readers that throughout this section we denote $n=n_1$ for notation simplicity, since we only focus on the first training stage.
\begin{proposition}[Reconstruct the feature]\label{prop::reconstructed feature}
   Suppose $m_2,n\geq Cd^{4}$ for some sufficiently large $C$. With high probability jointly on $\bV$ and the training datasets $\cD_1$ and $\cD_2$, there exists a matrix $\bB^{\star}\in \RR^{r\times m_2}$ satisfying $\norm{\bB^{\star}}{\rm op}\lesssim \frac{d^6}{\lambda_{\min}(\bH)}\sqrt{\frac{1}{m_2}}$ such that for any $\bx' \in \cD_2$, we have
   \begin{align*}
       \norm{\bB^{\star}\mathbf{h}^{(1)}(\bx')-\bp(\bx')}{\rm 2} \lesssim \frac{\sqrt{r}}{\lambda_{\min}(\bH)}\cdot \Bigg( \frac{\iota^{p+2}d^5}{m_2} +\frac{\iota d^3}{\sqrt{m_2}}+ \frac{\iota ^{p+3/2}d}{\sqrt{n}} +\frac{\iota Lr^2\normA \log^2 d}{d^{1/6}} \Bigg).
   \end{align*}
   Here $L\lesssim \iota r^{\frac{p-1}{2}}$ is a Lipschitz constant satisfying $\norm{\nabla g^{\star}(\bp(\bx))}{\rm 2} \le L$  with high probability.
\end{proposition}
With the proposition above, we directly have the following result.
\begin{corollary}\label{lem:bound_fn_diff}
Under the same assumption in Proposition \ref{prop::reconstructed feature}, with high probability, we have
\begin{align*}
\sup_{\bx \in \mathcal{D}_2}\abs{g(\bB^{\star}\bh^{(1)}(\bx)) - g(\bp(\bx))} \lesssim \norm{g}{L^2}\cdot \frac{r^{p/2}}{\lambda_{\min}(\bH)}\cdot \Bigg( \frac{\iota^{p+2}d^5}{m_2} +\frac{\iota d^3}{\sqrt{m_2}}+ \frac{\iota ^{p+3/2}d}{\sqrt{n}} +\frac{\iota Lr^2\normA \log^2 d}{d^{1/6}} \Bigg).
\end{align*}
\end{corollary}
We provide the main proof of Proposition \ref{prop::reconstructed feature} in Appendix \ref{appendix::main proof reconstruct}, and defer the proof of Corollary \ref{lem:bound_fn_diff} and other supporting lemmas to Appendix \ref{sec::reconstruct proof}.
\subsubsection{Proof of Proposition \ref{prop::reconstructed feature}}\label{appendix::main proof reconstruct}
\begin{proof}
    Denote the target features by $\bp(\bv)=[\bv^{\top}\bA_1\bv,\cdots,\bv^{\top}\bA_r\bv]^{\top}\in \RR^{r}$ for any $\bv \in \RR^{d}$, and we further let $\bP=[\bp(\bv_1),\bp(\bv_2),\cdots,\bp(\bv_{m_2})]^{\top}\in \RR^{m_2\times r}$. Then for any $\bx'\in\cD_2$, we have the following decomposition
 \begin{align*}
        \frac{1}{m_2}\bP^{\top}\mathbf{h}^{(1)}(\bx') 
&=\frac{1}{m_2n}\sum_{i=1}^{n}\sum_{j=1}^{m_2}f^{\star}(\bx_i)K^{(0)}_{m_2}(\bx_i,\bx')\sigma_2\paren{\bv_j^{\top}\bx_i}\bp(\bv_j)\\
&=\underbrace{\frac{1}{n}\sum_{i=1}^{n}\frac{1}{m_2}\sum_{j=1}^{m_2}f^{\star}(\bx_i)\paren{K^{(0)}_{m_2}(\bx_i,\bx')-K^{(0)}(\bx_i,\bx')}\sigma_2\paren{\bv_j^{\top}\bx_i}\bp(\bv_j)}_{\bD_{1,1}}\\
    &\quad+ \underbrace{\frac{1}{n}\sum_{i=1}^{n}f^{\star}(\bx_i){K^{(0)}(\bx_i,\bx')}\paren{\frac{1}{m_2}{\sum_{j=1}^{m_2}\sigma_2\paren{\bv_j^{\top}\bx_i}\bp(\bv_j)}-\frac{c_2}{B(d,2)}\bp(\bx_i
    )}}_{\bD_{1,2}}\\
    &\quad +\underbrace{\frac{c_2}{nB(d,2)}\paren{\sum_{i=1}^{n}f^{\star}(\bx_i)K^{(0)}(\bx_i,\bx')\bp(\bx_i)-\EE_{\bx}\brac{f^{\star}(\bx)K^{(0)}(\bx,\bx')\bp(\bx_i)}}}_{\bD_2}\\
        &\quad +\underbrace{\frac{c_2}{B(d,2)}\EE_{\bx}\brac{f^{\star}(\bx)K^{(0)}(\bx,\bx')\bp(\bx)}}_{\bD_3}.
    \end{align*}

We will derive an upper bound on the concentration error terms $\bD_{1,1}$, $\bD_{1,2}$ and $\bD_2$, respectively. Moreover, leveraging the asymptotic analysis in Appendix \ref{appendix::asymptotic analysis inner feature reconstruct}, we  show that $\bD_3\approx d^{-6}\bH \bp(\bx')$ with high probability, 
\begin{lemma}[Bound $\bD_{1,1}$ and $\bD_{1,2}$]\label{lemma:: bound A1 recons}
    Under the same assumptions in Proposition \ref{prop::reconstructed feature}, with high probability on $\bV$, $\cD_1$ and $\cD_2$, we have 
    \begin{align*}
    &\norm{\bD_{1,1}}{\infty} \le\frac{9C_4^{1/4}\iota^{p+2}}{m_2d}~~~~\text{and}~~~~ \norm{\bD_{1,2}}{\infty} \le \frac{9\iota C_4^{1/4}C_2}{\sqrt{m_2}d^3}.
\end{align*}
\end{lemma}

\begin{lemma}[Bound $\bD_{2}$]\label{lemma:: bound D2 recons}
    Under the same assumptions in Proposition \ref{prop::reconstructed feature}, with high probability on $\cD_1$ and $\cD_2$, we have 
    \begin{align*}
    \norm{\bD_2}{\infty}\lesssim \frac{\iota^{p+3/2}}{\sqrt{n}d^5}.
\end{align*}
\end{lemma}

\begin{lemma}[Compute $\bD_3$]\label{lemma:: bound D3 recons}
    Under the same assumptions in Proposition \ref{prop::reconstructed feature}, with high probability on $\cD_2$, for any $\bx'\in\cD_2$, we have 
  \begin{align*}
    \norm{\bD_3-\frac{c_2^2}{B(d,2)^2 d(d-1)}\cdot \bH \bp(\bx')}{\infty}
   \lesssim \frac{\iota L r^2\normA\log^2 d}{d^{6+1/6}} .
\end{align*}
\end{lemma}
We defer the detailed proof of the three lemmas to Appendix  \ref{sec::reconstruct proof}.
Combining all the results above and choosing  $$\bB^{\star}=\frac{B(d,2)^2 d(d-1)}{c_2^2}\cdot \frac{1}{m_2}\bH^{-1}\bP^{\top},$$ we have with high probability on $\bV$, $\cD_1$ and $\cD_2$,
\begin{align*}
    &\quad \norm{\bB^{\star}\mathbf{h}^{(1)}(\bx')-\bp(\bx')}{2} \\&\le \frac{B(d,2)^2 d(d-1)}{c_2^2}\cdot\norm{\bH^{-1}\left(\bD_{1,1}+\bD_{1,2}+\bD_{2}+\bD_{3}-\frac{c_2^2}{B(d,2)d(d-1)}\cdot \bH \bp(\bx')\right)}{2}\\
    &\lesssim \frac{d^6\sqrt{r}}{\lambda_{\min}(\bH)}\cdot \Bigg(\norm{\bD_{1,1}}{\infty}+\norm{\bD_{1,2}}{\infty}+\norm{\bD_{2}}{\infty}
    \\&\quad\hspace{6em}+\norm{\bD_{3}-\frac{c_2^2}{B(d,2)d(d-1)}\cdot \bH \bp(\bx')}{\infty}\Bigg)\\
    &\lesssim \frac{\sqrt{r}}{\lambda_{\min}(\bH)}\cdot \Bigg( \frac{\iota^{p+2}d^5}{m_2} +\frac{\iota d^3}{\sqrt{m_2}}+ \frac{\iota ^{p+3/2}d}{\sqrt{n}} +\frac{\iota Lr^2\normA \log^2 d}{d^{1/6}} \Bigg),
\end{align*}
To bound $\norm{\bB^{\star}}{\rm op}$, note that 
\begin{align*}
    \norm{\bB^{\star}}{\rm op}^2&=\norm{\bB^{\star}\bB^{\star \top }}{\rm op}\\
    &\lesssim \frac{d^{12}}{m_2^2 \lambda_{\min}^2(\bH)} \norm{\bP\bP^{\top}}{\rm op}\\
    &=\frac{d^{12}}{m_2 \lambda_{\min}^2(\bH)} \norm{\frac{1}{m_2}\sum_{j=1}^{m_2}\bp(\bv_j)\bp(\bv_j)^{\top}}{\rm op}.
\end{align*}
Moreover, for any $j\in[m_2]$, we have 
\begin{align*}
    \norm{\bp(\bv_j)\bp(\bv_j)^{\top}}{\rm op}=\norm{\bp(\bv_j)}{2}^2 &=\sum_{k=1}^{r}(\bv_j^{\top}\bA_k\bv_j)^2 \lesssim rd^2,
\end{align*} 
and we have 
\begin{align*}    \norm{\EE_{\bv}\brac{\paren{\bp(\bv)\bp(\bv)^{\top}}^2}}{\rm op}&=\norm{\EE_{\bv}\brac{ \sum_{k=1}^{r}(\bv^{\top}\bA_k\bv)^2{\bp(\bv)\bp(\bv)^{\top}}}}{\rm op}\\
    &\le\sum_{k=1}^{r}\norm{\EE_{\bv}\brac{ (\bv^{\top}\bA_k\bv)^2{\bp(\bv)\bp(\bv)^{\top}}}}{\rm op}\\
    &\le \sum_{k=1}^{r}d^2\norm{\EE_{\bv}\brac{ {\bp(\bv)\bp(\bv)^{\top}}}}{\rm op}\\
    &=rd^2.
\end{align*}
The second inequality holds because $\bp(\bv)\bp(\bv)^{\top}$ is positive semi-definite. By Matrix Bernstein Inequality, we have 
\begin{align*}
    \Pr\brac{ \norm{\frac{1}{m_2}\sum_{j=1}^{m_2}\bp(\bv_j)\bp(\bv_j)^{\top} -\bI}{\rm op} \geq 1 + \frac{\sqrt{r}d\iota}{\sqrt{m_2}}}&\le \exp\paren{-\frac{\frac{rd^2\iota^2}{2m_2}}{\frac{rd^2}{m_2}+\frac{rd^2}{3m_2}\cdot \frac{\sqrt{r}d\iota}{\sqrt{m_2}}}}\\
    &=\exp\paren{-\frac{\frac{\iota^2}{2}}{1+\frac{\sqrt{r}d\iota}{3\sqrt{m_2}}}}.
\end{align*}
Thus, when $m_2 \geq d^4$, we know with high probability on $\bV$,
\begin{align*}
   \norm{\frac{1}{m_2}\sum_{j=1}^{m_2}\bp(\bv_j)\bp(\bv_j)^{\top}}{\rm op}\le 1+\frac{\sqrt{r}d\iota}{\sqrt{m_2}} \lesssim 1.
\end{align*}
Thus, we have $\norm{\bB^{\star}}{\rm op}\lesssim \frac{d^6}{\lambda_{\min}(\bH)}\sqrt{\frac{1}{m_2}}$. 
The proof is complete.
\end{proof}
\subsubsection{Omitted Proofs in Appendices \ref{sec::reconstruction appendix} and \ref{appendix::main proof reconstruct}}\label{sec::reconstruct proof}
\begin{proof}[Proof of Lemma \ref{lemma:: bound A1 recons}]
Let's first bound $\bD_{1,1}$. We can rewrite $\bD_{1,1}$ as
\begin{align*}
    {\bD_{1,1}} =\frac{1}{n}\sum_{i=1}^{n}f^{\star}(\bx_i)\paren{K^{(0)}_{m_2}(\bx_i,\bx')-K^{(0)}(\bx_i,\bx')}\frac{1}{m_2}\sum_{j=1}^{m_2}\sigma_2\paren{\bv_j^{\top}\bx_i}\bp(\bv_j)
\end{align*}
By Lemma \ref{lemma:: concentration vtAvsigma2}, for any $k\in[r]$, we have with high probability on $\bV$
\begin{align*}
    \abs{\frac{1}{m_2}\sum_{j=1}^{m_2}(\bv_j^{\top}\bA_k\bv_j)\sigma_2(\bv^{\top}_j\bx_i)-\frac{c_2}{B(d,2)}\bx_i^{\top}\bA_k\bx_i}\le \frac{9\iota d^{-1}C_4^{1/4}}{\sqrt{m_2}}.
\end{align*}
Thus, by enumerating $\bA_k$ over $\{\bA_1,\bA_2,\dots,\bA_r\}$, we have with high probability on $\bV$,
\begin{align*}
    \norm{\frac{1}{m_2}\sum_{j=1}^{m_2}\bp(\bv_j)\sigma_2(\bv^{\top}_j\bx_i)-\frac{c_2}{B(d,2)}\bp(\bx_i)}{\infty}\le \frac{9\iota d^{-1}C_4^{1/4}}{\sqrt{m_2}}.
\end{align*}

On the other hand, by Lemma \ref{lemma::concentration initial kernel}, with high probability on $\bV$, we have for any $\bx_i \in\cD_1,\bx'\in \cD_2$,
\begin{align}
   \abs{ K^{(0)}_{m_2}(\bx_i,\bx')-K^{(0)}(\bx_i,\bx')} \le \frac{\iota}{\sqrt{m_2}d^2}. \label{equ:: bound A1 2}
\end{align}
Moreover, under the event $E_3$ (defined in \eqref{event::E3}),  with high probability on the dataset $\cD_1$, we have $\abs{f(\bx)}\le \iota^p$ for any $\bx \in \cD_1$. Thus, altogther we have 
\begin{align*}
    \norm{\bD_{1,1}}{\infty} \le \iota^p \cdot \frac{\iota}{\sqrt{m_2}d^2} \cdot \frac{9\iota d^{-1}C_4^{1/4}}{\sqrt{m_2}}=\frac{9C_4^{1/4}\iota^{p+2}}{m_2d}.
\end{align*}
with high probability. To bound $\bD_{1,2}$, from the proof above, we know with high probability,
\begin{align*}
    \norm{\frac{1}{m_2}\sum_{j=1}^{m_2}\bp(\bv_j)\sigma_2(\bv^{\top}_j\bx_i)-\frac{c_2}{B(d,2)}\bp(\bx_i)}{\infty}\le \frac{9\iota d^{-1}C_4^{1/4}}{\sqrt{m_2}}.
\end{align*}
Moreover, for any $\bx\in \cD_1$ and $ \bx' \in\cD_2$,
\begin{align}\label{equ:: bound Kernel}
    K^{(0)}(\bx,\bx')=\EE_{\bv}\brac{\sigma_2(\bv^{\top}\bx)\sigma_2(\bv^{\top}\bx')}\le \sqrt{\EE_{\bv}\brac{\sigma_2(\bv^{\top}\bx)^2}\EE_{\bv}\brac{\sigma_2(\bv^{\top}\bx')^2}}\le \frac{C_2}{d^2}.
\end{align}
Thus, we can bound $\bD_{1,2}$ with high probability by 
\begin{align*}
    \norm{\bD_{1,2}}{\infty}\le \frac{C_2}{d^2} \cdot \frac{9\iota d^{-1}C_4^{1/4}}{\sqrt{m_2}}=\frac{9\iota C_4^{1/4}C_2}{\sqrt{m_2}d^3}.
\end{align*}

The proof is complete.
\end{proof}

\begin{proof}[Proof of Lemma \ref{lemma:: bound D2 recons}]

Thus, let's focus on the concentration of a single element
\begin{align*}
    Y_k(\bx)=f^{\star}(\bx)K^{(0)}(\bx,\bx')\bx^{\top}\bA_{k}\bx,~~k=1,2,\dots,r.
\end{align*}
Similar to the proof of Lemma \ref{lemma:: bound A1}, we denote a truncated version of $Y_k$ by $$\widetilde{Y}_k(\bx)=f^{\star}(\bx)\mathbf{1}\set{f^{\star}(\bx)\le R}K^{(0)}(\bx,\bx')\bx^{\top}\bA_{k}\bx,~~k=1,2,\dots,r.$$
Here, $R=(C\eta)^p$ for some large constant $C$. Now, we decompose the concentration error as
\begin{align*}
    \frac{1}{n}\sum_{i=1}^{n} Y_k(\bx_i)-\EE_{\bx}\brac{Y_k(\bx)}&=\underbrace{\frac{1}{n}\sum_{i=1}^{n} \paren{Y_k(\bx_i)-\widetilde Y_k(\bx_i)}}_{\cL_0} +\underbrace{\frac{1}{n}\sum_{i=1}^{n} \paren{\widetilde Y_k(\bx_i)- \EE_{x_i}\brac{\widetilde Y_k(\bx_i)}}}_{\cL_1}\\
    &\quad+\underbrace{\frac{1}{n}\sum_{i=1}^{n} \paren{ \EE_{x_i}\brac{\widetilde Y_k(\bx_i)}-\EE_{x_i}\brac{ Y_k(\bx_i)}}}_{\cL_2}.
\end{align*}
By \eqref{event::E3}, we know with probability at least $1- 2ne^{-2\eta}$, $\cL_0=0$.

\paragraph{Bounding $\cL_1$.}  First we derive a uniform upper bound of $\widetilde Y_k(\bx)$, which is 
\begin{align*}
    \abs{Y_k(\bx)}&\le R\abs{K^{(0)}(\bx,\bx')\bx^{\top}\bA_{k}\bx}\\&\le R\abs{K^{(0)}(\bx,\bx')}\abs{\bx^{\top}\bA_k\bx}\\&\le R \cdot \frac{C_2}{d^2}d\norm{\bA_k}{op}\\
    &=\frac{RC_2\norm{\bA_k}{\rm op}}{d}.
\end{align*}
Then, we bound the second moments of $\widetilde Y_k(\bx)- \EE_{\bx}\brac{\widetilde Y_k(\bx)}$. Again by Lemma \ref{lemma::poly concentration sphere}, we know that there exists a sufficient large constant $C>0$ s.t. $\Pr\brac{\abs{\bx^{\top}\bA_k\bx}\geq C\iota }\le 2\exp(-\iota)$. By taking $\iota \geq 2\log d$, we have

\begin{align*}
    \Var\brac{\widetilde Y_k(\bx)}&\le \EE_{\bx}\brac{\widetilde Y^2_k(\bx)}\\
    &= \EE_{\bx}\brac{\widetilde Y^2_k(\bx)\mathbf{1}\left\{\abs{\bx^{\top}\bA_k\bx}\le C\iota\right\}}+\EE_{\bx}\brac{\widetilde Y^2_k(\bx)\mathbf{1}\left\{\abs{\bx^{\top}\bA_k\bx}>C\iota\right\}}\\
    & \le C^2r^2\normA\iota^2\EE_{\bx}\brac{\paren{K^{(0)}(\bx,\bx')}^2}+\frac{r^2\normA C_2^2}{d^4}\cdot \EE_{\bx}\brac{\mathbf{1}\left\{\abs{\bx^{\top}\bA_k\bx}>C\iota\right\}}\\
    &\le C^2r^2\normA\iota^2\cdot \sum_{i=2}^{\infty}\frac{c_i^4}{B(d,i)^3}+\frac{2r^2\normA C_2^2\exp(-\iota)}{d^4}\\
    &\lesssim   \frac{C'r^2\normA\iota^2}{d^{6}}.
\end{align*}
Here $C'$ is a sufficiently large constant independent of $d$. We invoke \eqref{equ:: bound Kernel} in the second inequality. Thus, by Bernstein's inequality, we have
\begin{align*}
    \Pr\brac{\abs{\cL_{1}}\geq \frac{R\iota}{d^3}\sqrt{\frac{C'\iota}{n}}} &\le 2\exp\paren{\frac{-\frac{\iota^3 C'r^2\normA}{2nd^6}}{\frac{C'r^2\normA\iota^2}{nd^6} +\frac{RC_2\norm{\bA_k}{\rm op}}{3nd}\sqrt{\frac{r^2\normA\iota}{nd^6}}}}\\
    &=2\exp\paren{\frac{-\frac{\iota}{2}}{1 +\frac{C_2\norm{\bA_k}{\rm op}}{3C'}\sqrt{\frac{d^4}{n\iota^3}}}}.
\end{align*}
Thus, when $n \geq C_2^2\norm{\bA_k}{\rm op}^2d^4$ , we have with high probability on the training dataset $\cD_1$,
\begin{align*}
    \abs{\cL_{1}}\le \frac{R\iota}{d^3}\sqrt{\frac{C'\iota}{n}}\lesssim \frac{R\iota^{3/2}}{\sqrt{n}d^3}.
\end{align*}

\paragraph{Bounding $\cL_2$.} It suffices to bound
\begin{align*}
    \abs{\EE_{\bx}\brac{\widetilde Y_k(\bx)}-\EE_{\bx}\brac{ Y_k(\bx)}}&\le \EE_{\bx}\brac{\abs{f^{\star}(\bx)}\mathbf{1}\set{f^{\star}(\bx)>R} \abs{K^{(0)}(x,x^{\prime})\bx^{\top}\bA_k\bx} }\\
    &\le \EE_{\bx}\brac{(f^{\star}(\bx))^2}^{\frac{1}{2}}\Pr\brac{f^{\star}(\bx)>R}^{\frac{1}{4}}\EE_{\bx}\brac{K^{(0)}(\bx,\bx')^8}^{\frac{1}{8}}\EE_{\bx}\brac{(\bx^{\top}\bA_k\bx)^8}^{\frac{1}{8}}\\
    &\le 1 \cdot \exp(-\eta/2)\cdot \frac{C_2}{d^2}\cdot (8-1)\\
    &= \frac{7C_2\exp(-\eta/2)}{d^2}.
\end{align*}
Here we invoke \eqref{equ:: bound Kernel} and Lemma \ref{lemma::poly concentration sphere} in the last inequality.
By taking $\eta = \iota \geq 2\log n + 8\log d$,  we can ensure that with high probability, we have
\begin{align*}
    \abs{\frac{1}{n}\sum_{i=1}^{n} Y_k(\bx_i)-\EE_{\bx}\brac{Y_k(\bx)}} \le \abs{\cL_1}+\abs{\cL_2} \lesssim \frac{R\iota^{3/2}}{\sqrt{n}d^3}.
\end{align*}
Thus, by taking $k$ over $[r]$, we have with high probability over the training set $\cD_1$, we have
\begin{align*}
    \norm{\bD_2}{\infty}\le \frac{\abs{c_2}}{B(d,2)}\cdot \frac{R\iota^{3/2}}{\sqrt{n}d^3} \lesssim \frac{R\iota^{3/2}}{\sqrt{n}d^5} \lesssim\frac{\iota^{p+3/2}}{\sqrt{n}d^5}.
\end{align*}
The proof is complete.

\end{proof}

\begin{proof}[Proof of Lemma \ref{lemma:: bound D3 recons}]

Note that for any $k\in[r]$, we have
\begin{align*}
    &\quad \EE_{\bx}\brac{f^{\star}(\bx)K^{(0)}(\bx,\bx')\bx^{\top}\bA_k\bx}\\
    &=\EE_{\bx}\brac{f^{\star}(\bx)\sum_{i=2}^{\infty}\frac{c_i^2Q_i(\bx^{\top}\bx')}{B(d,i)} \cdot \bx^{\top}\bA_k\bx}\\
    &=\sum_{i=2}^{\infty}\frac{c_i^2}{B(d,i)}\cdot \EE_{\bx}\brac{f^{\star}(\bx){Q_i(\bx^{\top}\bx')} \bx^{\top}\bA_k\bx}\\
    &=\frac{c_2^2}{B(d,2)d(d-1)}\cdot \left\langle\EE_{\bx}\brac{f^{\star}(\bx) (\bx^{\top}\bA_k\bx){(\bx\bx^{\top}-\bI)}},\bx'\bx'^{\top}-\bI\right\rangle\\
    &\quad + \sum_{i=3}^{\infty}\frac{c_2^2}{B(d,i)}\cdot \EE_{\bx}\brac{f^{\star}(\bx){Q_i(\bx^{\top}\bx')} \bx^{\top}\bA_k\bx}\\
    &=\frac{c_i^2}{B(d,2)d(d-1)}\cdot \left\langle T(\bA_k),\bx'\bx'^{\top}-\bI\right\rangle+ \sum_{i=3}^{\infty}\frac{c_i^2}{B(d,i)}\cdot \EE_{\bx}\brac{f^{\star}(\bx){Q_i(\bx^{\top}\bx')} \bx^{\top}\bA_k\bx}.
\end{align*}

Here $T$ is the linear operator defined in \eqref{equ::definition T}. Recall by Proposition \ref{prop::approximate stein lemma quad}, we have
\begin{align*}
    \norm{T(\bA_k)-\sum_{j=1}^{r}\bH_{k,j}\bA_j}{\rm F}\lesssim d^{-1/6}Lr^{2}\normA\log^2 d.
\end{align*}
Let's denote $\bR_k=T(\bA_k)-\sum_{j=1}^{r}\bH_{k,j}\bA_j$ so that $\norm{\bR_k}{\rm F}\lesssim d^{-1/6}Lr^{2}\normA\log^2 d$. Since $\langle \bR_k, \bx'\bx'^{\top}-\bI\rangle$ is a quadratic function of $\bx'$, and $\EE_{\bx'}\brac{\langle \bR_k, \bx'\bx'^{\top}-\bI\rangle^2}=\frac{2d}{d+2}\norm{\bR_k}{\rm F}^2$. By Lemma \ref{lemma::poly concentration sphere}, there exists a constant $C>0$ such that 
\begin{align*}
    \Pr\brac{\abs{\langle \bR_k, \bx'\bx'^{\top}-\bI\rangle}\geq C\iota\sqrt{\EE_{\bx'}\brac{\langle \bR_k, \bx'\bx'^{\top}-\bI\rangle^2}} }\le 2\exp(-\iota).
\end{align*}
Thus, by enumerating $k\in[r]$ and $\bx'\in\cD_2$, we obtain that with high probability $(1-nr\exp(-\iota))$ on $\cD_2$, for any $k\in[r]$, we have
\begin{align*}
    \abs{\left\langle T(\bA_k),\bx'\bx'^{\top}-\bI\right\rangle-\sum_{j=1}^{r}H_{k,j}\bx'^{\top}\bA_j\bx'}\lesssim \frac{\iota Lr^{2}\normA\log^2 d}{d^{1/6}}.
\end{align*}

Moreover, we have for any $\bx'$
\begin{align*}
    &\quad\abs{\sum_{i=3}^{\infty}\frac{c_i^2}{B(d,i)}\cdot \EE_{\bx}\brac{f^{\star}(\bx){Q_i(\bx^{\top}\bx')} \bx^{\top}\bA_k\bx}}\\
    &\le \sum_{i=3}^{\infty}\frac{c_i^2}{B(d,i)}\cdot \abs{\EE_{\bx}\brac{f^{\star}(\bx){Q_i(\bx^{\top}\bx')} \bx^{\top}\bA_k\bx}}\\
    &\le \sum_{i=3}^{\infty}\frac{c_i^2}{B(d,i)}\cdot \sqrt{\EE_{\bx}\brac{Q_i(\bx^{\top}\bx')^2}\EE_{\bx}\brac{f^{\star}(\bx)^2(\bx^{\top}\bA_k\bx)^2}}\\
    &\le \sum_{i=3}^{\infty} \frac{c_i^2}{B(d,i)^{3/2}}\cdot \sqrt{\EE_{\bx}\brac{f^{\star}(\bx)^2(\bx^{\top}\bA_k\bx)^2}}.
\end{align*}
Again by Lemma \ref{lemma::poly concentration sphere}, we know that there exists a sufficient large constant $C>0$ s.t. $\Pr\brac{\abs{\bx^{\top}\bA_k\bx}\geq C\iota }\le 2\exp(-\iota)$. By taking $\iota \geq (2p+2)\log d$, we have

\begin{align*}
    \EE_{\bx}\brac{f^{\star}(\bx)^2(\bx^{\top}\bA_k\bx)^2}
    &=\EE_{\bx}\brac{f^{\star}(\bx)^2(\bx^{\top}\bA_k\bx)^2\mathbf{1}\{\abs{\bx^{\top}\bA_k\bx}\le C\iota\}}\\&\quad+\EE_{\bx}\brac{f^{\star}(\bx)^2(\bx^{\top}\bA_k\bx)^2\mathbf{1}\{\abs{\bx^{\top}\bA_k\bx}> C\iota\}}\\
    &\lesssim C^2\iota^2+2{d^{2p+2}}{\exp(-\iota)}\\
    &\lesssim C^2\iota^2.
\end{align*}

Altogether, with high probability on $\cD_2$, for any $k\in[r]$, we have 
\begin{align*}
   &\quad \abs{\EE_{\bx}\brac{f^{\star}(\bx)K^{(0)}(\bx,\bx')\bx^{\top}\bA_k\bx}-\frac{c_2^2}{B(d,2)d(d-1)}\cdot \sum_{j=1}^{r}H_{k,j}\bx'^{\top}\bA_j\bx'}\\
   &\le \frac{\iota L r^2\normA\log^2 d}{B(d,2)d^{7/6}(d-1)} + \sum_{i=3}^{\infty}\frac{C\iota c_i^2}{B(d,i)^{3/2}}.
\end{align*}
Thus, by paralleling the $r$ entries together, we have with high probability on $\cD_2$
\begin{align*}
    \norm{\bD_3-\frac{c_2^2}{B(d,2)^2 d(d-1)}\cdot \bH \bp(\bx')}{\infty}
    &\le \frac{\iota L r^2\normA\log^2 d}{B(d,2)^2d^{7/6}(d-1)} + \sum_{i=3}^{\infty}\frac{C\iota c_i^2}{B(d,2)B(d,i)^{3/2}}\\
    &\lesssim \frac{\iota L r^2\normA\log^2 d}{d^{6+1/6}} .
\end{align*}
The proof is complete.
\end{proof}

\begin{proof}[Proof of Lemma \ref{lem:bound_fn_diff}]
By the mean value theorem, we have
\begin{align*}
\abs{g(\bB^{\star}\bh^{(1)}(\bx)) - g(\bp(\bx))} \lesssim \sup_{\lambda \in [0, 1]}\norm{\nabla g(\lambda \bB^{\star}\bh^{(1)}(\bx) + (1 - \lambda)\bp(\bx))}{2}\norm{\bB^{\star}\bh^{(1)}(\bx) - \bp(\bx)}{2}.
\end{align*}
Recall by \eqref{equ::bound grad g}, we have $\norm{\nabla g(\bz)}{2}\lesssim \norm{g}{L^2} \sum_{k=1}^{p}r^{\frac{p-k}{4}}\norm{\bz}{2}^{k-1}$. Note that with high probability, $\sup_{\bx \in \mathcal{D}_2}\norm{\bp(\bx)}{} \le \widetilde O(\sqrt{r})$. Therefore
\begin{align*}
\sup_{\bx \in \mathcal{D}_2}\sup_{\lambda \in [0, 1]}\norm{\nabla g(\lambda \bB^{\star}\bh^{(1)}(\bx) + (1 - \lambda)\bp(\bx))}{} \lesssim \norm{g}{L^2}r^{\frac{p-1}{2}}.
\end{align*}
Altogether, by Proposition \ref{prop::reconstructed feature},
\begin{align*}
&\quad\sup_{\bx \in \cD_2} \abs{g(\bB^{\star}\bh^{(1)}(\bx)) - g(\bp(\bx))} \\
&\lesssim \norm{g}{L^2}r^{\frac{p-1}{2}}\norm{\bB^{\star}{}\bh^{(1)}(\bx) - \bp(\bx)}{2}\\
&\le \norm{g}{L^2}\cdot \frac{r^{p/2}}{\lambda_{\min}(\bH)}\cdot \Bigg( \frac{\iota^{p+2}d^5}{m_2} +\frac{\iota d^3}{\sqrt{m_2}}+ \frac{\iota ^{p+3/2}d}{\sqrt{n}} +\frac{\iota Lr^2\normA \log^2 d}{d^{1/6}} \Bigg)
\end{align*}
The proof is complete.
\end{proof}

\subsection{Proof of Other Supporting Lemmas}
We first present the concentration of the initial kernel $K^{(0)}_{m_2}(\bx,\bx')$.
\begin{lemma}\label{lemma::concentration initial kernel}
    Let $K^{(0)}_{m_2}(\bx,\bx')=\frac{1}{m_2}\langle \sigma_2\paren{\bV\bx},\sigma_2\paren{\bV\bx'}\rangle$ be the initial kernel with inner width being $m_2$, and $K^{(0)}(\bx,\bx')=\EE_{\bv \sim \text{Unif-} \mathbb{S}^{d-1}(\sqrt{d})}\brac{\sigma_2(\bv^{\top}\bx)\sigma_2(\bv^{\top}\bx')}$ be the infinite-width kernel. Then there exists a constant $C$ s.t. when $m_2\geq Cd^{4}$, with high probability probability on $\bw$, $\bV$ and the training dataset $\cD$, for any $\bx\in \cD_1$ and $\bx' \in\cD_2$,  we have 
    \begin{align*}
        \abs{K^{(0)}_{m_2}(\bx,\bx')-K^{(0)}(\bx,\bx')} \le \frac{\iota}{\sqrt{m_2}d^2}.
    \end{align*}
\end{lemma}
\begin{proof}[Proof of Lemma \ref{lemma::concentration initial kernel}]
    By Assumption \ref{assump::activation}, for any $\bx,\bx'\in \cD$ and $\bv \in \mathbb{S}^{d-1}(\sqrt{d})$, we have
    \begin{align*}
        \abs{\sigma_2\paren{\bv^{\top}\bx}\sigma_2\paren{\bv^{\top}\bx'}}\le C_{\sigma}^{2}
    \end{align*}
    and
    \begin{align*}
        \EE_{\bv}\brac{\sigma_2\paren{\bv^{\top}\bx}^2\sigma_2\paren{\bv^{\top}\bx'}^2}\le \sqrt{\EE_{\bv}\brac{\sigma_2\paren{\bv^{\top}\bx}^4}\EE_{\bv}\brac{\sigma_2\paren{\bv^{\top}\bx'}^4}}\le \frac{C_4}{d^4}.
    \end{align*}
    Thus, by Bernstein inequality, we have 
    \begin{align}
        \Pr\brac{\abs{K^{(0)}_{m_2}(\bx,\bx')-K^{(0)}(\bx,\bx')} \geq \sqrt{\frac{t}{m_2}}}&\le 2 \exp \paren{\frac{-\frac{t}{2m_2}}{\frac{C_4}{m_2d^4}+\frac{C_{\sigma}^2}{3m_2}\sqrt{\frac{t}{m_2}}}} \nonumber\\
        &=\exp\paren{\frac{-{t}/2}{\frac{C_4}{d^4}+\frac{C_{\sigma}^2}{3}\sqrt{\frac{t}{m_2}}}}. \label{equ::initial kernel concentration}
    \end{align}
    By enumerating $\bx,\bx'$ over $\cD$, we have 
    \begin{align*}
        \Pr\brac{\max_{\bx,\bx'\in \cD}\abs{K^{(0)}_{m_2}(\bx,\bx')-K^{(0)}(\bx,\bx')} \geq \sqrt{\frac{t}{m_2}}}&\le n^2\exp\paren{\frac{-{t}/2}{\frac{C_4}{d^4}+\frac{C_{\sigma}^2}{3}\sqrt{\frac{t}{m_2}}}}.
    \end{align*}
    Thus, when $m_2\geq d^4$, we can take $t=\iota^2/d^4$ to bound the probability by $poly(d, n, m_2)e^{-\iota}$, which concludes our proof.
\end{proof}
Then we present the concentration of the reconstructed features.
\begin{lemma}\label{lemma:: concentration vtAvsigma2}
    Suppose $m_2\geq C_{\sigma}^2C_4^{-1/2}d^4\norm{\bA}{\rm op}^2$. Given any $\bA$ such that $\bv^{\top}\bA\bv$ is a quadratic spherical harmonic, with high probability on $\bV$, for any $\bx \in \cD$, we have
\begin{align*}
    \abs{\frac{1}{m_2}\sum_{i=1}^{m_2}(\bv_i^{\top}\bA\bv_i)\sigma_2(\bv^{\top}_i\bx)-\frac{c_2}{B(d,2)}\bx^{\top}\bA\bx}\le \frac{9\iota d^{-1}C_4^{1/4}}{\sqrt{m_2}}.
\end{align*}
\end{lemma}
\begin{proof}[Proof of Lemma \ref{lemma:: concentration vtAvsigma2}]
    Given any fixed $\bx\in \cD$ and $\bA$ such that $\bv^{\top}\bA\bv$ is a quadratic spherical harmonic, we have 
    \begin{align*}
        \EE_{\bv}\brac{(\bv^{\top}\bA\bv)^2\sigma^2_2(\bv^{\top}\bx)}
        &\le \sqrt{\EE_{\bv}\brac{(\bv^{\top}\bA\bv)^4}\EE_{\bv}\brac{\sigma^4_2(\bv^{\top}\bx)}}\\
        &\le (4-1)^{2* 2}\EE\brac{(\bv_i^{\top}\bA\bv_i)^2}d^{-2}C_4^{1/2}\\
        &=81d^{-2}C_4^{1/2}
    \end{align*}
and
\begin{align*}
    \abs{(\bv^{\top}\bA\bv)\sigma_2(\bv^{\top}\bx)}\le d\norm{\bA}{\rm op}\cdot C_{\sigma}=dC_{\sigma}\norm{\bA}{\rm op}.
\end{align*}
Since $\EE_{\bv}\brac{(\bv^{\top}\bA\bv)\sigma_2(\bv^{\top}\bx)}=\frac{c_2}{B(d,2)}\bx^{\top}\bA\bx$, by Bernstein Inequality, we have
\begin{align*}
    &\quad\Pr\brac{\abs{\frac{1}{m_2}\sum_{i=1}^{m_2}(\bv_i^{\top}\bA\bv_i)\sigma_2(\bv^{\top}_i\bx)-\frac{c_2}{B(d,2)}\bx^{\top}\bA\bx}\geq \frac{9\iota d^{-1}C_4^{1/4}}{\sqrt{m_2}}}\\
    &\le 2\exp\paren{-\frac{\frac{81d^{-2}C_4^{1/2}\iota^2}{2m_2}}{\frac{81d^{-2}C_4^{1/2}}{m_2}+\frac{1}{3m_2}\cdot d(C_{\sigma}\norm{\bA}{\rm op} \cdot \frac{9\iota d^{-1}C_4^{1/4}}{\sqrt{m_2}}  }}\\
    &= 2\exp \paren{-\frac{\iota^2/2}{1+\frac{C_{\sigma}d^2\norm{\bA}{\rm op}}{27C_4^{1/4}\sqrt{m_2}}\cdot \iota}}
\end{align*}
Thus, when $m_2\geq C_{\sigma}^2C_4^{-1/2}d^4\norm{\bA}{\rm op}^2$, by enumerating $\bx\in \cD$, we obtain that with high probability on $\bV$, for any $\bx \in \cD$, we have
\begin{align*}
    \abs{\frac{1}{m_2}\sum_{i=1}^{m_2}(\bv_i^{\top}\bA\bv_i)\sigma_2(\bv^{\top}_i\bx)-\frac{c_2}{B(d,2)}\bx^{\top}\bA\bx}\le \frac{9\iota d^{-1}C_4^{1/4}}{\sqrt{m_2}}.
\end{align*}
The proof is complete.
\end{proof}

\section{Approximation Theory of the Outer Layer}

\subsection{Proof of Proposition \ref{prop::construct random feature main}}\label{appendix::proof outer rf main}
Since we mainly focus on the first training stage throughout this section, we may sometimes denote $n=n_1$ for notation simplicity, and let the training set be $\cD_1=\{\bx_1,\bx_2,\dots,\bx_n\}$. Let's consider a formal version of Proposition \ref{prop::construct random feature main}.
\begin{proposition}
\label{prop::construct random feature}
Suppose g is a degree $p$ polynomial. By setting $\eta=C\iota^{-5}\normP^{-1} m_2^{-1/2}d^6$ for some constant $C>0$, with high probability over $\cD_1$, $\cD_2$, $\{\bw_i\}_{i=1}^{m_1}$ and $\bV$, there exists $\ba^{\star}\in\RR^{m_1}$ such that the parameter $\theta^{\star}=(\ba^{\star},\bW^{(1)},\bb^{(1)},\bV)$ gives rise to
\begin{align*}
\cL_2(\theta^{\star}):&=\frac{1}{n_2}\sum_{\bx \in \cD_2} \paren{f(\bx;\theta^{\star})-g(\bp(\bx))}^2\\
&\lesssim \norm{g}{L^2}^2\cdot \frac{r^{p}}{\lambda_{\min}(\bH)}\cdot \Bigg( \frac{\iota^{p+2}d^5}{m_2} +\frac{\iota d^3}{\sqrt{m_2}}+ \frac{\iota ^{p+3/2}d}{\sqrt{n}} +\frac{\iota Lr^2 \log^2 d}{d^{1/6}} \Bigg)^2\\
&\quad+{\frac{\iota^{p+1}\norm{g}{L^2}^2}{m_1}}\cdot\paren{\sum_{k = 0}^p \eta^{-k}\norm{\bB^{\star}}{\rm op}^kr^{\frac{p - k}{4}}}^2.
\end{align*}
Here $\ba^{\star}$ satisfies
\begin{align*}
    \frac{\norm{\ba^{\star}}{2}^2 }{m_1}\lesssim \iota^{p}{\norm{g}{L^2}^2}\cdot\paren{\sum_{k = 0}^p \eta^{-k}\norm{\bB^{\star}}{\rm op}^kr^{\frac{p - k}{4}}}^2=\widetilde \Omega(\normP^{2p} r^p).
\end{align*}
\end{proposition}
To prove the proposition, let's introduce the infinite-outer-width model as a transition term between the finite-outer-width model and the target function. We define the infinite-outer-width model as
\begin{align*}
    f_{\infty,m_2}(\bx;v)=\EE_{a,b,\bw}\brac{v(a,b,\bw)\sigma_{1}\paren{a\eta \langle \bw, \mathbf{h}^{(1)}(\bx) \rangle+b}},
\end{align*}
where $\mathbf{h}^{(1)}(\bx')=\frac{1}{n}\sum_{i=1}^{n}f^{\star}(\bx_i)\cdot K_{m_2}^{(0)}(\bx_i,\bx')\cdot\sigma_2\paren{\bV^{\top}\bx_i}$.

We can decompose the $L^2$ loss of the truth model $f(\bx;\theta)$ as 
\begin{align*}
    \hat{\cL}(\theta^{\star})&=\frac{1}{n}\sum_{\bx\in\cD_2} \paren{f(\bx;\theta)-f^{\star}(\bx)}^2\\
    &=\frac{1}{n}\sum_{\bx\in\cD_2}\paren{f(\bx;\theta)-f_{\infty,m_2}(\bx')+f_{\infty,m_2}(\bx')-g(\bB^{\star}\bh^{(1)}(\bx))+g(\bB^{\star}\bh^{(1)}(\bx))-g(\bp(\bx))}^2\\
    &\lesssim \underbrace{\frac{1}{n}\sum_{\bx\in\cD_2} \paren{f(\bx;\theta)-f_{\infty,m_2}(\bx')}^2}_{L_1}\\
    &\quad +\underbrace{\frac{1}{n}\sum_{\bx\in\cD_2} \paren{f_{\infty,m_2}(\bx')-g(\bB^{\star}\bh^{(1)}(\bx))}^2}_{L_2}\\
    &\quad+\underbrace{\frac{1}{n}\sum_{\bx\in\cD_2} \paren{g(\bB^{\star}\bh^{(1)}(\bx))-g(\bp(\bx))}^2}_{L_3}.
\end{align*}

We have bounded $L_3$ in Corollary \ref{lem:bound_fn_diff}. We state Lemmas \ref{lem:infinite_width_express} and \ref{lem:empirical_loss} as follows to bound $L_1$ and $L_2$, respectively.

\begin{lemma}[Bound $L_2$]\label{lem:infinite_width_express}
    Given $\bB^{\star}\in \RR^{r\times m_2}$ and setting the learning rate $\eta=C\iota^{-5}\normP^{-1}m_2^{-1/2}d^6 $ for a constant $C>0$, there exists $v : \{\pm 1\} \times \mathbb{R} \times \mathbb{R}^{m_2} \rightarrow \mathbb{R}$ such that
    \begin{align*}
    \norm{v}{L^2} \lesssim \norm{g}{L^2}\sum_{k = 0}^p \eta^{-k}\norm{\bB^{\star}}{\rm op}^kr^{\frac{p - k}{4}},
    \end{align*}
    and, with high probability over $\mathcal{D}_1$, $\mathcal{D}_2$ and $\bV$, the infinite-width network  satisfies 
    \begin{align*}
        \frac{1}{n}\sum_{\bx \in \mathcal{D}_2}(f_{\infty,m_2}(\bx;v) - g(\bB^{\star}\bh^{(1)}(\bx)))^2 \lesssim o\paren{\frac{1}{d^2n_1^2n_2^2m_1^2m_2^2}}.
    \end{align*}
\end{lemma}

\begin{lemma}[Bound $L_1$]\label{lem:empirical_loss}
Given the function $v:\{ \pm 1\}  \times\RR\times \RR^{m_2}\rightarrow \RR$ in Lemma \ref{lem:infinite_width_express}. With high probability over $\cD_1$, $\cD_2$, $\{\bw_i\}_{i=1}^{m_1}$ and $\bV$, it holds that for any $\bx \in \cD_2$,
\begin{align*}
& \abs{\frac{1}{m_1}\sum_{i=1}^{m_1}v(a_i, b_i, \bw_i)\sigma_1(\eta a_i \langle \bw_i, \bh^{(1)}(\bx) \rangle + b_i) - f_{\infty,m_2}(\bx;v)} \lesssim \sqrt{\frac{\iota^{p+1}\norm{v}{L^2}^2}{m_1}}, ~~\text{with}\\
        &\frac{1}{m_1}\sum_{i=1}^{m_1}v(a_i, b_i, \bw_i)^2 \lesssim\iota^{p} {\norm{v}{L^2}^2}.
\end{align*}
\end{lemma}
The proof of Lemmas \ref{lem:infinite_width_express} and \ref{lem:empirical_loss} is provided in Appendix \ref{appendix::proof outer rf}. Now we begin our proof of Proposition \ref{prop::construct random feature}.
\begin{proof}[Proof of Proposition \ref{prop::construct random feature}]
    By Corollary \ref{lem:bound_fn_diff}, Lemma \ref{lem:empirical_loss} and Lemma \ref{lem:infinite_width_express}, by defining the vector $\ba^{\star} \in \mathbb{R}^{m_1}$ by $a^{\star}_i = {v(a^{(0)}_i, b^{(1)}_i, \bw^{(1)}_i)}$ and letting $\theta^* = (\ba^*, \bW^{(1)}, \bb^{(1)}, \bV)$, we have with high probability that
\begin{align*}
\hat{\cL}_2(\theta^*) &\lesssim L_1+L_2+L_3\\
&\lesssim \norm{g}{L^2}^2\cdot \frac{r^{p}}{\lambda^2_{\min}(\bH)}\cdot \Bigg( \frac{\iota^{p+2}d^5}{m_2} +\frac{\iota d^3}{\sqrt{m_2}}+ \frac{\iota ^{p+3/2}d}{\sqrt{n}} +\frac{\iota Lr^2 \log^2 d}{d^{1/6}} \Bigg)^2\\
&\quad +\frac{ \iota^{p+1} \norm{g}{L^2}^2}{m_1}\cdot \left(\sum_{k = 0}^p \eta^{-k}\norm{\bB^{\star}}{\rm op}^kr^{\frac{p - k}{4}}\right)^2\\
&\quad + o\paren{\frac{1}{d^2n_1^2n_2^2m_1^2m_2^2}}\\
&\lesssim \norm{g}{L^2}^2\cdot \frac{r^{p}}{\lambda^2_{\min}(\bH)}\cdot \Bigg( \frac{\iota^{p+2}d^5}{m_2} +\frac{\iota d^3}{\sqrt{m_2}}+ \frac{\iota ^{p+3/2}d}{\sqrt{n}} +\frac{\iota Lr^2 \log^2 d}{d^{1/6}} \Bigg)^2\\
&\quad +\frac{\iota^{p+1} \norm{g}{L^2}^2}{m_1}\cdot \left(\sum_{k = 0}^p \eta^{-k}\norm{\bB^{\star}}{\rm op}^kr^{\frac{p - k}{4}}\right)^2.
\end{align*}
Here $\ba^{\star}$ satisfies
\begin{align*}
    \norm{\ba^{\star}}{2}^2&\le \sum_{i=1}^{m_1}v(a_i,b_i,\bw_i)^2\\&\lesssim m_1\iota^{p}\norm{v}{L^2}^2\\&\lesssim m_1{\iota^{p}\norm{g}{L^2}^2\paren{\sum_{k = 0}^p \eta^{-k}\norm{\bB^{\star}}{\rm op}^kr^{\frac{p - k}{4}}}^2}.
\end{align*}
 The proof is complete.
\end{proof}

\subsection{Omitted Proofs in Appendix \ref{appendix::proof outer rf main}}\label{appendix::proof outer rf}
\subsubsection{Random Feature Construction of Univariate Polynomials}

In this section, before proving Lemmas \ref{lem:infinite_width_express} and \ref{lem:empirical_loss}, we first construct univariate polynomials using the outer activation function $\sigma_1$ and the random features $a$ and $b$ progressively.

\begin{lemma}
    There exists $v_0(a, b)$, supported on $\{\pm 1\} \times [2, 3]$, such that for any $\abs{z} \le 1$
    \begin{align*}
        \E_{a, b}[v_0(a, b)\sigma(az + b)] = 1 ,~~ \sup_{a, b}\abs{v(a, b)} \lesssim 1.
    \end{align*}
\end{lemma}
\begin{proof}
    Let $v_0(a, b) = 12\cdot\mathbf{1}_{a = 1}(b - \frac52)\cdot \frac{\mathbf{1}_{b \in [2, 3]}}{\mu(b)}$. Then, since $z + b \geq 1$,
    \begin{align*}
        \E_{a, b}[v_0(a, b)\sigma(az + b)] &= 6\int_2^3(b - \frac52)\sigma(z + b)db\\
        &= 6\int_2^3(b - \frac52)(2z + 2b - 1)db\\
        &= z \cdot 6\int_2^3(b - \frac52)db + 6\int_2^3(b - \frac52)(2b - 1)db\\
        &= 1.
    \end{align*}
    The proof is complete.
\end{proof}

\begin{lemma}
    There exists $v_1(a, b)$, supported on $\{\pm 1\} \times [2, 3]$, such that for any $\abs{z} \le 1$
    \begin{align*}
        \E_{a, b}[v_1(a, b)\sigma(az + b)] = z ,~~ \sup_{a, b}\abs{v(a, b)} \lesssim 1.
    \end{align*}
\end{lemma}
\begin{proof}
Let $v_0(a, b) = \mathbf{1}_{a = 1}(-24b + 61)\cdot \frac{\mathbf{1}_{b \in [2, 3]}}{\mu(b)}$. Then, since $z + b \geq 1$,
    \begin{align*}
        \E_{a, b}[v_1(a, b)\sigma(az + b)] &= \frac12\int_2^3(-24b + 61)\sigma(z + b)db\\
        &= \frac12\int_2^3(-24b + 61)(2z + 2b - 1)db\\
        &= z \int_2^3(-24b + 61)db + \frac12\int_2^3(-24b + 61)(2b - 1)db\\
        &= z.
    \end{align*}
    The proof is complete.
\end{proof}

\begin{lemma}
    There exists $v_2(a, b)$, supported on $\{\pm 1\} \times [-2, 3]$, such that for any $\abs{z} \le 1$
    \begin{align*}
        \E_{a, b}[v_2(a, b)\sigma(az + b)] = z^2 ,~~ \sup_{a, b}\abs{v(a, b)} \lesssim 1.
    \end{align*}
\end{lemma}
\begin{proof}
First, see that
\begin{align*}
    \int_{-2}^2\sigma(z + b)db &= \int_{-2 + z}^{2 + z}\sigma(b)db\\
    &= \int_{-2 + z}^{-1}(-2b - 1)db + \int_{-1}^{1}b^2db + \int_{1}^{2 + z}(2b - 1)db\\
    &= [-b^2 - b]_{-2 + z}^{-1} + \frac23 + [b^2 - b]_1^{2 + z}\\
    &= (z - 2)^2 + (z - 2) + \frac23 + (z + 2)^2 - (z + 2)\\
    &= 2z^2 + \frac{14}{3}.
\end{align*}
Let $v_2(a, b) = \mathbf{1}_{a = 1}\frac{\mathbf{1}_{b \in [-2, 2]}}{\mu(b)} - \frac73v_0(a, b)$ Then
\begin{align*}
    \E_{a, b}[v_2(a, b)\sigma(az + b)] &= \frac12\int_{-2}^2\sigma(z + b)db - \frac73\\
    &= z^2 + \frac73 - \frac73\\
    &= z^2.
\end{align*}
The proof is complete.
\end{proof}

\begin{lemma}
    Let $v(b) = -\frac12 k(k-1)(k-2)(1 - b)^{k-3}\cdot \frac{\mathbf{1}_{b \in [0, 1]}}{\mu(b)}$. Then
    \begin{align*}
        \E_b[v_k(b)\sigma(z + b)] = z^k\cdot\mathbf{1}_{z > 0} - \frac{k(k-1)}{2}z^2 - kz - 1.
    \end{align*}
\end{lemma}
\begin{proof}
    Plugging in $v_k(b)$ and applying integration by parts yields 
    \begin{align*}
        \E_b[v_k(b)\sigma(z + b)] &= \int_0^1 -\frac12 k(k-1)(k-2)(1 - b)^{k-3}\sigma(z + b)db\\
        &= [\frac12k(k-1)(1 - b)^{k-2}\sigma(z + b)]_0^1 - \int_0^1 \frac12k(k-1)(1 - b)^{k-2}\sigma'(z + b)db\\
        &= -\frac12k(k-1)\sigma(z) + [\frac12k(1-b)^{k-1}\sigma'(z + b)]_0^1 - \int_0^1\frac12k(1-b)^{k-1}\sigma''(z + b)db\\
        &= -\frac12k(k-1)\sigma(z) -\frac12k\sigma'(z) - \int_0^1k(1-b)^{k-1}\mathbf{1}_{\abs{z + b} \le 1}db\\
    \end{align*}
    When $1 \geq z > 0$, we have
    \begin{align*}
        - \int_0^1k(1-b)^{k-1}\mathbf{1}_{\abs{z + b} \le 1}db = -\int_0^{1-z} k(1-b)^{k-1}db = [(1 - b)^{k}]_0^{1-z} = z^k - 1.
    \end{align*}
    When $-1 \le z \le 0$, we have
    \begin{align*}
        - \int_0^1k(1-b)^{k-1}\mathbf{1}_{\abs{z + b} \le 1}db = -\int_0^1k(1-b)^{k-1}db = -1.
    \end{align*}
    Since $z \in [-1, 1]$, we have that $\sigma(z) = z^2$ and $\sigma'(z) = 2z$. Therefore for $z \in [-1, 1]$
    \begin{align*}
        \E_b[v_k(b)\sigma(z + b)] = z^k\cdot\mathbf{1}_{z > 0} - \frac{k(k-1)}{2}z^2 - kz - 1.
    \end{align*}
    The proof is complete.
\end{proof}

\begin{lemma}\label{lem:transform_activation}
    There exists $v_k(a, b)$, supported on $\{\pm 1\} \times [-2, 3]$, such that for any $\abs{z} \le 1$
    \begin{align*}
        \E_{a, b}[v_k(a, b)\sigma(az + b)] = z^k ,~~ \sup_{a, b}\abs{v_k(a, b)} \lesssim \poly(k).
    \end{align*}
\end{lemma}

\begin{proof}
    We focus on $k \geq 3$. We have that
    \begin{align*}
        \E_b[v_k(b)\sigma(z + b)] &= z^k\cdot\mathbf{1}_{z > 0} - \frac{k(k-1)}{2}z^2 - kz - 1.\\
        \E_b[v_k(b)\sigma(-z + b)] &= (-z)^k\cdot\mathbf{1}_{z < 0} - \frac{k(k-1)}{2}z^2 + kz - 1.
    \end{align*}
    Therefore if $k$ is even
    \begin{align*}
        \E_b[v(b)\sigma(z + b) + v(b)\sigma(-z + b)] = z^k - k(k-1)z^2 - 2.
    \end{align*}
    Let $v_k(a, b) = 2v_k(b) + k(k-1)v_2(a, b) + 2$. Then
    \begin{align*}
        \E_{a, b}[v_k(a, b)\sigma(az + b)] = \E_b[v_k(b)\sigma(z + b) + v_k(b)\sigma(z - b)] + k(k-1)z^2 + 2 = z^k.
    \end{align*}
    If $k$ is odd,
    \begin{align*}
        \E_b[v(b)\sigma(z + b) - v(b)\sigma(-z + b)] = z^k - 2kz.
    \end{align*}
    Let $v_k(a, b) = 2av_k(b) + 2kv_1(a, b)$. Then
    \begin{align*}
        \E_{a, b}[v_k(a, b)\sigma(az + b)] = \E_b[v_k(b)\sigma(z + b) - v_k(b)\sigma(z - b)] + 2kz = z^k.
    \end{align*}
    The proof is complete.
\end{proof}

\subsubsection{Proof of Supporting Lemmas in Appendix \ref{appendix::proof outer rf main}}

\begin{proof}[Proof of Lemma \ref{lem:infinite_width_express}]
Let's consider a general version of Lemma \ref{lem:infinite_width_express}.
    \begin{lemma}\label{lemma::infinite outer width construct poly}
Let $g : \mathbb{R}^r \rightarrow \mathbb{R}$ be a degree $p$ polynomial, and let $\bB^{\star} \in \mathbb{R}^{r \times m_2}$. Given a set of vectors $\cD=\{\bz_1,\bz_2,\dots,\bz_n\}\subseteq \RR^{m_2}$ that satisfies $\eta\langle \bw,\bz\rangle\le 1$ for any $\bz \in \cD$ with probability at least $1-2(n_1+n_2)\exp(-\iota^2/2)$ over $\bw\sim \cN(\mathbf{0}_{m_2},\bI_{m_2})$ (uniformly over $\cD$). Then, there exists $v: \{\pm 1\}\times \mathbb{R} \times \mathbb{R}^m \rightarrow \mathbb{R}$ so that for all $\bz \in \cD$,
\begin{align*}
&\mathbb{E}_{a,  b, \bw}[v(a, b, \bw)\sigma_1(\eta a \langle \bw, \bz \rangle + b)] = g(\bB^{\star}\bz) + o\paren{\frac{1}{dn_1n_2m_1m_2}}  ,~~~\text{and}\\
&\norm{v}{L^2} \lesssim \norm{g}{L^2}\sum_{k = 0}^p \eta^{-k}\norm{\bB^{\star}}{\rm op}^{k} r^{\frac{p - k}{4}}.
\end{align*}
\end{lemma}
Thus, according to Proposition \ref{prop::bounded feature}, we could set the learning rate $\eta=C\iota^{-5}\normP^{-1}m_2^{-1/2}d^6 $ for a constant $C>0$ to ensure $\abs{\eta \langle \bw, \bh^{(1)}(\bx') \rangle}\le 1$ for any $\bx' \in \cD_2$ with high probability on $\bV$, $\cD_1$, and probability at least $1-2(n_1+n_2)\exp(-\iota^2/2)$ on $\bw$. Thus, taking $\cD=\{\bh(\bx)\}_{\bx \in\cD_2}$ concludes our proof.
\end{proof}

To prove Lemma \ref{lemma::infinite outer width construct poly}, we first decompose $g$ into sum of polynomials of different degrees and construct a function $v$ to express these polynomials accordingly.
\begin{lemma}\label{lemma::infinite outer width tensor construction}
    Given $\bz \in \RR^{m_2}$. Let $\bB^{\star} \in \RR^{r \times m_2}$ and $\bT_{k} \in (\RR^{r})^{\otimes k}$. Then, there exists $v_k: \RR^{m_2} \rightarrow \RR$ such that
    \begin{align*}
\EE_{\bw}\brac{v_k(\bw)(\eta \langle \bw ,\bz \rangle)^{k}}=\bT_{k}\paren{(\bB^{\star}\bz)^{\otimes k}} .
    \end{align*}
    Here $v_k$ satisfies
    \begin{align}\label{equ::range of v}
        \norm{v_k}{L^2}\lesssim \eta^{-k}\norm{\bB^{\star}}{op}^k \norm{\bT_k}{F}~~\text{and}~~\sup_w\abs{v_k(\bw)} &\lesssim m_2^{k/2}\eta^{-k}\norm{\bB^{\star}}{\rm op}^k\norm{\bT_k}{\rm F}.
    \end{align}
\end{lemma}
\begin{proof}[Proof of Lemma \ref{lemma::infinite outer width tensor construction}]
It suffices to solve
\begin{align*}
\mathbb{E}_\bw[v(\bw)\bw^{\otimes k}] = \eta^{-k}{\bB^{\star}}^{\otimes k}(\bT_k),
\end{align*}
where ${\bB^{\star}}^{\otimes k}(\bT_k) \in (\mathbb{R}^{m_2})^{\otimes k}$.
This is achieved by setting
\begin{align*}
v(\bw) := \eta^{-k}\vec(\bw^{\otimes k})^T\Mat(\mathbb{E}[\bw^{\otimes 2k}])^{-1}\vec({\bB^{\star}}^{\otimes k}(\bT_k)).
\end{align*}
Then,
\begin{align*}
\norm{v}{L^2}^2 = \eta^{-2k}\vec({\bB^{\star}}^{\otimes k}(\bT_k))^T\Mat(\mathbb{E}[\bw^{\otimes 2k}])^{-1}\vec({\bB^{\star}}^{\otimes k}(\bT_k)).
\end{align*}
Since
\begin{align*}
\Mat(\mathbb{E}[\bw^{\otimes 2k}]) \succeq k!\Pi_{\sym^k(\mathbb{R}^{m_2})},
\end{align*}
we have
\begin{align*}
\norm{v}{L^2}^2 \lesssim \eta^{-2k}\norm{{\bB^{\star}}^{\otimes k}(\bT_k)}{\rm F}^2 \le \eta^{-2k}\norm{\bB^{\star}}{\rm op}^{2k}\norm{\bT_k}{\rm F}^2.
\end{align*}
Finally,
\begin{align*}
\sup_w\abs{v(\bw)} &= \eta^{-k}\sup_w\abs{\vec(\bw^{\otimes k})^T\Mat(\mathbb{E}[\bw^{\otimes 2k}])^{-1}\vec({\bB^{\star}}^{\otimes k}(\bT_k))}\\
&\le \norm{\bw^{\otimes k}}{\rm F}\norm{{\bB^{\star}}^{\otimes k}(\bT_k)}{\rm F}\\
&\lesssim m_2^{k/2}\eta^{-k}\norm{\bB^{\star}}{\rm op}^k\norm{\bT_k}{\rm F}.
\end{align*}
The proof is complete.
\end{proof}
Then we begin our proof of Lemma \ref{lemma::infinite outer width construct poly}.
\begin{proof}[Proof of Lemma \ref{lemma::infinite outer width construct poly}]
We can write
\begin{align*}
g(\bz) = \sum_{k = 0}^p \langle \bT_k, \bz^{\otimes k}\rangle.
\end{align*}
By Lemma \ref{lemma::bound F norm tensor}, we have $\norm{\bT_k}{\rm F} \lesssim r^{\frac{p-k}{4}}\norm{g}{L^2}$.

Define $v_k(a, b)$ to be the function so that $\mathbb{E}_{a, b}[v_k(a, b)\sigma_1(az + b)] = z^k$, and let $v_k(\bw)$ be the function where $\mathbb{E}_\bw[v(\bw)(\eta\langle \bw, \bz\rangle)^k] = \langle \bT_k, (\bB^{\star}\bz)^{\otimes k} \rangle$. Next, define
\begin{align*}
v(a, b, \bw) = \sum_{k=0}^p v_k(a, b)v_k(\bw).
\end{align*}
Here $ v_k(a, b)$ is defined in Lemma \ref{lemma::infinite outer width tensor construction}. Then we have that
\begin{align*}
\norm{v}{L^2} \lesssim \sum_{k=0}^p(\mathbb{E}[v_k(\bw)^2])^{1/2} \le \norm{g}{L^2}\sum_{k = 0}^p \eta^{-k}\norm{\bB^{\star}}{\rm op}^kr^{\frac{p - k}{4}}.
\end{align*}
Note that $\norm{v}{L^2}=\cO \paren{{\rm poly}(m_2,d)}$ and $\abs{\sigma_1(\eta a \langle \bw, \bz \rangle + b)}\le \eta a \langle \bw, \bz \rangle + b)$ has polynomial growth. Since we have taken $\iota= C \log(dn_1n_2m_1m_2)$ for some sufficiently large $C>0$, we know by Cauchy inequality, 
\begin{align*}
    \abs {\mathbb{E}_{a,b,\bw} \brac{v(\bw)\sigma_1(\eta\langle \bw, \bz\rangle+b)\mathbf{1}\{\eta\langle \bw, \bz\rangle > 1\}}}\le o\paren{\frac{1}{dn_1n_2m_1m_2}}.
\end{align*}
Thus, we then have that
\begin{align*}
&\quad\mathbb{E}_{a,  b, \bw}[v(a, b, \bw)\sigma_1(\eta a \langle \bw, \bz \rangle + b)] \\&= \mathbb{E}_{a,  b, \bw}[v(a, b, \bw)\sigma_1(\eta a \langle \bw, \bz \rangle + b)\cdot \mathbf{1}_{\abs{\eta\langle \bw, \bz \rangle} \le 1 }] + o\paren{\frac{1}{dn_1n_2m_1m_2}} \\
&= \sum_{k = 0}^p\mathbb{E}_{a,  b, \bw}[v_k(a, b)v_k(\bw)\sigma_1(\eta a \langle \bw, \bz \rangle + b)\cdot \mathbf{1}_{\abs{\eta\langle \bw, \bz \rangle} \le 1 }] + o\paren{\frac{1}{dn_1n_2m_1m_2}} \\
&= \sum_{k = 0}^p\mathbb{E}_{\bw}[v_k(\bw)(\eta\langle \bw, \bz \rangle)^k\cdot \mathbf{1}_{\abs{\eta\langle \bw, \bz \rangle} \le 1 }]+ o\paren{\frac{1}{dn_1n_2m_1m_2}} \\
&= \sum_{k = 0}^p\mathbb{E}_{\bw}[v_k(\bw)(\eta\langle \bw, \bz \rangle)^k]+ o\paren{\frac{1}{dn_1n_2m_1m_2}} \\
&= \sum_{k = 0}^p\langle \bT_k, (\bB^{\star}\bz)^{\otimes k} \rangle+ o\paren{\frac{1}{dn_1n_2m_1m_2}} \\
&= g(\bB^{\star}\bz)+ o\paren{\frac{1}{dn_1n_2m_1m_2}} .
\end{align*}
The proof is complete.
\end{proof}

\begin{proof}[Proof of Lemma \ref{lem:empirical_loss}]

Fix $\bx\in\cD_2$. For notation simplicity, we denote $f^\infty_v(\bx)=f_{\infty,m_2}(\bx;v)$. Consider a truncation radius $R >0$ to be chosen later and let $E_x$ be the set of $\bw$ such that 
\begin{align*}
    \sup_{a, b}\abs{v(a, b, \bw)} \le R~~ \text{and} ~~\eta \langle \bw, \bh^{(1)}(\bx) \rangle \le 1.
\end{align*}
By the construction of $v(a,b,\bw)$ in the proof of Lemma \ref{lemma::infinite outer width construct poly}, we know it can be seen as a degree-$p$ polynomial of $\bw$. Thus, by Lemma \ref{lemma:polynomial concentration}, by taking $R=C\iota^{p/2}\norm{v}{L^2}$ for some sufficiently large $C>0$, we can ensure that $$\Pr[\sup_{a, b}\abs{v(a, b, \bw)} \le R] \ge 1- \exp(-\iota).$$ 
Moreover, by Proposition \ref{prop::bounded feature}, conditional on a high probability event on $\bV$, $\cD_1$ and $\cD_2$, by taking $\eta=C\iota^{-5}d^6m_2^{-1}$, we have $\Pr[\eta \langle \bw, \bh^{(1)}(\bx) \rangle \le 1] \ge 1- 4\exp(-\iota^2/2)$ for a single $\bx$. Now consider the random variables 
\begin{align*}
    Z_i := \mathbf{1}\{\bw_i \in E_x\}v(a_i, b_i, \bw_i)\sigma_1(\eta a_i \langle \bw_i, \bh^{(1)}(\bx) \rangle + b_i), ~~i=1,2,\dots,m_1.
\end{align*}
 We directly have that $\abs{Z_i} \lesssim \iota^{p/2}\norm{v}{L^2}$, and with high probability, 
 \begin{align*}
     \frac{1}{m_1}\sum_{i=1}^{m_1}v(a_i, b_i, \bw_i)^2 \lesssim \iota^{p}{\norm{v}{L^2}^2}.
 \end{align*}
Therefore by Hoeffding inequality, with probability at least $1-2\exp(-\iota)$, we have
\begin{align*}
&\Bigg|\frac{1}{m_1}\sum_{i=1}^{m_1}\mathbf{1}_{\bw_i \in E_x}v(a_i, b_i, \bw_i)\sigma_1(\eta a_i \langle \bw_i, \bh^{(1)}(\bx) \rangle + b_i)\\&\quad \hspace{5em}- \mathbb{E}[\mathbf{1}_{w \in E_x}v(a, b, \bw)\sigma_1(\eta a \langle \bw, \bh^{(1)}(\bx)\rangle + b)]\Bigg|\lesssim \sqrt{\frac{\iota^{p+1}\norm{v}{L^2}^2}{m_1}}.
\end{align*}
Similar to the proof of Lemma \ref{lemma::infinite outer width construct poly}, note that both $v(a,b,\bw)$ and $\abs{\sigma_1(\eta a \langle \bw, \bz \rangle + b)}$ has polynomial growth. Since we have taken $\iota= C \log(dn_1n_2m_1m_2)$ for some sufficiently large $C>0$, we know by Cauchy inequality, 
\begin{align*}
&\quad\abs{\mathbb{E}[\mathbf{1}_{w \in E_x}v(a, b, \bw)\sigma_1(\eta a \langle \bw, \bh^{(1)}(\bx)\rangle + b)] - f^\infty_v(\bx)} \\&= \abs{\mathbb{E}[\mathbf{1}_{w \not\in E_x}v(a, b, \bw)\sigma_1(\eta a \langle \bw, \bh^{(1)}(\bx)\rangle + b)]}\\
&\le \mathbb{P}(\bw \not\in E_x)(\mathbb{E}[v(a, b, \bw)^2\sigma_1(\eta a \langle \bw, \bh^{(1)}(\bx)\rangle + b)^2])^{1/2}\\
&\lesssim \exp(-C\log(dm_1m_2n_1n_2))\widetilde O(\norm{v}{L^2})\\
&\lesssim \frac{1}{m_1}.
\end{align*}
Finally, union bounding over $\bx\in \cD_2$, we see that
\begin{align*}
\sup_{\bx\in\cD_2}\abs{\frac{1}{m_1}\sum_{i=1}^{m_1}v(a_i, b_i, \bw_i)\sigma_1(\eta a_i \langle \bw_i, \bh^{(1)}(\bx) \rangle + b_i) - f^\infty_v(\bx)} &\lesssim \sqrt{\frac{\iota^{p+1}\norm{v}{L^2}^2}{m_1}} + \frac{1}{m_1} \\&\lesssim \sqrt{\frac{\iota^{p+1}\norm{v}{L^2}^2}{m_1}}.
\end{align*}
The proof is complete.
\end{proof}

\section{Generalization Theory}
\subsection{Formal Proof of Theorem \ref{thm:main_thm}}\label{appendix::generalization}
The proof is divided into two parts. The first part of proof formalizes the proof we present in Section \ref{sec::proof sketch}. The second part presents the generalization theory after we construct  $\ba^{\star}$ that gives small $L^2$ error by Proposition \ref{prop::construct random feature main}, with the formal version presented in Proposition \ref{prop::construct random feature}.
\subsubsection{Part1: Analysis Before Feature Reconstruction}
Denote $\bw_j=\epsilon^{-1}\bw_j^{(0)} \sim \cN(0,\bI_{m_2})$. Note that for any $\bx\in \cD_1$ and $j\in[m_1]$, we have
\begin{align*}
  \langle\bw^{(0)}_j, \bh^{(0)}(\bx) \rangle=  \langle \epsilon\bw_j, \bh^{(0)}(\bx) \rangle \sim \cN\paren{0,\epsilon^2\norm{\bh^{(0)}(\bx)}{2}^2}.
\end{align*}
Since $\norm{\bh^{(0)}(\bx)}{2}^2 =\sum_{k=1}^{m_2}\sigma^2_2(\bv_k^{\top}\bx) \le m_2C^2_{\sigma}$. By setting $\epsilon^{-1} = {C_{\sigma}\sqrt{2\iota m_2}}$, we know $ \langle\bw^{(0)}_j, \bh^{(0)}(\bx) \rangle \le 1$ with probability at least $1-2\exp(-\iota)$. Thus, uniformly bounding over $\bx \in \cD_1$ and $j\in[m_1]$, we know with high probability over $\bW$,  we have
\begin{align*}
    \langle\bw^{(0)}_j, \bh^{(0)}(\bx) \rangle \le 1~~\text{for any}~~\bx \in \cD_1,~~j\in[m_1].
\end{align*}
Then, according Algorithm \ref{alg:: training algo}, after one-step gradient descent on $\bW$, we know with high probability, for each $j\in[m_1]$,
\begin{align*}
\eta_1\nabla_{\bw_j^{(0)}} \cL(\theta^{(0)}) &= -\eta_1 \frac{a^{(0)}_j}{m_1}\cdot \frac{1}{n_1}\sum_{\bx\in\cD_1}f^*(\bx)\bh^{(0)}(\bx)\sigma_1'\left(\langle \epsilon\bw_j, \bh^{(0)}(\bx)\rangle\right)\\
&= -\frac{2\epsilon\eta_1}{m_1} {a^{(0)}_j} \cdot {\frac{1}{n_1}\sum_{i=1}^n f^*(\bx)\bh^{(0)}(\bx)\bh^{(0)}(\bx)^{\top}}\bw_j,
\end{align*}
which is a linear transformation on $\bw_j$. By taking $\eta_1=\frac{m_1}{2\epsilon m_2}\cdot \eta$ for some $\eta>0$ to be chosen later and $\lambda_1=\eta_1^{-1}$, we have
\begin{align*}
\bw^{(1)}_j &= \bw^{(0)}_j -\eta_1 \left[\nabla_{\bw_j^{(0)}} \cL(\theta^{(0)}) + \lambda_1 \bw^{(0)}_j\right]\\
&= -\eta_1\nabla_{\bw^{(0)}_j} \cL(\theta^{(0)})\\
&= \frac{\eta a^{(0)}_j}{m_2}\cdot\frac{1}{n_1}\sum_{i=1}^n f^*(\bx)\bh^{(0)}(\bx)\bh^{(0)}(\bx)^{\top}\bw_j .
\end{align*}
Then for any second-stage training sample $\bx'\in\cD_2$, the inner-layer neuron becomes
\begin{align*}
    \left\langle \bw_j^{(1)}, \sigma_2(\bV \bx') \right\rangle&=\frac{\eta a^{(0)}_j}{m_2}\left\langle \frac{1}{n_1}\sum_{i=1}^n f^*(\bx)\bh^{(0)}(\bx)\bh^{(0)}(\bx)^{\top}\bw_j,\bh^{(0)}(\bx')  \right\rangle\\
    &=\eta a^{(0)}_j\cdot  \left\langle\bw_j,{\frac{1}{n_1}\sum_{i=1}^{n}K^{(0)}_{m_2}(\bx,\bx')\bh^{(0)}(\bx)}\right\rangle\\
    &= \eta a^{(0)}_j\cdot\langle \bw_j, \mathbf{h}^{(1)}(\bx') \rangle.
\end{align*}
Thus, after the first training stage and reinitialization on $\bb=\bb^{(1)}$, the model becomes the following random-feature model in the second stage:
\begin{align*}
    f(\bx';\theta)=\frac{1}{m_1}\sum_{j=1}^{m_1}a_j\sigma_1\paren{\eta a^{(0)}_j \langle \bw_j, \mathbf{h}^{(1)}(\bx') \rangle+b^{(1)}_j}.
\end{align*}
By Proposition \ref{prop::construct random feature}, we know there exists $\ba^{\star}\in\RR^{m_1}$ such that with high probability over $\cD_1$, $\cD_2$, $\{\bw_i\}_{i=1}^{m_1}$ and $\bV$, by taking the parameter $\theta^{\star}=(\ba^{\star},\bW^{(1)},\bb^{(1)},\bV)$, it holds that
\begin{align*}
\hat{\cL}_2(\theta^{\star})&\lesssim \norm{g}{L^2}^2\cdot \frac{r^{p}}{\lambda_{\min}(\bH)}\cdot \Bigg( \frac{\iota^{p+2}d^5}{m_2} +\frac{\iota d^3}{\sqrt{m_2}}+ \frac{\iota ^{p+3/2}d}{\sqrt{n_1}} +\frac{\iota Lr^2\normA \log^2 d}{d^{1/6}} \Bigg)^2\\
&\quad+{\frac{\iota^{p+1}\norm{g}{L^2}^2}{m_1}}\cdot\paren{\sum_{k = 0}^p \eta^{-k}\norm{\bB^{\star}}{\rm op}^kr^{\frac{p - k}{4}}}^2.
\end{align*}
Here $\ba^{\star}$ satisfies
\begin{align*}
    \frac{\norm{\ba^{\star}}{2}^2 }{m_1}\lesssim \iota^{p}{\norm{g}{L^2}^2}\cdot\paren{\sum_{k = 0}^p \eta^{-k}\norm{\bB^{\star}}{\rm op}^kr^{\frac{p - k}{4}}}^2.
\end{align*}
The first part of the proof is complete.
\subsubsection{Part2: Generalization Theory}

Denote the population absolute loss as $\cL_1(f,g)=\EE_{\bx}\brac{\abs{f(\bx)-g(\bx)}}$. Moreover, we consider a truncated loss function as 
\begin{align*}
    \ell_{\tau}(z)=\min(\abs{z},\tau) ~~\text{and}~~ \cL_{1,\tau}(f,g)=\EE_{\bx}\brac{\ell_{\tau}\paren{f(\bx)-g(\bx)}},
\end{align*}
 where $\tau>0$ is the truncation radius. Moreover, we denote the empirical truncated absolute loss as $$\hat{\cL}_{1,\tau}(f,g)=\frac{1}{n_2}\sum_{\bx \in\cD_2}\ell_{\tau}\paren{f(\bx)-g(\bx)}. $$ Suppose Algorithm \ref{alg:: training algo} gives rise to  a set of parameters $\hat{\theta}=(\hat\ba,\bW^{(1)}, \bb^{(1)}, \bV)$, and we have constructed $\theta^{\star}=(\ba^{\star},\bW^{(1)}, \bb^{(1)}, \bV)$ that leads to small empirical loss, we decompose the population absolute loss as
\begin{align*}
    \cL_1(f(\cdot;\hat\theta),f^{\star})&= \underbrace{\hat{\cL}_{1,\tau}(f(\cdot;\hat\theta),f^{\star})}_{L_1} +\underbrace{\cL_{1,\tau}(f(\cdot;\hat\theta),f^{\star})-\hat{\cL}_{1,\tau}(f(\cdot;\hat\theta),f^{\star})}_{L_2}\\
    &\quad + \underbrace{\cL_1(f(\cdot;\hat\theta),f^{\star})-\cL_{1,\tau}(f(\cdot;\hat\theta),f^{\star})}_{L_3}.
\end{align*}
Here with a little abuse of notation, we consider $f^{\star}=g^{\star}(\bp)$ for learning the original target function and denote $f^{\star}=g(\bp)$ with $g$ being any degree $p$ polynomial for the transfer learning setting.  Next, we bound $L_1,L_2$ and $L_3$ respectively.
\paragraph{Bound $L_1$}

With a little abuse of notation, we denote $\hat\cL_2(\ba)=\hat\cL_2(\theta)$ for $\theta=(\ba,\bW^{(1)},\bb^{(1)},\bV)$ since we only optimize $\ba$ in the second stage. By Proposition \ref{prop::construct random feature}, we know with high probability, the empirical $L^2$ loss of $\theta^{\star}$ is bounded by
\begin{align*}
     \hat{\cL}_{2}(\ba^{\star})&=\frac{1}{n_2}\sum_{\bx \in \cD_2}\paren{f(\bx;\theta^{\star})-f^{\star}(\bx)}^2\\ &\lesssim \norm{g}{L^2}^2\cdot \frac{r^{p}}{\lambda^2_{\min}(\bH)}\cdot \Bigg( \frac{\iota^{p+2}d^5}{m_2} +\frac{\iota d^3}{\sqrt{m_2}}+ \frac{\iota ^{p+3/2}d}{\sqrt{n_1}} +\frac{\iota Lr^2\normA \log^2 d}{d^{1/6}} \Bigg)^2\\
&\quad+{\frac{\iota^{p+1}\norm{g}{L^2}^2}{m_1}}\cdot\paren{\sum_{k = 0}^p \eta^{-k}\norm{\bB^{\star}}{\rm op}^kr^{\frac{p - k}{4}}}^2.
\end{align*}
Here $\ba^{\star}$ satisfies
\begin{align*}
    \frac{\norm{\ba^{\star}}{2}^2 }{m_1}\lesssim \iota^{p}{\norm{g}{L^2}^2}\cdot\paren{\sum_{k = 0}^p \eta^{-k}\norm{\bB^{\star}}{\rm op}^kr^{\frac{p - k}{4}}}^2.
\end{align*}
In the second training stage, let's set the weight decay in the second training stage as \begin{align*}
   \lambda_2=\lambda&=\norm{\ba^{\star}}{2}^{-2}\norm{g}{L^2}^2\cdot \Bigg( \frac{r^{p}}{\lambda_{\min}^2(\bH)}\cdot \Bigg( \frac{\iota^{p+2}d^5}{m_2} +\frac{\iota d^3}{\sqrt{m_2}}+ \frac{\iota ^{p+3/2}d}{\sqrt{n_1}} +\frac{\iota Lr^2\normA \log^2 d}{d^{1/6}} \Bigg)^2\\
&\quad +\frac{ \iota^{p+1}}{m_1}\cdot \left(\sum_{k = 0}^p \eta^{-k}\norm{\bB^{\star}}{\rm op}^kr^{\frac{p - k}{4}}\right)^2\Bigg)
\end{align*}
so that the empirical $L^2$ loss is directly bounded by
\begin{align*}
     \hat{\cL}_{2}(\ba^{\star}):=\frac{1}{n_2}\sum_{\bx \in \cD_2}\paren{f(\bx;\theta^{\star})-f^{\star}(\bx)}^2 \lesssim\lambda\norm{\ba^{\star}}{2}^2.
\end{align*}
We further consider the regularized second-stage training loss to be 
\begin{align*}
    \hat{\cL}_{2,\lambda}(\ba)= \frac{1}{n_2}\sum_{\bx \in \cD_2}\paren{f(\bx;(\ba,\bW^{(1)},\bb^{(1)},\bV))-f^{\star}(\bx)}^2 + \frac{\lambda}{2}\norm{\ba}{2}^2.
\end{align*}
Note that this loss is strongly convex, so it has a global minimum $\ba^{(\infty)}=\argmin{
\hat{\cL}}_{2,\lambda}(\ba)$. Thus, we have
\begin{align*}
\cL_{2,\lambda}(\ba^{(\infty)}) \le \cL_{2,\lambda}(\ba^{\star})   \lesssim {\lambda}\norm{\ba^{\star}}{2}^2.
\end{align*}
Since ${\cL}_{2,\lambda}(\ba)$ is $\lambda$- strongly convex, and we can write $f(\bx;(\ba,\bW^{(1)},\bb^{(1)},\bV))=\ba^{\top}\Psi(\bx)$, where $\Psi(
\bx)=\Vec\paren{{m_1^{-1}}\sigma_1\left(\eta a^{(0)}_i \langle \bw_i, \bh^{(1)}(\bx) \rangle +b^{(1)}_i\right)}$. Therefore, by Lemma \ref{prop::bounded feature} and our choice of $\eta$ to ensure $\eta \langle \bw_i, \bh^{(1)}(\bx) \rangle \le 1$ with high probability, we know with high probability,
\begin{align*}
    \lambda_{\max}\paren{\nabla^2_{\ba}\hat{\cL}_{2,\lambda}}\le \frac{2}{n_2}\sum_{\bx\in \cD_2}\norm{\Psi(
\bx)}{2}\lesssim \frac{1}{m_1}.
\end{align*}
Thus, ${\cL}_{2,\lambda}(\ba)$ is $\lambda+\cO(\frac{1}{m_1})$- smooth. By choosing the second-stage learning rate $\eta_2=\Omega(m_1)$, after $T=\widetilde\cO({\lambda}^{-1})={\rm poly}(d,n,m_1,m_2,\norm{g}{L^2})$ steps, we can reach an iterate $\hat\ba=\ba^{(T)}$  so that
\begin{align*}
    \hat\cL_2(\hat \ba)\lesssim \hat\cL_2(\ba^{\star})~~~\text{and}~~~\norm{\hat\ba}{2}\lesssim \norm{\ba^{\star}}{2}.
\end{align*}
Denoting $\hat \theta=(\hat \ba,\bW^{(1)},\bb^{(1)},\bV)$,  it holds that 
\begin{align*}
\hat{\cL}_{1,\tau}(f(\cdot;\hat\theta),f^{\star})& \le \frac{1}{n_2}\sum_{\bx \in \mathcal{D}_2}\abs{f(\bx; \hat\theta) - f^{\star}(\bx)}  \le\sqrt{\cL_2\paren{\hat
\ba}}\le\sqrt{\cL_2\paren{
\ba^{\star}}}.
\end{align*}
Thus, we have
\begin{align*}
L_1=\hat{\cL}_{1,\tau}(f(\cdot;\hat\theta),f^{\star})&\le \sqrt{\frac{1}{n_2}\sum_{\bx \in \cD_2}\paren{f(\bx;\theta^{\star})-f^{\star}(\bx)}^2}\\
&\lesssim \norm{g}{L^2}\cdot \frac{r^{p/2}}{\lambda_{\min}(\bH)}\cdot \Bigg( \frac{\iota^{p+2}d^5}{m_2} +\frac{\iota d^3}{\sqrt{m_2}}+ \frac{\iota ^{p+3/2}d}{\sqrt{n}} +\frac{\iota Lr^2\normA \log^2 d}{d^{1/6}} \Bigg)\\
&\quad+\sqrt{\frac{\iota^{p+1}\norm{g}{L^2}}{{m_1}}}\cdot \left(\sum_{k = 0}^p \eta^{-k}\norm{\bB^{\star}}{\rm op}^kr^{\frac{p - k}{4}}\right).
\end{align*}
Here $\hat\ba$ satisfies
\begin{align*}
     \frac{\norm{\hat\ba}{2}^2 }{m_1}\lesssim\frac{\norm{\ba^{\star}}{2}^2 }{m_1}\lesssim {\iota^{p}{\norm{g}{L^2}^2}\cdot\paren{\sum_{k = 0}^p \eta^{-k}\norm{\bB^{\star}}{\rm op}^kr^{\frac{p - k}{4}}}^2}.
\end{align*}
We assume $\norm{\ba}{2}^2\le m_1B_a^2$, where $B_a$ satisfies
\begin{align*}
    B_a^2\lesssim {\iota^{p}{\norm{g}{L^2}^2}\cdot\paren{\sum_{k = 0}^p \eta^{-k}\norm{\bB^{\star}}{\rm op}^kr^{\frac{p - k}{4}}}^2}.
\end{align*}

\paragraph{Bound $L_2$}
To bound $L_2$, we rely on standard Rademacher complixity analysis. The following lemma provides an upper bound on the Rademacher complixity of the random feature model. 

\begin{lemma}\label{lemma::rademacher} Let $\mathcal{F} = \{f_\theta : \theta = (\ba, \bW^{(1)}, \bb^{(1)}, \bV), \norm{\ba}{2} \le \sqrt{m_1}B_a\}$. Recall the empirical Rademacher complexity of $\cF$ as 
\begin{align*}
\mathcal{R}_n(\mathcal{F}) = \EE_{\sigma \in \{\pm 1\}^n}\brac{\sup_{f\in \cF}\frac{1}{n_2}\sum_{i=1}^{n}\sigma_i f(\bx_i)},
\end{align*}
Here the dataset $\{\bx_1,\bx_2,\dots,\bx_{n_2}\}=\cD_2$. Then with high probability, we have
\begin{align*}
\mathcal{R}_n(\mathcal{F}) \lesssim \frac{B_a}{\sqrt{n_2}}.
\end{align*}
\end{lemma}
The proof is provided in Appendix \ref{appendix::omit proof generalization}. Since the $\ell_{\tau}$ is $1$-Lipschitz, by standard Rademacher complexity analysis, we have that with high probability that
\begin{align*}
L_2&=\mathbb{E}_{\bx}\ell_{\tau}\paren{f(\bx;\hat\theta) - f^*(\bx)} -\frac{1}{n_2}\sum_{\bx \in \mathcal{D}_2}\ell_{\tau}\paren{f(\bx; \hat\theta) - f^{\star}(\bx)}\\ &\lesssim\mathcal{R}_n(\mathcal{F}) +\tau\sqrt{\frac{\iota}{n_2}}\\
&\lesssim  \sqrt{\frac{B_a^2}{n_2}} +\tau\sqrt{\frac{\iota}{n_2}}.
\end{align*}
\paragraph{Bound $L_3$}
Finally, we relate the truncated loss $\ell_{\tau}$ to the $L_1$ population loss.
\begin{lemma}\label{lemma::truncated loss error}
By letting $\tau=\Omega(\max(\iota^{p},B_a))$, with high probability over $\hat{\theta}$, we have 
\begin{align*}
   L_3= \EE_{\bx}\brac{\abs{f(\bx;\hat\theta) - f^*(\bx)}}-\mathbb{E}_{\bx}\brac{\ell_{\tau}\paren{f(\bx;\hat\theta) - f^*(\bx)}}\le o\paren{\frac{1}{n_1n_2m_1m_2d}}.
\end{align*}
\end{lemma}
Here we recall that $n=n_1+n_2$. The proof is provided in Appendix \ref{appendix::omit proof generalization}.
\paragraph{Put the loss together}
By invoking the upper bound of $L_1$, $L_2$ and $L_3$ and plugging the values of $\tau$, $\eta$, $\norm{\bB^{\star}}{\rm op}$, $\lambda_{\min}(\bH)$, $B_a$, $L$ and $\norm{g}{L^2}$, we have
\begin{align*}
    \EE_{\bx}\brac{\abs{f(\bx;\hat\theta) - f^*(\bx)}}&=L_1+L_2+L_3\\
    &\lesssim \frac{1}{n_2}\sum_{\bx \in \mathcal{D}_2}\ell_{\tau}\paren{f(\bx; \theta^{\star}) - f^{\star}(\bx)} +  \sqrt{\frac{B_a^2}{n_2}} +\tau\sqrt{\frac{\iota}{n_2}} +o\paren{\frac{1}{n_1n_2m_1m_2d}}\\
    &\lesssim \norm{g}{L^2}\cdot \frac{r^{p/2}}{\lambda_{\min}(\bH)}\cdot \Bigg( \frac{\iota^{p+2}d^5}{m_2} +\frac{\iota d^3}{\sqrt{m_2}}+ \frac{\iota ^{p+3/2}d}{\sqrt{n_1}} +\frac{\iota Lr^2\normA \log^2 d}{d^{1/6}} \Bigg)\\
&\quad +\sqrt{\frac{\iota^{p+1}\norm{g}{L^2}}{{m_1}}}\cdot \left(\sum_{k = 0}^p \eta^{-k}\norm{\bB^{\star}}{\rm op}^kr^{\frac{p - k}{4}}\right)+  \sqrt{\frac{B_a^2}{n_2}} +\tau\sqrt{\frac{\iota}{n_2}}. \\
&\lesssim \frac{r^{p/2}}{\lambda_{\min}(\bH)}\cdot \Bigg( \frac{\iota^{p+2}d^5}{m_2} +\frac{\iota d^3}{\sqrt{m_2}}+ \frac{\iota ^{p+3/2}d}{\sqrt{n_1}} +\frac{\iota Lr^2\normA \log^2 d}{d^{1/6}} \Bigg)\\
&\quad+\sqrt{\frac{\iota^{6p+1}r^{p/2} \normP^{2p} (r^{1/4} \vee \lambda^{-1}_{\min}(\bH))^{p}}{m_1}}\\
&\quad+\sqrt{\frac{\iota^{6p+1}r^{p/2}  \normP^{2p} (r^{1/4} \vee \lambda^{-1}_{\min}(\bH))^{p}}{n_2}}\\
&=\widetilde \cO \paren{ \sqrt{\frac{r^{p} \normP^{2p}}{{\min(n_2,m_1)}} }+ \sqrt{\frac{d^6r^{p+1}}{m_2}}+ \sqrt{\frac{d^2r^{p+1}}{n}} +\frac{ r^{p+2}\normA }{d^{1/6}} }.
\end{align*}
The proof is complete.

\subsection{Omitted Proofs in Appendix \ref{appendix::generalization}}\label{appendix::omit proof generalization}
\begin{proof}[Proof of Lemma \ref{lemma::rademacher}]
    Given $\theta=(\ba,\bW^{(1)},\bb^{(1)},\bV)$, since we can write 
    \begin{align*}
        f_\theta(\bx)=\ba^{\top}\Psi(\bx),~~\text{where}~~\Psi(
\bx)=\Vec\paren{{m_1^{-1}}\sigma_1\left(\eta a^{(0)}_i \langle \bw_i, \bh^{(1)}(\bx) \rangle +b^{(1)}_i\right)}.
    \end{align*}
By Proposition \ref{prop::bounded feature} and our choice of $\eta$ to ensure $\abs{\eta \langle \bw_i, \bh^{(1)}(\bx) \rangle} \le 1$ with high probability for any $\bx\in\cD_2$, we obtain that for any $i\in[m_1]$ and $\bx \in\cD_2$,
\begin{align*}
    \abs{\eta a^{(0)}_i \langle \bw_i, \bh^{(1)}(\bx) \rangle +b^{(1)}_i}\le a_i^{(0)}\abs{\eta \langle \bw_i, \bh^{(1)}(\bx) \rangle} +b_i^{(1)} \lesssim 1.
\end{align*}
Thus, $\norm{\Psi(\bx)}{2}^2\le m_1^{-1}.$ by the standard linear Rademacher bound, with high probability, the empirical Rademacher complexity is upper bounded by
\begin{align*}
    \mathcal{R}_n(\mathcal{F}) \lesssim \frac{\sqrt{m_1}B_a}{n_2}\sqrt{\sum_{\bx\in\cD_2}\norm{\Psi(\bx)}{2}^2} \le \frac{B_a}{\sqrt{n_2}}.
\end{align*}
The proof is complete.
\end{proof}
\begin{proof}[Proof of Lemma \ref{lemma::truncated loss error}]
    We can bound the difference between $\ell_\tau$ and $L_1$ loss by
    \begin{align}
         &\quad\EE_{\bx}\brac{\abs{f(\bx;\hat\theta) - f^*(\bx)}}-\mathbb{E}_{\bx}\brac{\ell_{\tau}\paren{f(\bx;\hat\theta) - f^*(\bx)}}\nonumber\\
         &\le \EE_{\bx}\brac{\abs{f(\bx;\hat\theta) - f^*(\bx)} \mathbf{1}\left\{\abs{f(\bx;\hat\theta) - f^*(\bx)}\ge \tau\right \}}\nonumber\\
         &\le \sqrt{ \EE_{\bx}\brac{{\abs{f(\bx;\hat\theta) - f^*(\bx)}^2}}\Pr\brac{\abs{f(\bx;\hat\theta) - f^*(\bx)}\ge \tau}}\nonumber\\
         &\lesssim  \sqrt{ \EE_{\bx}\brac{{\paren{f(\bx;\hat\theta)^2 + f^{\star}(\bx)^2}}} \brac{\Pr\brac{\abs{f(\bx;\hat\theta)}\ge \tau/2  } +\Pr\brac{\abs{f^{\star}(\bx)}\ge \tau/2}}}\label{equ::bound trunctaed loss}
    \end{align}
Recall that we can write 
    \begin{align*}
        f(\bx;\hat\theta)=\hat\ba^{\top}\Psi(\bx),~~\text{where}~~\Psi(
\bx)=\Vec\paren{{m_1^{-1}}\sigma_1\left(\eta a^{(0)}_i \langle \bw_i, \bh^{(1)}(\bx) \rangle +b^{(1)}_i\right)}.
    \end{align*}
 By following the proof of Lemma \ref{lemma::rademacher} and applying Proposition \ref{prop::bounded feature} for one single sample point $\bx$ (instead of the whole set $\cD_2$), we know with high probability over $\bV$, $\bw$ and $\cD_1$ (we denote this event by $E_1$), we have for any $i\in[m_1]$,
\begin{align*}
\abs{\eta \langle \bw_i, \bh^{(1)}(\bx) \rangle}\le 1
\end{align*}
holds with high probability on $\bx$. Also, since for any $i\in[m_1]$, $\bw_i \sim \cN(\mathbf{0}_{m_2},\bI_{m_2})$, we know $\norm{\bw_i}{2}\lesssim \sqrt{m_2\iota}$ for any $i \in[m_1]$ with high probability. We denote this joint event on $\bw_1,\bw_2,\dots,\bw_{m_1}$ by $E_2$. Thus, conditional on events $E_1$ and $E_2$, we have 
\begin{align*}
    \abs{f(\bx;\hat\theta)}&\lesssim \frac{\norm{\hat \ba}{2}}{\sqrt{m_1}}~~~\text{with high probability on $\bx$.}
\end{align*}
We  denote this conditional event by $E_{x,1}$. Moreover, since $\eta=C\iota^{-5}m_2^{-1/2}d^6$, we have
\begin{align*}
      \abs{f(\bx;\hat\theta)}&\le \frac{\norm{\hat\ba}{2}}{m_1} \sum_{j=1}^{m_1} \paren{\eta\norm{\bw_j}{2}\norm{{\frac{1}{n_2}\sum_{i=1}^{n}K^{(0)}_{m_2}(\bx_i,\bx')\bh^{(0)}(\bx_i)}}{2} +3}\\
    &\le \frac{\sqrt{m_2\iota}\norm{\hat\ba}{2}}{m_1} \sum_{j=1}^{m_1}  \frac{\eta}{n}\sum_{i=1}^{n} C^2_{\sigma}\cdot \sqrt{m_2}C_{\sigma} +3\\
    &\lesssim \sqrt{m_2}d^6\norm{\hat\ba}{2}
\end{align*}
holds for any $\bx$. Moreover, since $f^{\star}(\bx)$ is a degree-$2p$ polynomial of $\bx$, we know by Lemma \ref{lemma::poly concentration sphere}, with probability at least $1-\exp(-\iota)$, we have $\abs{f^{\star}}\le C_f\iota^{p}$ for sufficiently large $C_f>0$. Besides, we have $\EE_{\bx}\brac{f^{\star}(\bx)^2}\lesssim 1$. Altogether, conditional on $E_1$ and $E_2$, by choosing $\tau =C' \max(\iota^p,m_1^{-1/2}\norm{\hat{\ba}}{2})=\Omega(\max(\iota^{p},B_a)$ for some sufficiently large $C$, we have  
\begin{align*}
    &\quad\EE_{\bx}\brac{{\paren{f(\bx;\hat\theta)^2 + f^{\star}(\bx)^2}}} \brac{\Pr\brac{\abs{f(\bx;\hat\theta)}\ge \tau/2  } +\Pr\brac{\abs{f^{\star}(\bx)}\ge \tau/2}}\\
    &\lesssim \paren{m_2d^{12}\norm{\hat\ba}{2}^2+1} \paren{\Pr[\bx \not\in E_{x,1}]+\Pr[\bx \not\in E_{x,2}]}\\
    &\lesssim o\paren{\frac{1}{d^2m_1^2m_2^2n_1^2n_2^2}}.
\end{align*}
The last inequality holds because of the definition of high probability events and the choice of $\iota$ with $\iota=C\log(dm_1m_2n_1n_2)$ for sufficiently large $C$. Plugging the result into \eqref{equ::bound trunctaed loss} concludes our proof.
\end{proof}

\end{document}